\documentclass{ecai} 

\usepackage{microtype}
\usepackage{graphicx}
\usepackage{subfigure}
\usepackage{amsmath}
\usepackage{multirow}
\usepackage{float}
\usepackage{algorithm}
\usepackage{tikz}
\usepackage{pgfmath}
\usepackage{tabularx}
\usepackage{algorithmic}
\usepackage{amsmath}
\usepackage{multirow}
\usepackage{latexsym}
\usepackage{booktabs} 
\usepackage{enumitem}

\usepackage{hyperref}
\hypersetup{hidelinks}
\hypersetup{pdflang={en},pdfdisplaydoctitle}

\usepackage{amsmath}
\usepackage{amssymb}
\usepackage{mathtools}
\usepackage{amsthm}
\usepackage{bbm}

\usepackage[capitalize,noabbrev]{cleveref}

\theoremstyle{plain}

\newtheorem{theorem}{Theorem}
\newtheorem{corollary}{Corollary}
\newtheorem{lemma}{Lemma}
\newtheorem{assumption}{Assumption}
\newtheorem{remark}{Remark}

\newcommand{\smartparagraph}[1]{{\bf #1}\ }

\usepackage[textsize=tiny]{todonotes}

\usepackage{ifthen}
\providecommand{\dodraft}{true}
\newboolean{isdraft}
\setboolean{isdraft}{\dodraft}
\ifthenelse{\boolean{isdraft}}{%
        \newcommand{\mcnote}[1]{{\color{purple}{{\bf MC: #1} }}}
} {
        \newcommand{\mcnote}[1]{}
}

\usepackage{booktabs} 

\usepackage{caption} 

\newcommand{\BibTeX}{B\kern-.05em{\sc i\kern-.025em b}\kern-.08em\TeX}

\begin{document}


\begin{frontmatter}


\paperid{1591} 


\title{FilFL: Client Filtering for Optimized Client Participation in Federated Learning}


\author[A]{\fnms{Fares}~\snm{Fourati}\footnote{\textbf{Equal contribution.} Corresponding authors. \\Emails: salma.kharrat@kaust.edu.sa and fares.fourati@kaust.edu.sa.}}
\author[A]{\fnms{Salma}~\snm{Kharrat}\footnotemark}
\author[A,B]{\fnms{\\Vaneet}~\snm{Aggarwal}} 
\author[A]{\fnms{Mohamed-Slim}~\snm{Alouini}} 
\author[A]{\fnms{Marco}~\snm{Canini}} 


\address[A]{KAUST}
\address[B]{Purdue University}


\begin{abstract}
Federated learning, an emerging machine learning paradigm, enables clients to collaboratively train a model without exchanging local data. Clients participating in the training process significantly impact the convergence rate, learning efficiency, and model generalization. We propose a novel approach, client filtering, to improve model generalization and optimize client participation and training. The proposed method periodically filters available clients to identify a subset that maximizes a combinatorial objective function with an efficient greedy filtering algorithm. Thus, the clients are assessed as a combination rather than individually. We theoretically analyze the convergence of federated learning with \textit{client filtering} in heterogeneous settings and evaluate its performance across diverse vision and language tasks, including realistic scenarios with time-varying client availability. Our empirical results demonstrate several benefits of our approach, including improved learning efficiency, faster convergence, and up to 10\% higher test accuracy than training without client filtering.
\end{abstract}
\end{frontmatter}


\section{Introduction}
Federated learning (FL) is an emerging machine learning paradigm that enables collaborative training across multiple clients while preserving their local data privacy \cite{konevcny2015federated, shokri2015privacy, konevcny2016federated, konevcny2017stochastic, li2020federated}. The most commonly used approach in this setting, federated averaging (FedAvg) \cite{mcmahan2017communication}, alternates between local training and server aggregation and broadcasts the latest version of the global model. However, FL faces various challenges,\footnote{Although privacy is not the primary concern of this work, it remains a significant challenge in FL. However, conventional techniques like differential privacy and secure multiparty computation could be used in conjunction with our proposed method.} such as training with many clients and data heterogeneity, where the clients' data are non-IID, i.e., different clients have different data distributions \cite{Bonawitz19, HosseinalipourMag2020, Papaya, ganguly2023multi, wang2023towards}. 

Recent works have analyzed the effect of data heterogeneity on the convergence of local-update stochastic gradient descent (SGD) \cite{reddi2020adaptive, haddadpour2019convergence, khaled2020tighter, stich2019error, woodworth2020local, koloskova2020unified, huo2020faster, zhang2020fedpd, pathak2020fedsplit, malinovskiy2020local, li2020federated, Abdelmoniem.EuroMLSys22}. Such heterogeneity leads to unstable and slow convergence \cite{li2020federated}, resulting in suboptimal or even detrimental model performance \cite{subopt}. This occurs because the data distributions on the clients may differ significantly from the global distribution, causing clients to converge towards their local optima rather than the global optimum, refer to Appendix B \cite{fourati2023filfl} for more details. Furthermore, given communication constraints, training with all clients may not be possible; previous works have considered client selection schemes that select a subset $\mathcal{A}_t$ of $K$ clients from a total of $N$ clients to participate at each training round $t$. Although client selection methods address communication constraints and make the training more practical, they also increase the challenge of managing heterogeneity. Refer to Appendix B \cite{fourati2023filfl} for an extended related work. 

\begin{figure}[t]
\centering
\includegraphics[scale=0.5]{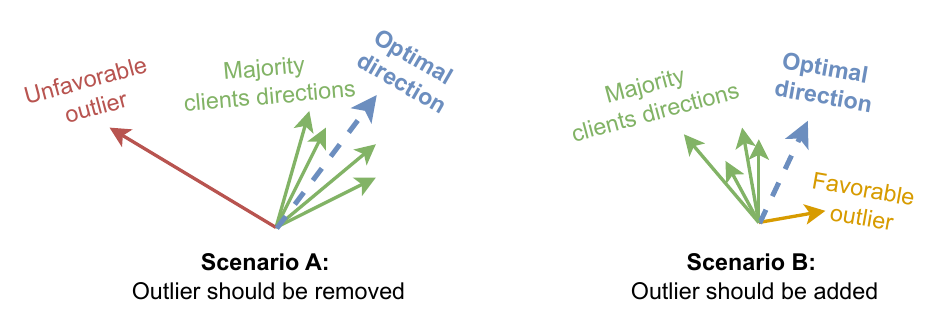}
  \caption{Visualization of two scenarios with different suggested descent directions from different clients. Arrows are color-coded to indicate the quality of direction: blue for optimal, orange for favorable outlier, green for majority consensus, and red for unfavorable outlier.}
\label{fig:vectors}
\vspace{18pt}
\end{figure}

To address the aforementioned FL challenges, various client selection schemes have been proposed in earlier studies. Some aim to provide unbiased estimations of the gradients that would result from full participation, such as sampling based on the number of local data points \cite{li2019convergence} or sampling uniformly at random with weighted updates (RS) \cite{li2020federated}. While these approaches approximate full participation, they are not explicitly designed to accelerate the training process. Other schemes select subsets of clients that carry representative gradient information for full participation by encouraging diverse gradient selections (DivFL) \cite{balakrishnan2021diverse}. However, promoting diversity may also include unfavorable outlier gradients. Additional strategies explicitly aimed at accelerating training include selecting clients with higher update norms more frequently \cite{chen2020optimal} or employing a power-of-choice (PoC) method that biases selection towards clients with higher local losses \cite{cho2020client}. However, these approaches consider clients separately rather than as part of a collaborative unit, i.e., they make decisions based on individual performances without considering their collaborative performance at the current stage of the training process. 


%
%

Assessing clients based on their collaborative performance is essential to optimize client participation beyond mere element-wise selection. Considering gradients from collaboratively-unfavorable clients or excluding collaboratively-favorable ones can lead to degraded collaborative performance. To illustrate this, consider the simplified example depicted in Fig. \ref{fig:vectors}, where we illustrate two possible client combination scenarios. In these scenarios, the blue arrow represents an oracle for the optimal descent direction. In scenario A (Fig. \ref{fig:vectors}), the red arrow, having a significantly different direction and larger norm than most other directions, might be selected by methods that prioritize directions with larger norms or that encourage diversity. However, excluding the red arrow and keeping the green arrows can lead to a better approximation of the optimal descent direction. In scenario B (Fig.\ref{fig:vectors}), while the orange direction differs from the majority of directions (green) and may be overlooked by methods that rely on similarity metrics between gradients or due to its small norm, its inclusion—based on its contribution to the subset of green directions—leads to a better approximation of the optimal direction. These examples highlight the importance of assessing collaboration when choosing clients.

In this work, we include combinatorial optimization in the standard FL training to optimize client participation further. We introduce FilFL, which includes a \textit{client filtering} procedure that looks for the best combination of clients within the available ones, which can be conducted as a periodic prepossessing step to any off-the-shelf client selection scheme. To achieve this, we formulate a combinatorial optimization problem to periodically identify the clients most compatible for collaboration. Namely, our objective is to identify the optimal subset of available clients whose averaged performance yields the lowest loss. Solving this combinatorial optimization problem would necessitate an exponential number of tests, rendering it computationally infeasible. As a result, we employ an efficient greedy approach to approximate its solution. To this end, we present two greedy filtering algorithms: a deterministic one and a randomized variant, both relying on marginal gains from adding and removing clients from subsets of clients. Using different vision and language tasks and realistic federated scenarios with time-varying client availability, we evaluate the performance of combining our \textit{client filtering} methods with different FL algorithms, such as FedAvg and FedProx~\cite{li2020federated}, and with various client selection schemes, such as RS, PoC, and DivFL.

\vspace{3pt}
\smartparagraph{Contributions.} 
We propose FilFL, a novel approach that includes combinatorial optimization through \textit{client filtering} in FL to optimize client participation, accelerate the training process, and improve the overall global model performance.
 To the best of our knowledge, we are the first to define a non-monotone combinatorial optimization problem in the context of FL, aiming to identify the subset of clients from the available clients whose averaged performance yields the lowest loss. We propose a greedy filtering algorithm ($\chi$GF) with deterministic (DGF) and randomized (RGF) versions to approximate its solution.
 We provide a theoretical analysis showing that FilFL achieves a convergence rate of $\mathcal{O}(\frac{1}{t}) + \mathcal{O}(\varphi)$ for $t$ time steps, where $\varphi$ represents a time constant, under certain assumptions.
Empirical evaluations on various vision and language tasks under realistic scenarios of time-varying available clients show that FilFL outperforms FL methods, achieving faster training and up to a 10 percentage point increase in test accuracy. Furthermore, ablation studies and filtering performance analysis have been conducted.

\vspace{3pt}
A companion report of this paper with complete technical details is available at \cite{fourati2023filfl}. The code can be accessed at \url{https://github.com/salmakh1/FilFL}.


\section{Problem Formulation}
\label{Problemformulation}

Unlike standard FL training algorithms, where all the available clients are considered for selection and participation, we formulate a bi-level optimization problem that combines the standard continuous training objective with a discrete filtering objective.

\subsection{Training Objective}

We consider the canonical objective of fitting a global model to the non-IID data $\mathcal{D}$ held across clients \cite{mcmahan2017communication}. Thus, we consider the following distributed optimization problem:
\begin{equation}
\label{FD}
  \min _{\mathbf{w}}\left\{F(\mathbf{w}) \triangleq \sum_{k=1}^N p_k F_k(\mathbf{w})\right\},
\end{equation}
where $N$ is the number of clients, and $p_k$ is the weight of the $k$-th client such that $p_k \geq 0$ and $\sum_{k=1}^N p_k=1$. Suppose the $k$-th client holds the $m_k$ training data: $x_{k, 1}, x_{k, 2}, \cdots, x_{k, m_k}$. The local objective $F_k(\cdot)$ is defined as:
$F_k(\mathbf{w}) \triangleq \frac{1}{m_k} \sum_{j=1}^{m_k} \ell\left(\mathbf{w} ; x_{k, j}\right)$
where $\ell(\cdot ; \cdot)$ is some training loss function. While the training objective seeks the best client weights, the filtering objective finds the best combination of clients to optimize these weights. Although the former is continuous and the latter is discrete, both are interconnected and combined, which have led to remarkable improvements.

\subsection{Filtering Objective}

Our filtering objective is to find a subset of clients $\mathcal{S}^f$ that approximates a solution to the following combinatorial optimization problem:
\begin{equation}
\label{reward_}
\max_{\mathcal{S} \in \mathcal{S}_t} \quad  -F\left(\mathbf{\frac{1}{|\mathcal{S}|} \sum_{k \in \mathcal{S}} \mathbf{w}_t^k}\right),
\end{equation}
such that $\mathbf{w}_t^k$ is the weight of the $k^{\text{th}}$ client in round $t$. Thus, the combinatorial problem aims at finding a subset $\mathcal{S}^f \in \mathcal{S}_t$ where the average of the weights of the clients in the subset $\mathcal{S}^f$ minimizes the weighted average of the local losses, i.e., maximizes the function $-F$. Following the literature on combinatorial optimization, we define the problem as a maximization problem.

Unfortunately, solving the problem defined in Eq. \eqref{reward_} is both communication and computationally expensive. Even the evaluation of one possible set of clients $\mathcal{S}$ requires all clients to evaluate the combination of that set, i.e., each client $k$ needs to compute,  $F_k(\mathbf{\frac{1}{|\mathcal{S}|} \sum_{k \in \mathcal{S}} \mathbf{w}_t^k})$ on their local datasets. Finding or even approximating a solution requires several evaluations, which introduces additional communication and computational overhead on the participating clients. 

To make this approach more practical, we propose reformulating the problem into a centrally solvable form, thereby minimizing communication overhead. Therefore, we suggest using a central filtering dataset, denoted by $\mathcal{V}$, without requiring the clients to share any datasets. This can be done in several ways, by leveraging a subset of the server's validation data for filtering, using samples from a public dataset,\footnote{Previous works in FL have used public datasets for various purposes \cite{huang2022learn, zhang2021parameterized, lin2020ensemble, cheng2021fedgems, li2020federated}.} or randomly choosing a client to perform filtering on a subset of their validation dataset, in each filtering round. We later show that these approaches, solving on a server dataset or a variable filtering dataset, depending on the chosen client (see \cref{variable} for details about the stochastic dataset), are possible and show that the filtering dataset can be stochastic, and does not need to adhere to any prohibitive requirements, for example, can be as small as 8 samples, as discussed in detail in \cref{sensitivityanalysis}. 

Unless mentioned otherwise, in the following, we consider a server-held filtering dataset $\mathcal{V}$ with $m$ samples: $x_{1}, x_{2}, \cdots, x_{m}$. Thus, our filtering objective can be defined as follows:

\begin{equation}
\label{reward}
\max_{\mathcal{S} \in \mathcal{S}_t} \left\{\mathcal{R}(\mathcal{S}) \triangleq   -F_{\mathcal{V}}\left(\mathbf{\frac{1}{|\mathcal{S}|} \sum_{k \in \mathcal{S}} \mathbf{w}_t^k}\right)\right\},
\end{equation}
where 
$F_{\mathcal{V}}(\mathbf{w}) \triangleq \frac{1}{m} \sum_{j=1}^{m} \ell\left(\mathbf{w} ; x_{j}\right)
$ as the loss on dataset $\mathcal{V}$.  

While the reformulation proposed in Eq. \eqref{reward} of the objective in Eq. \eqref{reward_} offers improved tractability, saving communication and computation when evaluated centrally, achieving an exact solution remains non-trivial.
 Finding an exact solution to the problem in Eq. (\ref{reward}) would typically still necessitate an exponential number of queries, rendering it computationally infeasible. Furthermore, notice that the function in Eq. (\ref{reward}) is not necessarily monotone\footnote{A function $f$ is monotone, if any set $A$ is a subset of $B$ ($A \subseteq B$), then $f(A) \leq f(B)$ \cite{fourati2024combinatorial}.}.
Suppose we have a set of clients $A$ and a new client $c$. If the new client $c$ has a high loss, adding $c$ to the set of clients $A$ may increase the overall loss, thereby decreasing the objective value $\mathcal{R}(A \cup \{c\})$ compared to $\mathcal{R}(A)$, thus violating monotonicity of the function. Thus, we seek to devise a non-monotone approximation algorithm to solve this problem efficiently. 

\begin{figure*}[t]
\centering 
\includegraphics[scale=0.7]{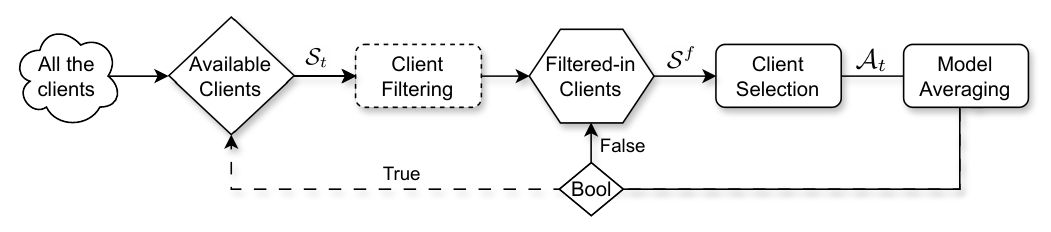}
\caption{FilFL incorporates \textit{client filtering} in FL, which is activated when the boolean condition 'Bool' becomes true, either when new clients become available or when $h$ rounds have elapsed since the last filtering call. Otherwise, the condition remains false. 
In both scenarios, clients are selected from the filtered-in subset of clients, denoted as $\mathcal{S}^f$.}
\label{fig:diagram}
\end{figure*}

\section{Client Filtering}
\label{ClientFilteringSection}

We introduce our approach, FilFL, which incorporates \textit{client filtering} into standard FL algorithms such as FedAvg and FedProx, alongside with different client selection algorithms, such as RS, PoC, and DivFL. FilFL filters the available clients, considering only the filtered-in clients $\mathcal{S}^f$ as potential participants in the training process. This ensures that the chosen client selection method is only applied to the chosen subset $\mathcal{S}^f$, rather than the entire pool of available clients $\mathcal{S}_t$. To implement \textit{client filtering}, we define a combinatorial objective function on the discrete and large space of client combinations in Eq. \eqref{reward} and introduce a periodic greedy algorithm denoted as $\chi$GF, which approximates a solution for this objective, optimizing client combinations for better client participation in FL.

\subsection{Client Filtering in FL (FilFL)}


\begin{algorithm}[t]
\small
\caption{FilFL}\label{alg:FilFL}
\begin{algorithmic}[1]
\REQUIRE{ $T$, $E$, $\eta$, $\mathbf{w}_{1}$, $K$, $\mathcal{S}_{0}$, $h$, $n$, $\chi$}
\STATE \textbf{Initialize} $\mathcal{S}^f \leftarrow \mathcal{S}_{0}$
\FOR{t = $1$, $\cdots$, $T$}
    \IF{$(t\mod h == 0)$ OR $(\mathcal{S}_t \neq \mathcal{S}_{t-1})$}
        \STATE Server broadcasts $\mathbf{w}_{t}$ to all clients in $\mathcal{S}_t$
        \FOR{client $k \in \mathcal{S}_t$ \textbf{in parallel}}
            \STATE Update $\mathbf{w}^k$ for $E$ local SGD steps
            \STATE Send $\mathbf{w}^k$ back to the server
        \ENDFOR
        \STATE $\mathcal{S}^f$, $\mathcal{A}_t$ = \textit{client filtering}$(\text{Shuffle}(\mathcal{S}_t), n, \chi)$
    \ELSE
    \STATE Server selects $\mathcal{A}_t$ including at most $K$ clients from $\mathcal{S}^f$
    \STATE Server broadcasts $\mathbf{w}_{t}$ to all clients in $\mathcal{A}_t$
    \FOR{client $k \in \mathcal{A}_t$ \textbf{in parallel}}
        \STATE Update $\mathbf{w}^k$ for $E$ local SGD steps
        \STATE Send $\mathbf{w}^k$ back to the server
\ENDFOR
    \ENDIF

\STATE Server aggregates:
\STATE \quad $\mathbf{w}_{t+1} \leftarrow \frac{1}{|\mathcal{A}_t|} \sum_{k \in \mathcal{A}_t} \mathbf{w}^k$
\ENDFOR
\end{algorithmic}
\end{algorithm}

FilFL is a FL approach that incorporates \emph{client filtering}. Algorithm \ref{alg:FilFL} presents its pseudocode. FilFL applies \emph{client filtering} (line 4) whenever the current set of available clients differs from the previous round. Furthermore, to improve computational efficiency, FilFL applies \emph{client filtering} periodically every $h$ rounds. We empirically observe similar results when running $\chi$GF every round or running it every few rounds; a sensitivity analysis to $h$ is given in Section \ref{sensistivity_h_main_section}. The \textit{client filtering} procedure (cf. Algorithm \ref{alg:GF}) determines $\mathcal{S}^f$ by approximating a solution for the problem defined in Eq. (\ref{reward}). To determine the set of active clients $\mathcal{A}_t$, FilFL uses any client selection method to select $K$ clients from $\mathcal{S}^f$ (line 6).
In case $\mathcal{S}^f$ only contains $K$ or fewer clients, FilFL uses $\mathcal{S}^f$ as $\mathcal{A}_t$ (line 6). FilFL then runs local steps of SGD for each active client in $\mathcal{A}_t$ (lines 8-11). Finally, the server aggregates the weights returned from the active clients and moves to the next round.

\begin{remark}
FilFL generalizes standard FL.
FilFL adds an extra layer in FL, which is \textit{client filtering}. Using an identity filtering algorithm that accepts all the available clients, i.e., $\mathcal{S}^f = \mathcal{S}_t$, FilFL reduces to standard FL training schemes. Thus, FilFL can be considered as a generalization of those. In this paper, we propose $\chi$GF for filtering. However, future work might consider other filtering methods.
\end{remark}

\begin{remark}
  Client filtering and client selection are distinct yet complementary methods with key differences. First, client filtering does not produce a subset with a fixed cardinality, $K$; therefore, client selection is subsequently applied to the filtered-in group. Second, client filtering can be implemented periodically, whereas client selection occurs in every communication round.  Finally, we opted to separate the two for the sake of generality, allowing the flexibility to combine any filtering algorithm with any off-the-shelf selection method.  
\end{remark}

\begin{remark}
FilFL reduces the complexity of client selection schemes. Firstly, FilFL skips client selection whenever $|\mathcal{S}^f|\leq K$ (line 6). Furthermore, \textit{client filtering} often leads to the rejection of multiple clients. As a result, when FilFL applies client selection on the filtered-in set $\mathcal{S}^f$ instead of the full set of available clients $\mathcal{S}_t$, the search space for client selection becomes smaller. For instance, the DivFL selection method complexity is $\mathcal{O}(N\mathcal{G}(N)K)$, where $N$ represents the number of all the clients, $K$ is the cardinality constraint, and $\mathcal{G}(N)$ represents the cost of calling their oracle function, which is a linearly increasing function of $N$. Consequently, the complexity of DivFL is $\mathcal{O}(N^2K)$. However, by incorporating $\chi$GF
with DivFL, the selection complexity is reduced to $\mathcal{O}(|\mathcal{S}^f|^2K)$, with $|\mathcal{S}^f|$ the number of filtered-in clients typically being smaller than $n$, smaller than $N$. 
\end{remark}


\subsection{Greedy Filtering ($\chi$GF)}

\begin{algorithm}[t]
\small
\caption{$\chi$GF ($\chi \in \{D, R\}$)}
\label{alg:GF}
\begin{algorithmic}[1]
\REQUIRE $\mathcal{S}_t$, $n$, $\chi$
    \STATE \textbf{Initialize} $X_{0} \leftarrow \emptyset,$ $Y_{0} \leftarrow \mathcal{S}_t$
    \FOR{index $i \in \{1, ... , n\}$}
        \STATE $u_i \leftarrow$ client of index $i$ in $\mathcal{S}_t$
        \STATE $a_i \leftarrow \mathcal{R}(X_{i-1} \cup \left\{u_i\right\}) - \mathcal{R}(X_{i-1})$ 
        \STATE $b_i \leftarrow \mathcal{R}(Y_{i-1} \setminus \left\{u_i\right\}) - \mathcal{R}(Y_{i-1})$
        \STATE $a_i^{\prime} \leftarrow \max(a_i, 0)$ and  $b_i^{\prime} \leftarrow \max(b_i, 0)$
        \IF {$\chi = D$} 
        \STATE $\quad$ $p_i = \mathbf{1}\{ a_i > b_i \}$
        \ELSIF{$\chi = R$}
        \STATE $\quad$ $p_i = \frac{a_i^{\prime}}{a_i^{\prime} + b_i^{\prime}}$  ($p_i = 1 \textbf{ if } a_i^{\prime} =  b_i^{\prime} =  0$ )
        \ENDIF
        \STATE \textbf{with probability} $p_i$ \textbf{do} 
        \STATE $\quad$ $X_i \leftarrow X_{i-1} \cup \left\{u_i\right\}$ and $Y_i \leftarrow Y_{i-1}$
        \STATE \textbf{else}
        \STATE $\quad$ $Y_i \leftarrow Y_{i-1} \setminus \left\{u_i\right\}$ and $X_i \leftarrow X_{i-1}$
    \ENDFOR
    \STATE Select $Z$ including at most $K$ clients from $X_n$
    \STATE \textbf{Return} $X_n$, $Z$
\end{algorithmic}
\end{algorithm}

Motivated by the successful application of greedy algorithms in combinatorial optimization \cite{feige2011maximizing, buchbinder2015tight, fourati2023randomized, fourati2024combinatorial}, we introduce a greedy \textit{client filtering} algorithm, called $\chi$GF. While monotone approximation algorithms, greedily adds elements based on their adding marginal gains \cite{fourati2024combinatorial}, non-monotone algorithms considers both the marginal gain of adding and the marginal gain of removing the same entity \cite{feige2011maximizing, buchbinder2015tight, fourati2023randomized}. Adapting the non-monotone algorithm in \cite{fourati2023randomized}, which has been demonstrated to be robust to small errors in function evaluations, as shown in Corollary 2 in \cite{fourati2023randomized}, we propose two versions for filtering: randomized (RGF) and deterministic (DGF). Algorithm~\ref{alg:GF} lists their pseudocode. The algorithm iterates over each available client and decides whether to add it to the set of clients $X$ (initially empty) or remove it from the set of clients $Y$ (initially containing all available clients). The server determines $X$ and $Y$ in a greedy fashion using measures of marginal gains of adding and removing until a decision is made for all individual clients. The algorithm returns the chosen (filtered-in) set of clients. Specifically, let $X_i$ and $Y_i$ be two sets of clients. Initially, $X_0 = \emptyset$ and $Y_0 = \mathcal{S}_t$. The algorithm has at most $n$ steps, where $n$ is the maximum number of considerable clients. In step $i$, $\chi$GF computes two variables: $a_i$ and $b_i$, defined as follows:
\begin{equation}
\begin{aligned}
    &a_i \triangleq \mathcal{R}(X_{i-1} \cup \left\{u_i\right\}) - \mathcal{R}(X_{i-1}), \\
    &b_i \triangleq  \mathcal{R}(Y_{i-1} \setminus \left\{u_i\right\}) - \mathcal{R}(Y_{i-1}).
\end{aligned}\label{eq:abdefn}
\end{equation}
These two variables are important for the decision-making process. $a_i$ measures the marginal gain of adding client $u_i$ to $X_{i-1}$, while $b_i$ measures the marginal gain of removing client $u_i$ from $Y_{i-1}$, which can be positive due to non-monotonicity. While DGF decides by comparing both marginal gains via $p_i = \mathbf{1}\{ a_i > b_i \}$, RGF decides based on $p_i = \frac{a_i^{\prime}}{a_i^{\prime} + b_i^{\prime}}$, where $a_i^{\prime} = \max(a_i, 0)$ and $b_i^{\prime} = \max(b_i, 0)$. In the special case when $a_i^{\prime} =  b_i^{\prime} =  0$, we set $p = 1$ for RGF. With probability $p$, the client $u_i$ is added to the set $X_{i-1}$ and kept in $Y_{i-1}$; otherwise, the client is removed from $Y_{i-1}$ and $X_{i-1}$ is kept the same. Therefore, $X_i \subseteq Y_i$ for all $i=1,\dots,n$. After checking all $n$ clients, it can be easily seen that by the algorithm's construction, both sets $X_n$ and $Y_n$ contain the same clients, i.e., $X_n \equiv Y_n$. 
Hereafter, at round $t$, we refer to the final set $X_n$ as the filtered-in set $\mathcal{S}^f$.

\begin{remark}
\label{rejection_rate_remark}
In cases where both $a_i$ and $b_i$ are non-positive, i.e., $a_i^{\prime} = b_i^{\prime} = 0$, the RGF algorithm accepts the client with a probability of $1$. On the other hand, even when both $a_i$ and $b_i$ are non-positive, the DGF algorithm may reject this client with a probability of $1$ if $a_i < b_i$. Hence, by design, DGF can reject more clients than RGF. This observation is empirically validated in Fig.~\ref{fig:filteredsize}. Generally, the clients that are accepted by RGF and rejected by DGF have minimal impact on FilFL performance, as they are the ones with both negative marginal gains of adding them to $X_{i-1}$ or removing them from $Y_{i-1}$.
\end{remark}

\begin{remark}
The computational complexity of using $\chi$GF is $\mathcal{O}(n \mathcal{I}(m))$, where $n$ is the number of considerable available clients, fixed by the user, and $\mathcal{I}(m)$ is the cost of inference over the server dataset of size $m$ data points. Therefore, the computational cost of using the $\chi$GF algorithm does not scale with the scaling number of clients and increases only linearly with the number of considered available clients $n$ (for reference, DivFL's computational cost scales quadratically with the total number of clients $N$). Therefore, our method remains practical even as the number of clients increase. Furthermore, the cost of forward passes can be reduced by distributing the computation across multiple graphical processing units, leading to faster and more efficient computations.
\end{remark}

\section{FilFL Convergence Analysis}
\label{ConvAnalysisSection}

We now provide a theoretical analysis of the convergence properties of our proposed FilFL algorithm (see Algorithm \ref{alg:FilFL}). Specifically, we analyze the convergence of the average model weights $\Bar{\mathbf{w}}_t$ at round $t$ to the optimal solution $\mathbf{w}^*$, under practical assumptions of non-IID data, partial client participation, and local updates. Our analysis focuses on the impact of incorporating \textit{client filtering} into the FedAvg setting, assuming random sampling as the client selection method. While our results mainly apply to FedAvg with random sampling, they can be easily extended to other methods. In the following, we provide the necessary definitions and assumptions for our analysis and present the theorem statement for convergence. The proofs of the main lemmas are provided in Appendix D \cite{fourati2023filfl}.

\subsection{Assumptions and Definitions}
The following assumptions are standard assumptions for the convergence analysis in the literature of FL, such as \cite{ balakrishnan2021diverse, li2019convergence}.

\begin{assumption}
\label{ass1}
$F_1, \cdots, F_N$ are all $L$-smooth\footnote{For all $k$,$\mathbf{v}$ and $\mathbf{w}, F_k(\mathbf{v}) \leq F_k(\mathbf{w})+(\mathbf{v}-$ $\mathbf{w})^T \nabla F_k(\mathbf{w})+\frac{L}{2}\|\mathbf{v}-\mathbf{w}\|_2^2$.}. 
\end{assumption}
\begin{assumption}
\label{ass2}
$F_1, \cdots, F_N$ are all $\mu$-strongly convex\footnote{For all $k$, $\mathbf{v}$ and $\mathbf{w}, F_k(\mathbf{v}) \geq F_k(\mathbf{w})+(\mathbf{v}-$ $\mathbf{w})^T \nabla F_k(\mathbf{w})+\frac{\mu}{2}\|\mathbf{v}-\mathbf{w}\|_2^2$.}. 
\end{assumption}

\begin{assumption}
\label{ass3}
Let $\psi_t^k$ be sampled from the $k$-th client's local data uniformly at random. The variance of stochastic gradients in each client is bounded by $\sigma_k^2$, i.e., $\mathbb{E}\left[\left\|\nabla F_k\left(\mathbf{w}_t^k, \psi_t^k\right)-\nabla F_k\left(\mathbf{w}_t^k\right)\right\|^2 \right]\leq \sigma_k^2$ for all $k=1, \cdots, N$. 
\end{assumption} 

\begin{assumption}
\label{ass4}
The norms of the stochastic gradients are uniformly bounded by $G$, i.e., $\left\|\nabla F_k\left(\mathbf{w}_t^k, \psi_t^k\right)\right\|^2 \leq G^2$ for all $k=1, \cdots, N$ and $t=1, \cdots, T-1$.
\end{assumption}

 \begin{assumption}
\label{ass5}
Statistical heterogeneity: $F^*-\sum_{k \in[N]} p_k F_k^*$ is bounded, where $F^*:=\min _\mathbf{w} F(\mathbf{w})$ and $F_k^*:=\min _\mathbf{v} F_k(\mathbf{v})$.
\end{assumption}

\begin{assumption} 
\label{samplingassumptio}
Assume $\mathcal{A}_t$ contains a subset of $K$ indices randomly selected with replacement according to the sampling probabilities $p_{i}^{'}=1/|\mathcal{S}^f|$, with simple averaging for aggregation \footnote{A theoretical analysis of this sampling scheme was provided in \cite{li2019convergence}.}.
\end{assumption}

Limited to realistic scenarios (for communication efficiency and low straggler effect), FilFL samples a subset $\mathcal{A}_t$ from the filtred-in set $\mathcal{S}^f$ and then only performs updates on them. This makes the analysis intricate since $\mathcal{A}_t$ varies each $E$ steps. However, we can use an approach similar to the one used in \cite{li2019convergence} to circumvent this difficulty. We assume that FilFL activates all clients at the beginning of each round and then uses the parameters maintained in only a few sampled clients to produce the next-round parameter. It is clear that this updating scheme is equivalent to the original. 

Let $\mathbf{w}_t^k$ be the model parameter maintained in the $k$-th client at the $t$-th step. Let $\mathcal{I}_E$ be the set of global synchronization steps, i.e., $\mathcal{I}_E=\{iE \mid i=1,2, \cdots\}$. If $t+1 \in \mathcal{I}_E$, i.e., the time step to communication, FilFL activates all clients. Then, the update of our algorithm can be described as: for all $k \in[N]$,
$$
\begin{aligned}
 & \mathbf{v}_{t+1}^k=\mathbf{w}_t^k-\eta_t \nabla F_k\left(\mathbf{w}_t^k, \psi_t^k\right), \\
& \mathbf{w}_{t+1}^k= \begin{cases}\mathbf{v}_{t+1}^k & \text { if } t+1 \notin \mathcal{I}_E, \\ 
\\
\text { sample } \mathcal{A}_{t+1} \text { from } \mathcal{S}^f_{t+1} \\ 
\text { and average }\left\{\mathbf{v}_{t+1}^k\right\}_{k \in \mathcal{A}_{t+1}} & \text { if } t+1 \in \mathcal{I}_E .\end{cases} 
\end{aligned}
$$
Let $\mathbf{w}^* \in \arg \min _\mathbf{w} F(\mathbf{w})$ and $\mathbf{v}_k^* \in \arg \min _\mathbf{v} F_k(\mathbf{v})$ for $k \in[N]$. 
Let $\bar{\mathbf{v}}_t \triangleq \sum_{k \in[N]} p_k \mathbf{v}_t^k$, and $\bar{\mathbf{w}}_t \triangleq \sum_{k \in[N]} p_k \mathbf{w}_t^k$,
where $p_k \geq 0$ is the given weight of the $k^{\text {th }}$ client and w.l.o.g., we assume $\sum_k p_k=1$. 

Filtering the clients before selection, using biased greedy filtering algorithms, made the theoretical analysis more challenging. Compared to previous theoretical federated convergence analysis, such as \cite{li2019convergence} and \cite{balakrishnan2021diverse}, that introduce $\bar{\mathbf{v}}_t$ and $\bar{\mathbf{w}}_t$, to proceed with our analysis we introduce an extra variable $\bar{\textbf{z}}_t$, defined as follows
$\bar{\mathbf{z}}_t \triangleq \frac{1}{|\mathcal{S}_t^{f}|}\sum_{k \in \mathcal{S}_t^{f}} \mathbf{v}_t^k$. Furthermore, we define a filtering gap as follows: 
\begin{equation}
    \label{epsilont}
    \delta_t \triangleq F(\bar{\mathbf{v}}_t) - F(\bar{\mathbf{z}}_t).
\end{equation}

An optimal filtering method leads to the highest $\delta_t$ possible at every round $t$. In FilFL, using $\chi$GF as a filtering method, we expect the $\delta_t$ to be optimized over the rounds. In Lemma 1, in Appendix D,
we show that $\mathbb{E}\left[\delta_t\right]$ is lower bounded by a constant $\delta$.
 
\subsection{FilFL Theoretical Convergence Results}

We present our convergence result as follows.

\begin{theorem}
\label{theorem-statement}
Let assumptions \ref{ass1}, \ref{ass2}, \ref{ass3},  \ref{ass4}, \ref{ass5}, and \ref{samplingassumptio} hold, then we have
\begin{equation}
\begin{aligned}
\mathbb{E} [\left\|\overline{\mathbf{w}}_{t+1}-\mathbf{w}^*\right\|^2] 
& \stackrel{}{\leq} \mathcal{O}(\frac{1}{t}) + \mathcal{O}(\varphi)
\end{aligned}
\end{equation}
for some time constant $\varphi$ that depends on the filtering.
\end{theorem}

\begin{proof}
Note that
{
\begin{equation}
\begin{aligned}
\label{maininequality-main-paper}
\mathbb{E}\left[\left\|\overline{\mathbf{w}}_{t+1}-\mathbf{w}^*\right\|^2\right] & = \mathbb{E}\left[\left\|\overline{\mathbf{w}}_{t+1}-\overline{\mathbf{v}}_{t+1}\right\|^2\right] 
+ 
\mathbb{E}\left[\left\|\overline{\mathbf{v}}_{t+1}-\mathbf{w}^*\right\|^2\right] \\
& + 2\mathbb{E}\left[\left\langle \overline{\mathbf{w}}_{t+1} -\overline{\mathbf{v}}_{t+1} , \overline{\mathbf{v}}_{t+1}-\mathbf{w}^*\right\rangle\right].
\end{aligned}
\end{equation}}

We bound the three terms in Eq. (\ref{maininequality-main-paper}).
Using Lemma 4
result, shown in Appendix D,
we have
$
\label{lemma3-main-text}
\mathcal{T}_1 \triangleq \mathbb{E} \left[\| \bar{\mathbf{w}}_{t+1} - \mathbf{\bar{v}}_{t+1}\|^2\right]
\stackrel{}{\leq} \xi,
$
for some constant $\zeta$ and $\xi = \zeta  - \frac{2\delta}{\mu}$. Moreover, using Lemma 1, 2, and 3 in \cite{li2019convergence}, define $\mathcal{T}_2 \triangleq \mathbb{E}\left[\left\|\overline{\mathbf{v}}_{t+1}-\mathbf{w}^*\right\|^2\right]$, we have 
$
\label{lilemmas-main-text}
\mathcal{T}_2 \leq\left(1-\eta_t \mu\right) \mathbb{E}\left[\left\|\bar{\mathbf{w}}_t-\mathbf{w}^*\right\|^2\right]+\eta_t^2 B,
$
for a stepsize $\eta_t$ and some constant $B$. Furthermore, using Corollary 1,
in Appendix D,
we have
$
\mathcal{T}_3 \triangleq
\mathbb{E} \left[ \left\langle \overline{\mathbf{w}}_{t} -\overline{\mathbf{v}}_{t} , \overline{\mathbf{v}}_{t}-\mathbf{w}^*\right\rangle \right] \leq \rho \sqrt{\xi},
$
for some constant $\rho$. 

Define $\Delta_t \triangleq \mathbb{E} \left[\left\|\overline{\mathbf{w}}_{t}-\mathbf{w}^*\right\|^2\right]$, and $\varphi = \xi + 2 \rho \sqrt{\xi}$,  thus
$
\Delta_{t+1} \leq \left(1-\eta_t \mu\right) \Delta_t + \eta_t^2 B + \varphi.
$
With a stepsize, $\eta_t = \frac{\beta}{t}$, for $\beta \geq \frac{1}{\mu}$, the final convergence result follows from Lemma 3 in \cite{mirzasoleiman2020coresets}.
\end{proof}
The above result provides a convergence rate guarantee of $\mathcal{O}(\frac{1}{t})$ for FilFL up to a certain neighborhood of size $\mathcal{O}(\varphi)$, which depends on the \textit{client filtering}. While our approach differs from that of DivFL, we obtain similar theoretical guarantees (albeit with different constants) and better empirical results. Furthermore, our experiments show that FilFL enhances different FL algorithms; see Experiments Section, which includes FedAvg and FedProx. It is worth noting that a good filtering algorithm implies larger values of $\delta_t$ for all $t$, as defined in Eq. (\ref{epsilont}). This, in turn, leads to a larger value of $\delta$, thus smaller $\xi$, hence a smaller value of $\varphi$. Our greedy filtering algorithms are designed to maximize $\delta_t$, thereby minimizing $\varphi$. Empirical results demonstrate that both $\chi$GF accelerate the training and lead to better test accuracy. As discussed in the Experiments section, both versions of $\chi$GF enjoy significantly large approximation ratios of the optimal solution $OPT$, specifically, $\mathcal{R}(\mathcal{S}^f) \geq 0.96 \mathcal{R}(OPT)$, indicating that greedy filtering identifies near optimal combinations of clients over the rounds.

\section{Experiments}
\label{Experiments}
As we are the first to propose \textit{client filtering} in FL, we evaluate the performance of combining $\chi$GF with different FL algorithms, such as FedAvg \cite{mcmahan2017communication} and FedProx \cite{li2020federated} with different client selection schemes, namely, random selection (RS) \cite{li2019convergence}, power-of-choice (PoC) \cite{cho2020client}, and diverse selection (DivFL) \cite{balakrishnan2021diverse}. Moreover,
we conduct ablation studies, analyzing the sensitivity of FilFL to different filtering periodicity values and for various filtering dataset scenarios, including different sizes and distributions, and we examine the
behavior of $\chi$GF.

\subsection{Setup}
We experiment with different vision and language tasks in a range of scenarios. We use Shakespeare dataset \cite{caldas2018leaf}, built from ``The Complete Works of William Shakespeare,'' where each speaking role in every play is considered a different client. The task is a next-character prediction with 80 classes of characters in total. We use a small filtering dataset from a different distribution, specifically consisting of parts of this paper's introduction, as shown in Table 3
in the Appendix \cite{fourati2023filfl}. We use a two-layer LSTM \cite{hochreiter1997long} classifier containing 256 hidden units with an 8-dimension embedding layer. Moreover, we use CIFAR-10 \cite{krizhevsky2009learning} in a non-IID setting with ResNet18 \cite{he2016deep}. We split CIFAR-10 train dataset into private and filtering datasets, where the filtering partition fraction is $0.01$. Similar to existing works \cite{acar2021federated, he2020fedml, yurochkin2019bayesian}, to simulate the non-IID data distribution among clients, we use the Dirichlet distribution Dir($\alpha$), with $\alpha = 0.5$. We use the existing CIFAR-10 test sets as global test sets. 
Furthermore, we use Federated Extended MNIST (FEMNIST) \cite{caldas2018leaf}, which is built by partitioning the data in Extended MNIST \cite{cohen2017emnist, lecun1998mnist} based on the writer of the digit/character. We use the test set as a global test set. Similar to \cite{caldas2018leaf}, we use a model with two convolutional layers followed by pooling and ReLU and a final dense layer with 2048 units. 

In the following experiments, we consider $N$ clients, with only $n$ considerable available ones, with $K$ selected clients, periodicity $h$, and filtering data size $m$.
FilFL samples $\mathcal{A}_t$ from the filtered-in set of clients $\mathcal{S}^f$, while other FL algorithms sample $\mathcal{A}_t$ from the full set of available clients $\mathcal{S}_t$. We experiment with three different seeds and present the averaged results together with the standard deviation. Appendix C
reports further details about the setup.

\subsection{FilFL Outperforms Standard FL Algorithms}
\label{main_filfl_convergence}

FilFL, for any given FL algorithm and any applied client selection algorithm, includes an extra layer of \textit{client filtering} using $\chi$GF. In the following sections, we demonstrate the advantages of adding this extra layer to various combinations of FL algorithms and client selection methods. For the same FL algorithm and client selection, we assess the marginal gain of adding such a filtering step.

\subsubsection{FilFL (FedAvg with $\chi$GF and PoC) vs FedAvg (PoC)}

We compare the performance of FilFL (FedAvg with $\chi$GF) against FedAvg, both using PoC for client selection on different datasets. Fig.~\ref{fig:fedavg_poc_test_acc} illustrates the results of the Shakespeare dataset, with a small filtering dataset from a different distribution; specifically consisting of parts of this paper's introduction (see the filtering dataset in Appendix C.3).
Our results demonstrate that FilFL using DGF or RGF performs significantly better than FedAvg. In particular, as depicted in the left plot, FilFL with both filtering methods accomplishes accelerated training and attains around 10 percentage points higher test accuracy than FedAvg. Furthermore, we conducted the t-test, and the resulting two-tailed p-value was 0.0001, considered extremely statistically significant. After 200 rounds, the right plot displays a lower training loss for FedAvg. In Appendix E.1, 
we present the results on CIFAR-10 and FEMNIST, which exhibit improved training and better test accuracy by 5 and 7 percentage points, respectively.

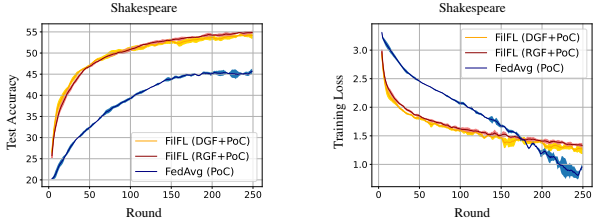
\begin{figure}[!tbp]
\begin{minipage}{0.24\textwidth}
\begin{tikzpicture}
  \node (img)  {\includegraphics[scale=0.23]{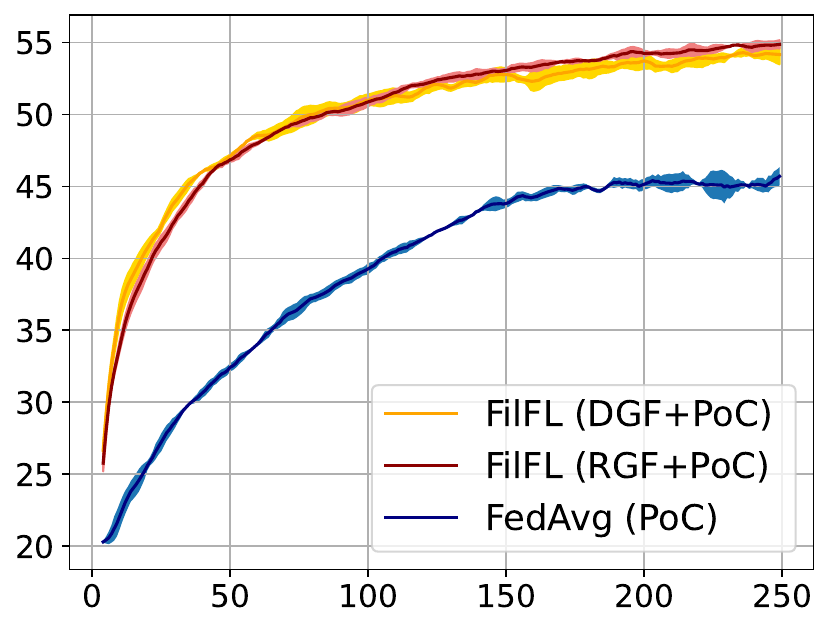}};
  \node[below=of img, node distance=0cm, rotate=0cm, anchor=center,yshift=3.7cm] {\tiny Shakespeare};
  \node[below=of img, node distance=0cm, rotate=0cm, anchor=center,yshift=1cm] {\tiny Round};
  \node[left=of img, node distance=0cm, rotate=90, anchor=center,yshift=-1.0cm] {\tiny Test Accuracy};
 \end{tikzpicture}
\end{minipage}%
\begin{minipage}{0.24\textwidth}
\begin{tikzpicture}
  \node (img)  {\includegraphics[scale=0.23]{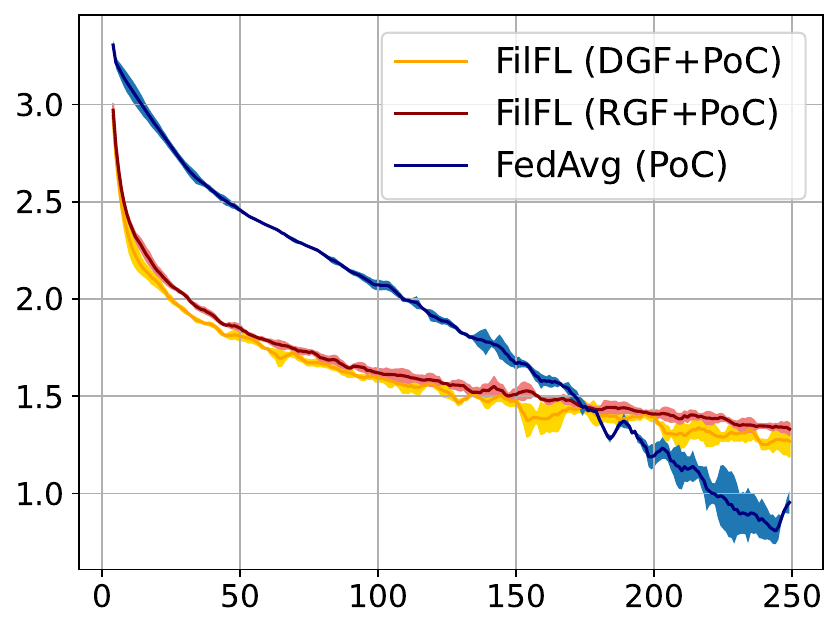}};
  \node[left=of img, node distance=0cm, rotate=90, anchor=center,yshift=-1.0cm] {\tiny Training Loss};
    \node[below=of img, node distance=0cm, rotate=0cm, anchor=center,yshift=3.7cm] {\tiny Shakespeare};
  \node[below=of img, node distance=0cm, rotate=0cm, anchor=center,yshift=1cm] {\tiny Round};
\end{tikzpicture}
\end{minipage}%
\caption{FilFL (FedAvg with $\chi$GF) vs FedAvg (w/o filtering) both with PoC on Shakespeare dataset with $N = 143$, $n=100$, $K=10$, $m=34$, and $h=5$.}
\label{fig:fedavg_poc_test_acc}
\vspace{10pt}
\end{figure}

\subsubsection{FilFL (FedProx with $\chi$GF and RS) vs FedProx (RS)} We compare the performance of FilFL (FedProx with $\chi$GF) against FedProx, both using RS for selection. Fig. \ref{fig:fedprox_femnist} demonstrates that FilFL using $\chi$GF achieves significantly superior performance compared to FedProx on the FEMNIST dataset. Specifically, the left plot illustrates that FilFL with DGF and RGF achieves approximately 7 and 4 percentage points higher test accuracy, respectively than FedProx. The right plot reveals lower training loss for FilFL than FedProx. Moreover, Fig.\ref{fig:fedprox_shakespeare_}, shows the results on the Shakespeare dataset, where FilFL with DGF and RGF attains around 3 and 6 percentage points higher test accuracy, respectively than FedProx.

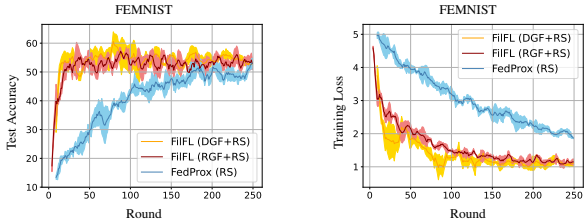
\begin{figure}[!tbp]
\begin{minipage}{0.24\textwidth}
\begin{tikzpicture}
  \node (img)  {\includegraphics[scale=0.23]{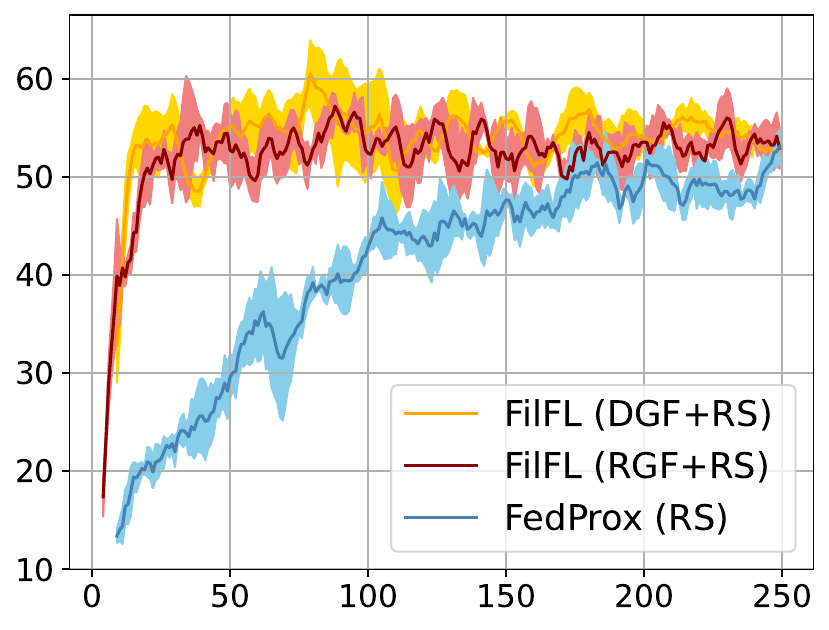}};
  \node[below=of img, node distance=0cm, rotate=0cm, anchor=center,yshift=3.7cm] {\tiny FEMNIST};
  \node[below=of img, node distance=0cm, rotate=0cm, anchor=center,yshift=1cm] {\tiny Round};
  \node[left=of img, node distance=0cm, rotate=90, anchor=center,yshift=-1.0cm] {\tiny Test Accuracy};
 \end{tikzpicture}
\end{minipage}%
\begin{minipage}{0.24\textwidth}
\begin{tikzpicture}
  \node (img)  {\includegraphics[scale=0.23]{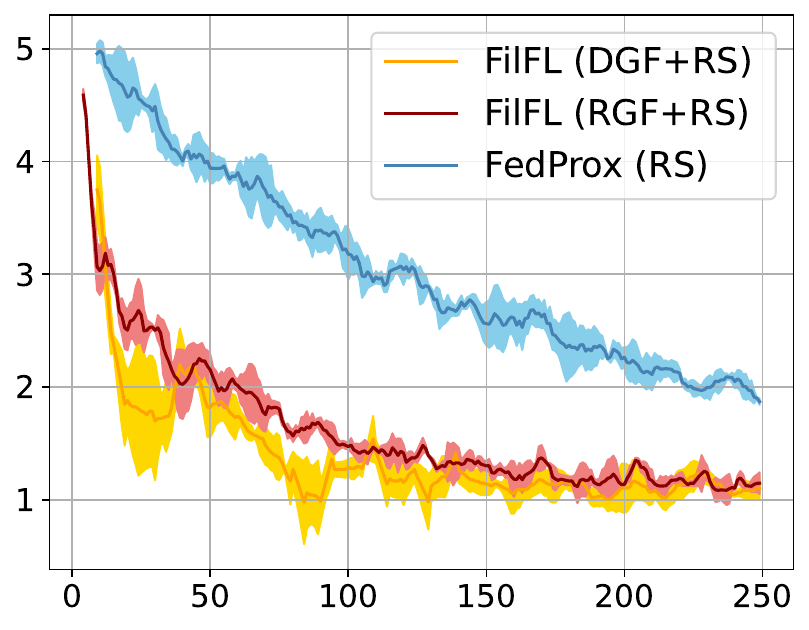}};
  \node[left=of img, node distance=0cm, rotate=90, anchor=center,yshift=-1.0cm] {\tiny Training Loss};
    \node[below=of img, node distance=0cm, rotate=0cm, anchor=center,yshift=3.7cm] {\tiny FEMNIST};
  \node[below=of img, node distance=0cm, rotate=0cm, anchor=center,yshift=1cm] {\tiny Round};
\end{tikzpicture}
\end{minipage}%
\caption{FilFL (FedProx with $\chi$GF) vs FedProx (w/o filtering) both with RS on FEMNIST dataset with $N = 190$, $n=50$, $K=5$, $m=2000$, and $h=5$.}
\label{fig:fedprox_femnist}
\vspace{18pt}
\end{figure}

\begin{figure}[!tbp]
\begin{minipage}{0.24\textwidth}
\begin{tikzpicture}
  \node (img)  {\includegraphics[scale=0.23]{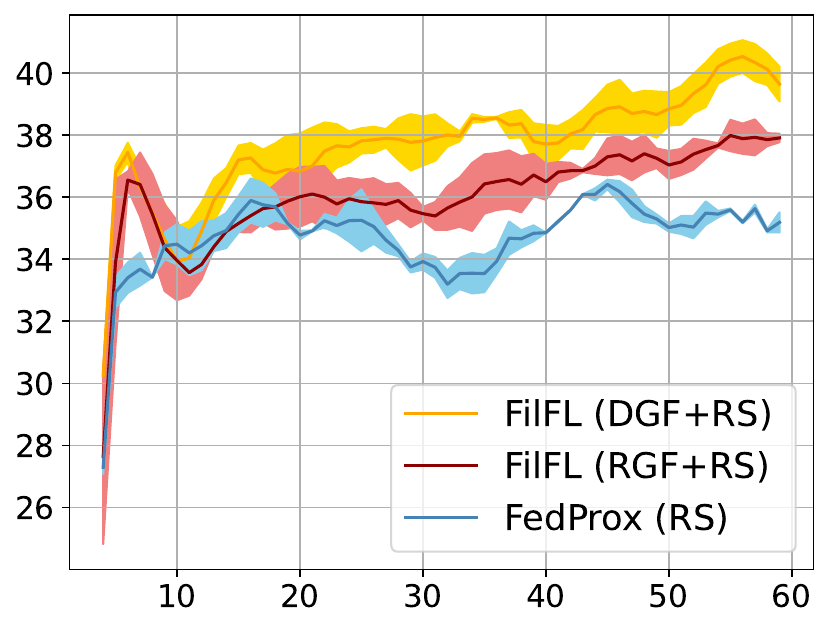}};
  \node[below=of img, node distance=0cm, rotate=0cm, anchor=center,yshift=3.7cm] {\tiny Shakespeare};
  \node[below=of img, node distance=0cm, rotate=0cm, anchor=center,yshift=1.0cm] {\tiny Round};
  \node[left=of img, node distance=0cm, rotate=90, anchor=center,yshift=-1.0cm] {\tiny Test Accuracy};
 \end{tikzpicture}
\end{minipage}%
\begin{minipage}{0.24\textwidth}
\begin{tikzpicture}
  \node (img)  {\includegraphics[scale=0.23]{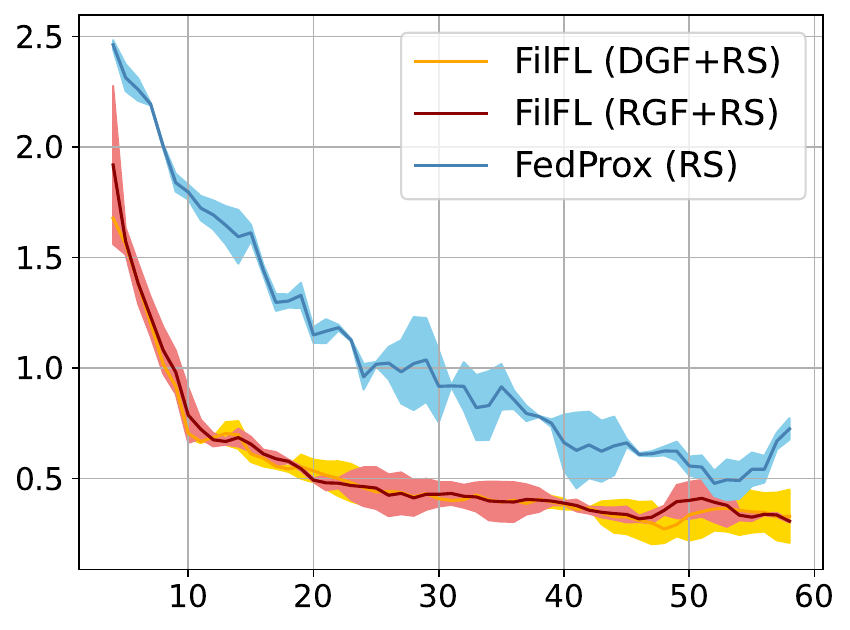}};
  \node[left=of img, node distance=0cm, rotate=90, anchor=center,yshift=-1.0cm] {\tiny Training Loss};
    \node[below=of img, node distance=0cm, rotate=0cm, anchor=center,yshift=3.7cm] {\tiny Shakespeare};
  \node[below=of img, node distance=0cm, rotate=0cm, anchor=center,yshift=1.0cm] {\tiny Round};
\end{tikzpicture}
\end{minipage}%
\caption{FilFL (FedProx + $\chi$GF + RS) vs FedProx (RS) without filtering on Shakespeare dataset.}
\label{fig:fedprox_shakespeare_}
\vspace{16pt}
\end{figure}
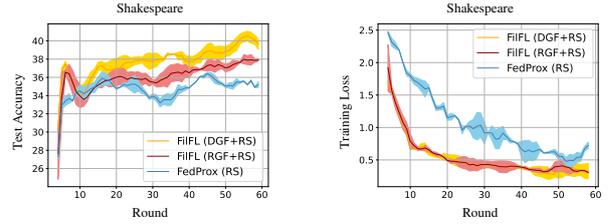

\subsubsection{FilFL (FedAvg with $\chi$GF and RS) vs FedAvg (DivFL).} 
As shown in \cite{balakrishnan2021diverse}, FedAvg with DivFL performs better than FedAvg with RS or PoC. However, it remains computationally more expensive than both selection methods. To investigate whether a simple client selection method like RS combined with $\chi$GF can outperform a sophisticated selection method like DivFL, we compare FilFL using RS against FedAvg (DivFL). Fig.\ref{fig:Divfl} shows that on the CIFAR-10 dataset, $\chi$GF achieves 10 percentage points higher accuracy than FedAvg (DivFL) (left plot). While FedAvg (DivFL) exhibits lower training loss than FilFL (right plot), it suffers from significantly larger test loss (see the Appendix), which can be due to the overfitting of FedAvg (DivFL) and the better generalization capabilities of FilFL. Moreover, our results indicate that although FilFL with RS and FedAvg (DivFL) have similar convergence theoretical results, our approach empirically outperforms FedAvg (DivFL). The two-tailed p-value from the t-test is 0.0018, considered as very statistically significant. In the Appendix, we show that FilFL with DivFL surpasses FedAvg (DivFL).

\begin{figure}[!tbp]
\begin{minipage}{0.24\textwidth}
\begin{tikzpicture}
  \node (img)  {\includegraphics[scale=0.23]{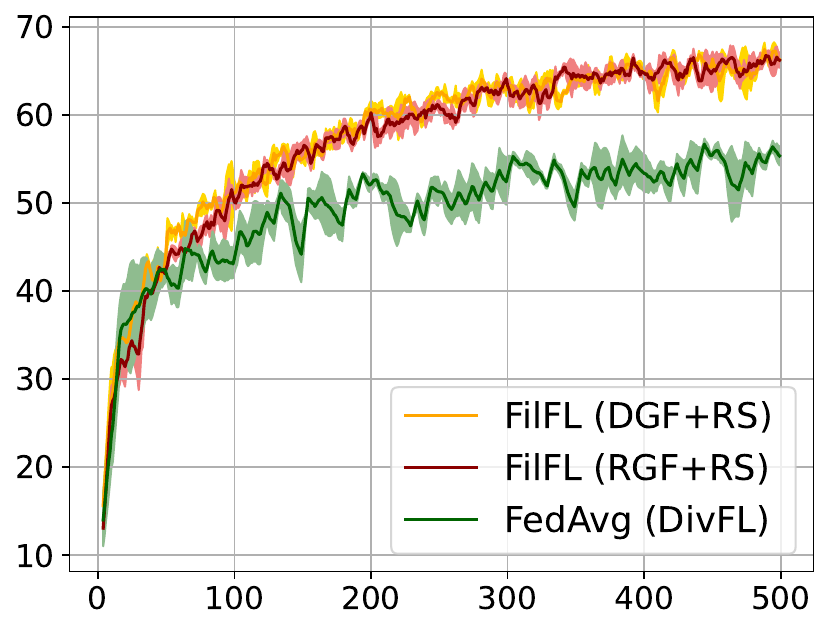}};
      \node[below=of img, node distance=0cm, rotate=0cm, anchor=center,yshift=3.7cm] {\tiny CIFAR-10};
  \node[below=of img, node distance=0cm, rotate=0cm, anchor=center,yshift=1cm] {\tiny Round};
  \node[left=of img, node distance=0cm, rotate=90, anchor=center,yshift=-1.0cm] {\tiny Test Accuracy};
 \end{tikzpicture}
\end{minipage}%
\begin{minipage}{0.24\textwidth}
\begin{tikzpicture}
  \node (img)  {\includegraphics[scale=0.23]{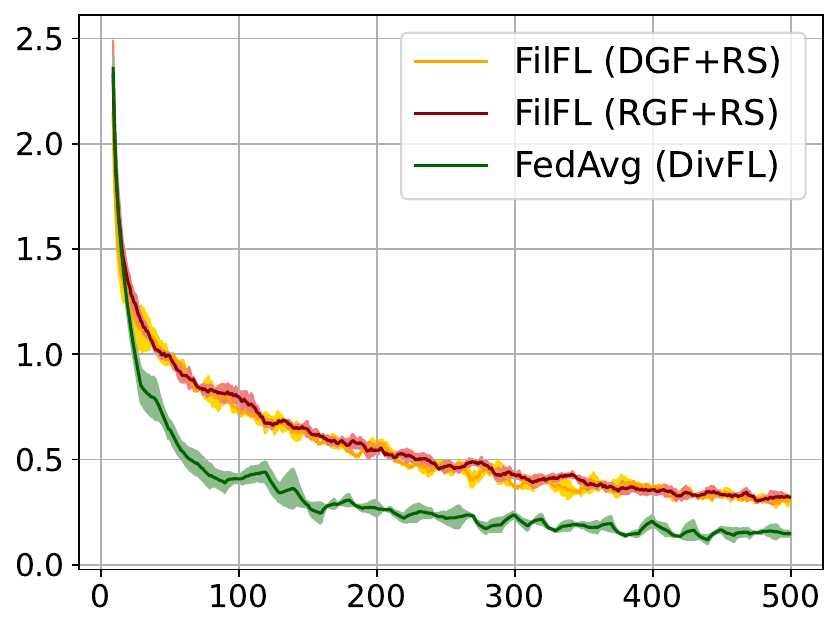}};
        \node[below=of img, node distance=0cm, rotate=0cm, anchor=center,yshift=3.7cm] {\tiny CIFAR-10};
  \node[left=of img, node distance=0cm, rotate=90, anchor=center,yshift=-1.0cm] {\tiny Training Loss};
  \node[below=of img, node distance=0cm, rotate=0cm, anchor=center,yshift=1.0cm] {\tiny Round};
\end{tikzpicture}
\end{minipage}%
\caption{ FilFL (FedAvg with $\chi$GF with RS) vs FedAvg (DivFL w/o filtering) on CIFAR-10 dataset with $N = 200$, $n=30$, $K=3$, $m=500$, and $h=5$.}
\label{fig:Divfl}
\vspace{16pt}
\end{figure}
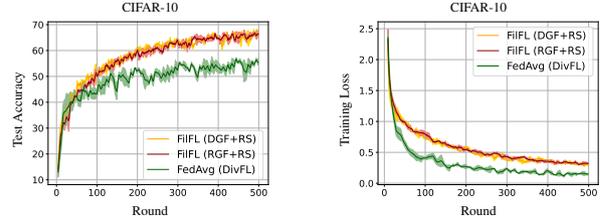

\subsection{Ablation Studies}
\label{sensitivityanalysis}

We conduct an ablation study of the proposed approach, testing the filtering approach with various periodicity, using filtering datasets of different sizes and distributions, and using variable filtering datasets.

\subsubsection{Sensitivity to Filtering Periodicity} 
\label{sensistivity_h_main_section}

The set of available clients may remain the same over several rounds; however, their model weights change due to local training and weight aggregation. This means that \textit{client filtering} in each round may not necessarily exclude the same clients. The optimal set of clients changes significantly as the model weights change over rounds. However, \textit{client filtering} may filter in similar sets of clients for a few rounds when the weights do not change much. That is why we suggest running \textit{client filtering} periodically and applying client selection on the filtered-in set for a few rounds to exploit the set it has already found. We experiment with different periodicities $h \in \{1,3,5,10,20\}$, as shown in Fig.\ref{fig:periodsensitivityfemnist2_main}, and find that FilFL's performance is similar for these values of $h$. However, from a computational perspective, our approach is more efficient for larger periodicity $h$. 

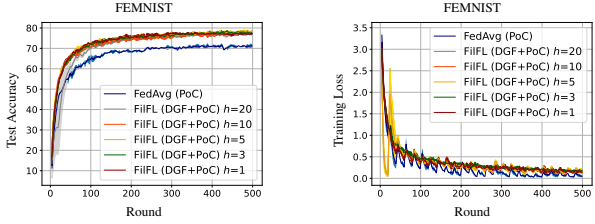
\begin{figure}[!t]
\begin{minipage}{0.24\textwidth}
\begin{tikzpicture}
  \node (img)  {\includegraphics[scale=0.23]{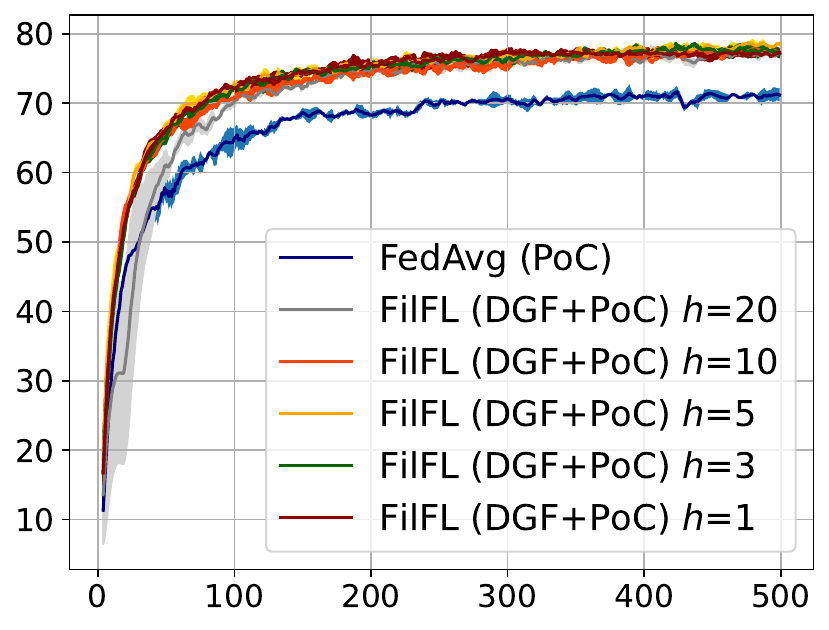}};
  \node[below=of img, node distance=0cm, rotate=0cm, anchor=center,yshift=3.7cm] {\tiny FEMNIST};
  \node[below=of img, node distance=0cm, rotate=0cm, anchor=center,yshift=1.0cm] {\tiny Round};
  \node[left=of img, node distance=0cm, rotate=90, anchor=center,yshift=-1.0cm] {\tiny Test Accuracy};
 \end{tikzpicture}
\end{minipage}%
\begin{minipage}{0.24\textwidth}
\begin{tikzpicture}
  \node (img)  {\includegraphics[scale=0.23]{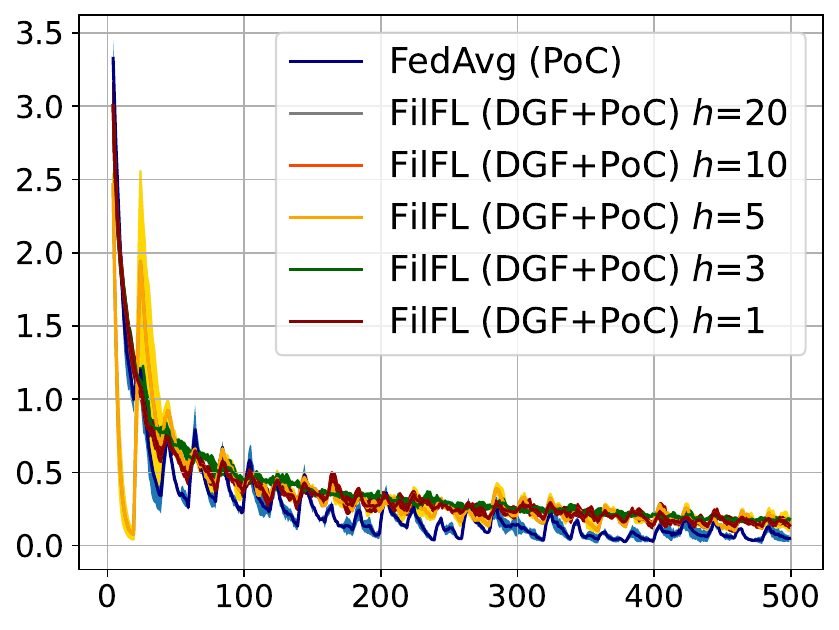}};
  \node[left=of img, node distance=0cm, rotate=90, anchor=center,yshift=-1.0cm] {\tiny Training Loss};
    \node[below=of img, node distance=0cm, rotate=0cm, anchor=center,yshift=3.7cm] {\tiny FEMNIST};
  \node[below=of img, node distance=0cm, rotate=0cm, anchor=center,yshift=1.0cm] {\tiny Round};
\end{tikzpicture}
\end{minipage}%
\caption{FilFL (FedAvg + $\chi$GF + PoC) sensitivity to periodicity $h$ on FEMNIST dataset.}
\label{fig:periodsensitivityfemnist2_main}
\vspace{12pt}
\end{figure}

\subsubsection{Sensitivity to Filtering Dataset Size \& Distribution} 


We evaluate the effectiveness of FilFL under different filtering datasets scenarios, showing its robustness across various sizes and distributions. In the Shakespeare experiment, we use small datasets consisting of parts of this paper's introduction, containing only 34, 17, and 8 samples. Fig. \ref{fig:sizesensitivityshakespeare}, shows that FilFL remains effective even with tiny filtering datasets with different distributions than the clients' datasets. The left plot shows higher test accuracy for FilFL than FedAvg, with a slight advantage for larger values of $m$.
The middle and right plots also reveal lower training loss for smaller $m$ and lower test loss for larger $m$, indicating that larger $m$ leads to better generalization. Similar results concerning the effect of dataset size on the FEMNIST dataset are presented in Appendix, with datasets of 2000, 1000, and 500 samples. Hence, FilFL shows insensitivity to the number of data points, performing well even with smaller datasets and under distribution shifts, thereby proving its versatility and robustness.

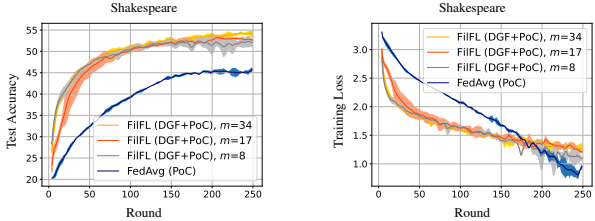
\begin{figure}[!tbp]
\begin{minipage}{0.24\textwidth}
\begin{tikzpicture}
  \node (img)  {\includegraphics[scale=0.23]{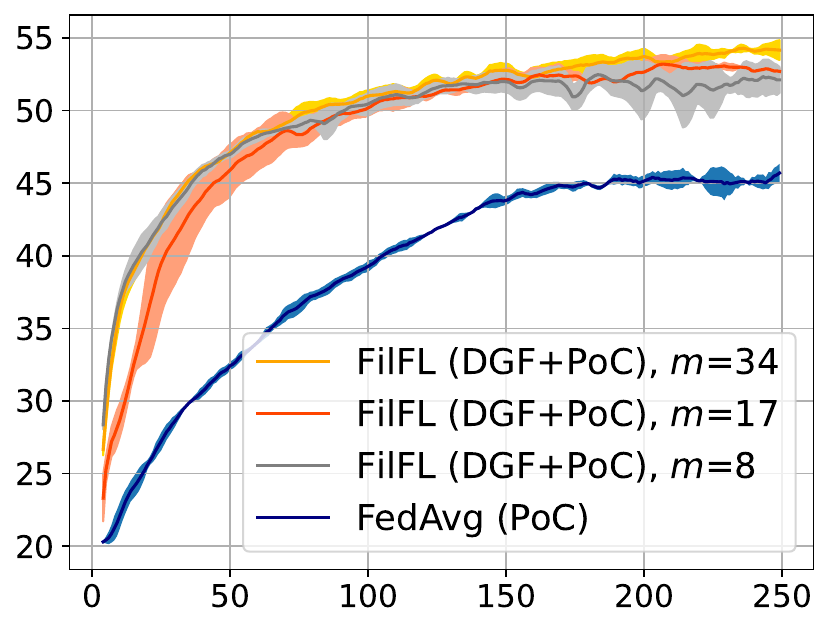}};
    \node[below=of img, node distance=0cm, rotate=0cm, anchor=center,yshift=3.7cm] {\tiny Shakespeare};
  \node[below=of img, node distance=0cm, rotate=0cm, anchor=center,yshift=1cm] {\tiny Round};
  \node[left=of img, node distance=0cm, rotate=90, anchor=center,yshift=-1.0cm] {\tiny Test Accuracy};
 \end{tikzpicture}
\end{minipage}%
\begin{minipage}{0.24\textwidth}
\begin{tikzpicture}
 \node (img)  
{\includegraphics[scale=0.23]{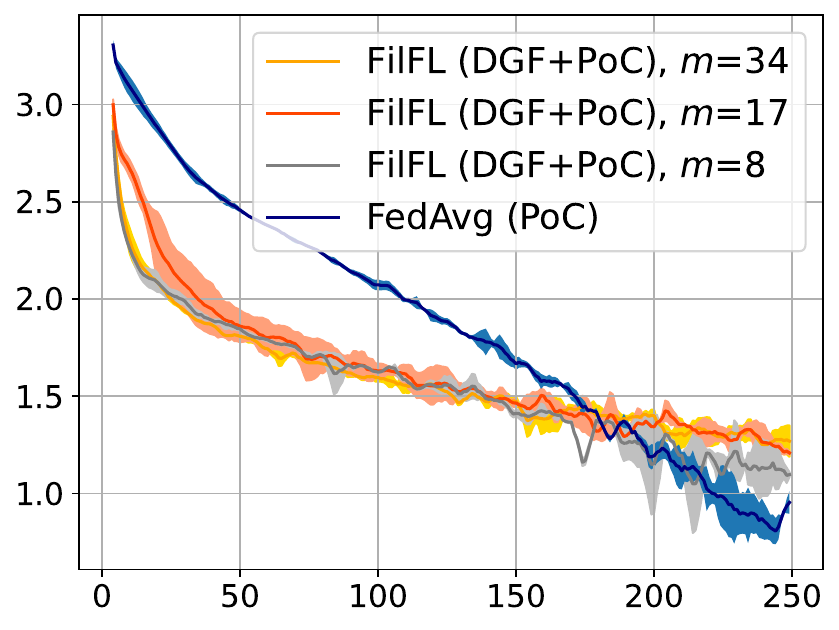}};
    \node[below=of img, node distance=0cm, rotate=0cm, anchor=center,yshift=3.7cm] {\tiny Shakespeare};
  \node[left=of img, node distance=0cm, rotate=90, anchor=center,yshift=-1.0cm] {\tiny Training Loss};
  \node[below=of img, node distance=0cm, rotate=0cm, anchor=center,yshift=1cm] {\tiny Round};
\end{tikzpicture}
\end{minipage}%
\caption{ FilFL (FedAvg with DGF) sensitivity to filtering dataset size $m$ on Shakespeare dataset with PoC for client selection, $N = 143$, $n=100$, $K=10$, and $h=5$.}
\label{fig:sizesensitivityshakespeare}
\vspace{18pt}
\end{figure}

\subsubsection{Sensitivity to Variable Filtering Datasets}
\label{variable}

We evaluate the use of a variable dataset for client filtering. Instead of solving the filtering objective on a central dataset, possibly on a subset of the server validation dataset or one single client throughout the training, we consider the case of randomly selecting a client from the available clients to perform the \textit{client filtering} task. The chosen client performs \textit{client filtering} on its own validation dataset. Therefore, the filtering dataset becomes variable depending on the chosen client in that round. Our results demonstrate that FilFL, using RGF, even in such a stochastic scenario, achieves significantly better performance than FedAvg. In particular, as depicted in Fig. \ref{fig:stoch_cifar10_main}, FilFL accomplishes accelerated training and attains approximately 10 percentage points higher test accuracy than FedAvg.

\begin{figure}[!tbp]
\centering
\begin{minipage}{0.25\textwidth}
\begin{tikzpicture}
  \node (img)  {\includegraphics[scale=0.23]{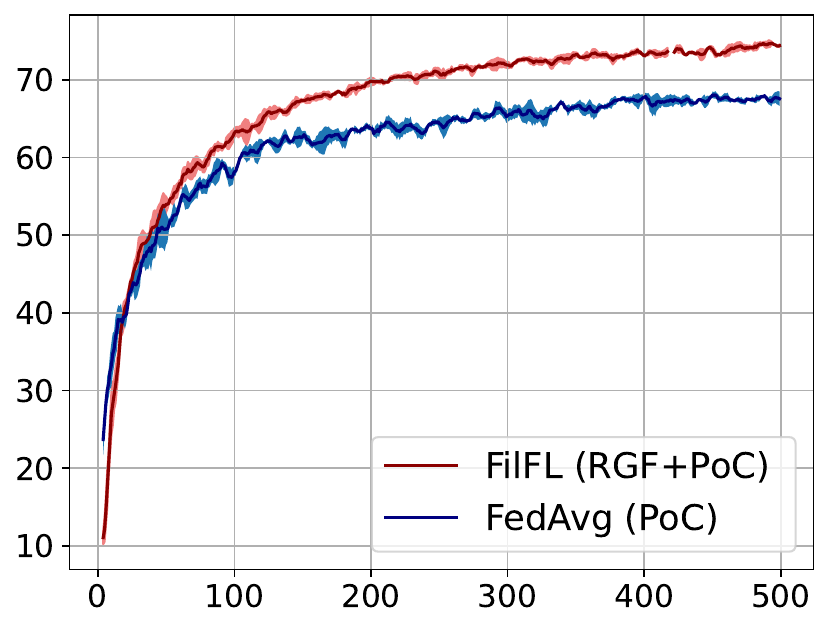}};
  \node[below=of img, node distance=0cm, rotate=0cm, anchor=center,yshift=3.7cm] {\tiny CIFAR-10};
  \node[below=of img, node distance=0cm, rotate=0cm, anchor=center,yshift=1.0cm] {\tiny Round};
  \node[left=of img, node distance=0cm, rotate=90, anchor=center,yshift=-1.0cm] {\tiny Test Accuracy};
 \end{tikzpicture}
\end{minipage}%
\begin{minipage}{0.25\textwidth}
\begin{tikzpicture}
  \node (img)  {\includegraphics[scale=0.23]{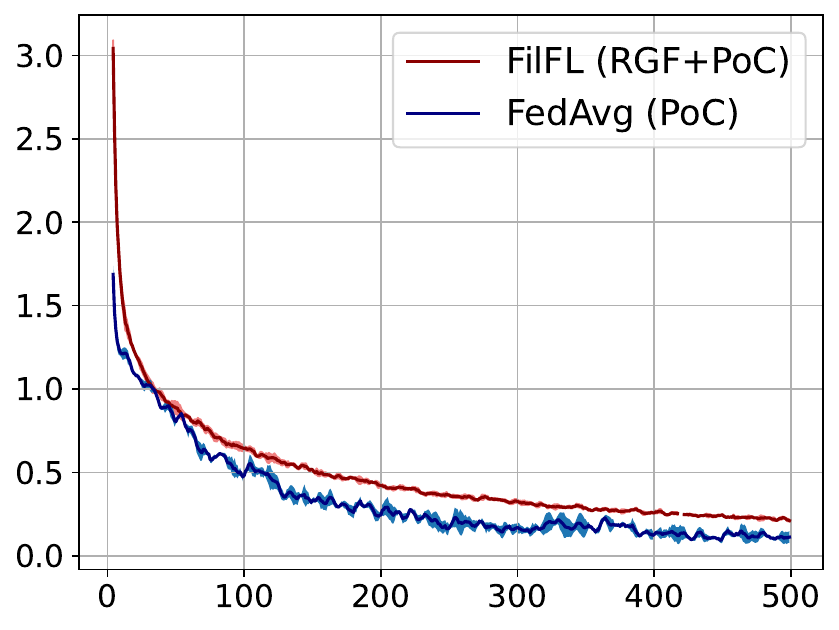}};
  \node[left=of img, node distance=0cm, rotate=90, anchor=center,yshift=-1.0cm] {\tiny Training Loss};
    \node[below=of img, node distance=0cm, rotate=0cm, anchor=center,yshift=3.7cm] {\tiny CIFAR-10};
  \node[below=of img, node distance=0cm, rotate=0cm, anchor=center,yshift=1.0cm] {\tiny Round};
\end{tikzpicture}
\end{minipage}%
\caption{FilFL (FedAvg + RGF + PoC) vs FedAvg (PoC) without filtering on CIFAR-10 dataset.}
\label{fig:stoch_cifar10_main}
\vspace{12pt}
\end{figure}
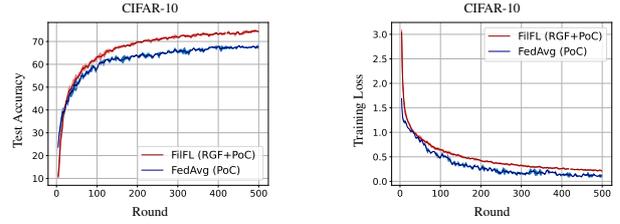

\subsection{$\chi$GF Behavior}
\label{main_section_gf}

We examine the filtering rates and approximation ratios of the RGF and DGF algorithms compared to brute force search results.

\subsubsection{Filtering Rates} 

$\chi$GF rejects multiple clients, with the average rejection rate varying depending on the task and the version (randomized or deterministic). As mentioned in Remark \ref{rejection_rate_remark}, DGF rejects more clients than RGF, roughly half the number of clients (cf. Fig. \ref{fig:filteredsize}).
Therefore, DGF is more efficient in reducing the complexity of client selection by significantly reducing the sampling space.

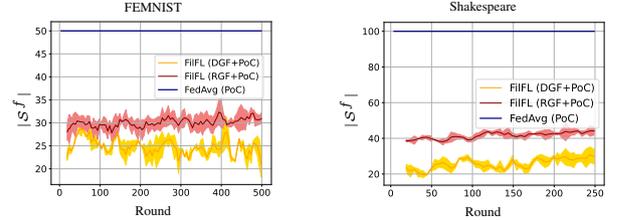
\begin{figure}[!tbp]
\begin{minipage}{0.24\textwidth}
\begin{tikzpicture}
  \node (img)  {\includegraphics[scale=0.23]{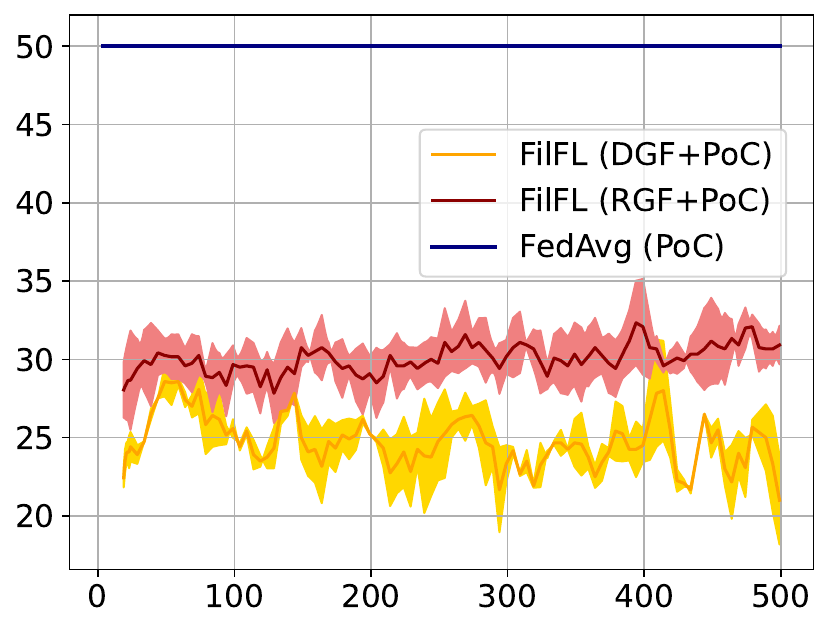}};
  \node[below=of img, node distance=0cm, rotate=0cm, anchor=center,yshift=3.7cm] {\tiny FEMNIST};
  \node[below=of img, node distance=0cm, rotate=0cm, anchor=center,yshift=1cm] {\tiny Round};
  \node[left=of img, node distance=0cm, rotate=90, anchor=center,yshift=-1.0cm] {\tiny $|\mathcal{S}^f|$};
 \end{tikzpicture}
\end{minipage}%
\begin{minipage}{0.24\textwidth}
\begin{tikzpicture}
  \node (img)  {\includegraphics[scale=0.23]{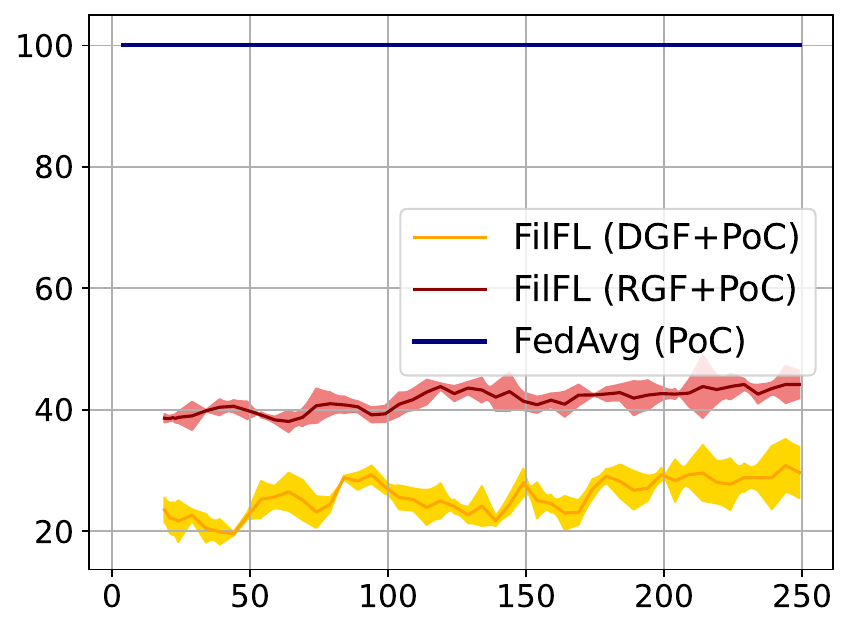}};
  \node[left=of img, node distance=0cm, rotate=90, anchor=center,yshift=-1.0cm] {\tiny $|\mathcal{S}^f|$};
    \node[below=of img, node distance=0cm, rotate=0cm, anchor=center,yshift=3.7cm] {\tiny Shakespeare};
  \node[below=of img, node distance=0cm, rotate=0cm, anchor=center,yshift=1cm] {\tiny Round};
\end{tikzpicture}
\end{minipage}%
\caption{The number of filtered-in clients, denoted as $|\mathcal{S}^f|$, for FilFL (FedAvg with $\chi$GF), over the rounds in different settings of CIFAR-10, FEMNIST, and Shakespeare datasets, with considerable available clients $n$ being 30, 50, and 100, respectively. For FedAvg without filtering, we consider $\mathcal{S}^f$ to be equal to $\mathcal{S}_t$.}
\label{fig:filteredsize}
\vspace{24pt}
\end{figure}

\subsubsection{Approximation Ratios} 

\begin{minipage}[H]{0.20\textwidth}
Fig. \ref{fig:approximations} shows the approximation ratios of both $\chi$GF versions compared to the optimal filtering (OPT) on CIFAR-10 with $N=200$ and $n=10$, which we find by evaluating $2^n - 1$ combinations. We find that both $\chi$GF versions achieve approximation ratios higher
\end{minipage}
\hspace{0.01\textwidth}
\begin{minipage}[H]{0.26\textwidth}
\vspace{-.3in}
\begin{figure}[H]
    \centering
    \begin{tikzpicture}
        \node (img) {\includegraphics[scale=0.23]{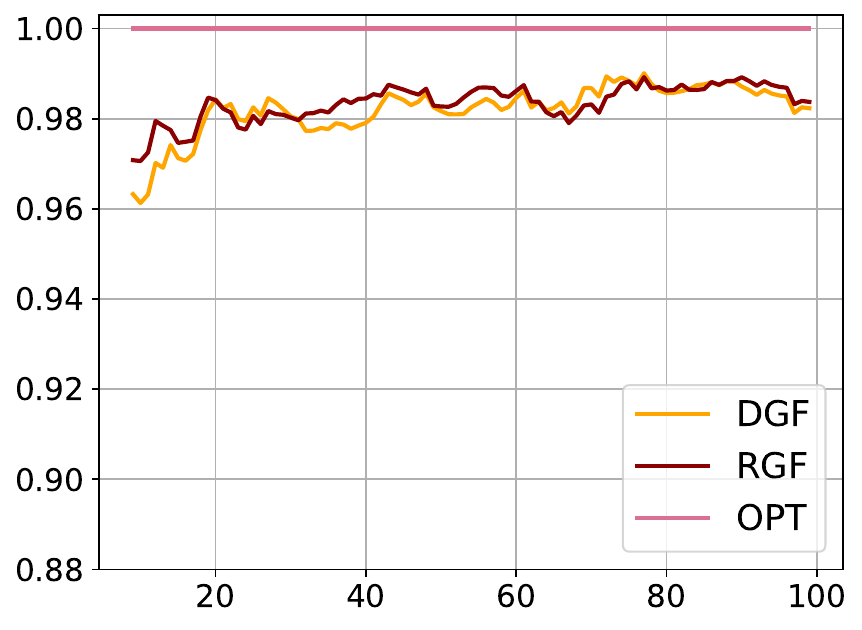}};
        \node[below=of img, node distance=0cm, rotate=0cm, anchor=center,yshift=3.7cm] {\tiny CIFAR-10};
        \node[left=of img, node distance=0cm, rotate=90, anchor=center,yshift=-1cm] {\tiny Approximation Ratio};
        \node[below=of img, node distance=0cm, rotate=0cm, anchor=center,yshift=1cm] {\tiny Round};
    \end{tikzpicture}
    \caption{ Approximation ratios of the filtering objective solution.}
    \label{fig:approximations}  
\end{figure}
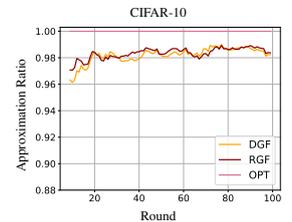
\vspace{.01 in}
\end{minipage}
than 0.96, i.e., $\mathcal{R}(\mathcal{S}^f) \geq 0.96 \mathcal{R}(OPT)$ over the multiple rounds. This indicates that greedy filtering identifies near-optimal combinations of clients.
Finally, the filtering performance can be measured by the improved FL performance and the higher approximation ratios. Since both versions of $\chi$GF show similarly high ratios and improved FL performance, both can be considered effective for filtering.

\section{Conclusion}

We proposed \textit{client filtering} as a promising technique to optimize client participation and training in FL. Our proposed FL algorithm, FilFL, which incorporates the greedy filtering algorithm $\chi$GF, has proven theoretical convergence guarantees and empirically shows better learning efficiency, accelerated convergence, and higher test accuracy across different vision and language tasks.

\newpage
\bibliography{main}

\begin{thebibliography}{66}
\providecommand{\natexlab}[1]{#1}
\providecommand{\url}[1]{\texttt{#1}}
\expandafter\ifx\csname urlstyle\endcsname\relax
  \providecommand{\doi}[1]{doi: #1}\else
  \providecommand{\doi}{doi: \begingroup \urlstyle{rm}\Url}\fi

\bibitem[Abdelmoniem et~al.(2022)Abdelmoniem, Ho, Papageorgiou, and Canini]{Abdelmoniem.EuroMLSys22}
A.~M. Abdelmoniem, C.-Y. Ho, P.~Papageorgiou, and M.~Canini.
\newblock {Empirical Analysis of Federated Learning in Heterogeneous Environments}.
\newblock In \emph{EuroMLSys}, 2022.

\bibitem[Acar et~al.(2021)Acar, Zhao, Navarro, Mattina, Whatmough, and Saligrama]{acar2021federated}
D.~A.~E. Acar, Y.~Zhao, R.~M. Navarro, M.~Mattina, P.~N. Whatmough, and V.~Saligrama.
\newblock Federated learning based on dynamic regularization.
\newblock \emph{arXiv preprint arXiv:2111.04263}, 2021.

\bibitem[Balakrishnan et~al.(2021)Balakrishnan, Li, Zhou, Himayat, Smith, and Bilmes]{balakrishnan2021diverse}
R.~Balakrishnan, T.~Li, T.~Zhou, N.~Himayat, V.~Smith, and J.~Bilmes.
\newblock Diverse client selection for federated learning via submodular maximization.
\newblock In \emph{International Conference on Learning Representations}, 2021.

\bibitem[Bonawitz et~al.(2019)Bonawitz, Eichner, Grieskamp, Huba, Ingerman, Ivanov, Kiddon, Kone\v{c}n\'~{y}, Mazzocchi, McMahan, Van~Overveldt, Petrou, Ramage, and Roselander]{Bonawitz19}
K.~Bonawitz, H.~Eichner, W.~Grieskamp, D.~Huba, A.~Ingerman, V.~Ivanov, C.~Kiddon, J.~Kone\v{c}n\'~{y}, S.~Mazzocchi, B.~McMahan, T.~Van~Overveldt, D.~Petrou, D.~Ramage, and J.~Roselander.
\newblock {Towards Federated Learning at Scale: System Design}.
\newblock In \emph{MLSys}, 2019.

\bibitem[Buchbinder et~al.(2015)Buchbinder, Feldman, Seffi, and Schwartz]{buchbinder2015tight}
N.~Buchbinder, M.~Feldman, J.~Seffi, and R.~Schwartz.
\newblock A tight linear time (1/2)-approximation for unconstrained submodular maximization.
\newblock \emph{SIAM Journal on Computing}, 44\penalty0 (5):\penalty0 1384--1402, 2015.

\bibitem[Caldarola et~al.(2022)Caldarola, Caputo, and Ciccone]{caldarola2022improving}
D.~Caldarola, B.~Caputo, and M.~Ciccone.
\newblock Improving generalization in federated learning by seeking flat minima.
\newblock In \emph{European Conference on Computer Vision}, pages 654--672. Springer, 2022.

\bibitem[Caldas et~al.(2018)Caldas, Duddu, Wu, Li, Kone{\v{c}}n{\`y}, McMahan, Smith, and Talwalkar]{caldas2018leaf}
S.~Caldas, S.~M.~K. Duddu, P.~Wu, T.~Li, J.~Kone{\v{c}}n{\`y}, H.~B. McMahan, V.~Smith, and A.~Talwalkar.
\newblock Leaf: A benchmark for federated settings.
\newblock \emph{arXiv preprint arXiv:1812.01097}, 2018.

\bibitem[Chen and Chao(2020)]{chen2020fedbe}
H.-Y. Chen and W.-L. Chao.
\newblock Fedbe: Making bayesian model ensemble applicable to federated learning.
\newblock \emph{arXiv preprint arXiv:2009.01974}, 2020.

\bibitem[Chen et~al.(2020)Chen, Horvath, and Richtarik]{chen2020optimal}
W.~Chen, S.~Horvath, and P.~Richtarik.
\newblock Optimal client sampling for federated learning.
\newblock \emph{arXiv preprint arXiv:2010.13723}, 2020.

\bibitem[Cheng et~al.(2021)Cheng, Wu, Xiao, and Liu]{cheng2021fedgems}
S.~Cheng, J.~Wu, Y.~Xiao, and Y.~Liu.
\newblock Fedgems: Federated learning of larger server models via selective knowledge fusion.
\newblock \emph{arXiv preprint arXiv:2110.11027}, 2021.

\bibitem[Cho et~al.(2020)Cho, Wang, and Joshi]{cho2020client}
Y.~J. Cho, J.~Wang, and G.~Joshi.
\newblock Client selection in federated learning: Convergence analysis and power-of-choice selection strategies.
\newblock \emph{arXiv preprint arXiv:2010.01243}, 2020.

\bibitem[Cohen et~al.(2017)Cohen, Afshar, Tapson, and Van~Schaik]{cohen2017emnist}
G.~Cohen, S.~Afshar, J.~Tapson, and A.~Van~Schaik.
\newblock Emnist: Extending mnist to handwritten letters.
\newblock In \emph{2017 international joint conference on neural networks (IJCNN)}, pages 2921--2926. IEEE, 2017.

\bibitem[Duan et~al.(2021)Duan, Li, and Lu]{duan2021feddna}
J.-H. Duan, W.~Li, and S.~Lu.
\newblock Feddna: Federated learning with decoupled normalization-layer aggregation for non-iid data.
\newblock In \emph{Machine Learning and Knowledge Discovery in Databases. Research Track: European Conference, ECML PKDD 2021, Bilbao, Spain, September 13--17, 2021, Proceedings, Part I 21}, pages 722--737. Springer, 2021.

\bibitem[Feige et~al.(2011)Feige, Mirrokni, and Vondr{\'a}k]{feige2011maximizing}
U.~Feige, V.~S. Mirrokni, and J.~Vondr{\'a}k.
\newblock Maximizing non-monotone submodular functions.
\newblock \emph{SIAM Journal on Computing}, 40\penalty0 (4):\penalty0 1133--1153, 2011.

\bibitem[Fourati et~al.(2023{\natexlab{a}})Fourati, Aggarwal, Quinn, and Alouini]{fourati2023randomized}
F.~Fourati, V.~Aggarwal, C.~Quinn, and M.-S. Alouini.
\newblock Randomized greedy learning for non-monotone stochastic submodular maximization under full-bandit feedback.
\newblock In \emph{International Conference on Artificial Intelligence and Statistics}, pages 7455--7471. PMLR, 2023{\natexlab{a}}.

\bibitem[Fourati et~al.(2023{\natexlab{b}})Fourati, Kharrat, Aggarwal, Alouini, and Canini]{fourati2023filfl}
F.~Fourati, S.~Kharrat, V.~Aggarwal, M.-S. Alouini, and M.~Canini.
\newblock Filfl: Client filtering for optimized client participation in federated learning.
\newblock \emph{arXiv preprint arXiv:2302.06599}, 2023{\natexlab{b}}.

\bibitem[Fourati et~al.(2024)Fourati, Quinn, Alouini, and Aggarwal]{fourati2024combinatorial}
F.~Fourati, C.~J. Quinn, M.-S. Alouini, and V.~Aggarwal.
\newblock Combinatorial stochastic-greedy bandit.
\newblock In \emph{Proceedings of the AAAI Conference on Artificial Intelligence}, volume~38, pages 12052--12060, 2024.

\bibitem[Ganguly et~al.(2023)Ganguly, Hosseinalipour, Kim, Brinton, Aggarwal, Love, and Chiang]{ganguly2023multi}
B.~Ganguly, S.~Hosseinalipour, K.~T. Kim, C.~G. Brinton, V.~Aggarwal, D.~J. Love, and M.~Chiang.
\newblock Multi-edge server-assisted dynamic federated learning with an optimized floating aggregation point.
\newblock \emph{IEEE/ACM Transactions on Networking}, 2023.

\bibitem[Haddadpour and Mahdavi(2019)]{haddadpour2019convergence}
F.~Haddadpour and M.~Mahdavi.
\newblock On the convergence of local descent methods in federated learning.
\newblock \emph{arXiv preprint arXiv:1910.14425}, 2019.

\bibitem[He et~al.(2020)He, Li, So, Zeng, Zhang, Wang, Wang, Vepakomma, Singh, Qiu, et~al.]{he2020fedml}
C.~He, S.~Li, J.~So, X.~Zeng, M.~Zhang, H.~Wang, X.~Wang, P.~Vepakomma, A.~Singh, H.~Qiu, et~al.
\newblock Fedml: A research library and benchmark for federated machine learning.
\newblock \emph{arXiv preprint arXiv:2007.13518}, 2020.

\bibitem[He et~al.(2016)He, Zhang, Ren, and Sun]{he2016deep}
K.~He, X.~Zhang, S.~Ren, and J.~Sun.
\newblock Deep residual learning for image recognition.
\newblock In \emph{Proceedings of the IEEE conference on computer vision and pattern recognition}, pages 770--778, 2016.

\bibitem[Hochreiter and Schmidhuber(1997)]{hochreiter1997long}
S.~Hochreiter and J.~Schmidhuber.
\newblock Long short-term memory.
\newblock \emph{Neural computation}, 9\penalty0 (8):\penalty0 1735--1780, 1997.

\bibitem[Hosseinalipour et~al.(2020)Hosseinalipour, Brinton, Aggarwal, Dai, and Chiang]{HosseinalipourMag2020}
S.~Hosseinalipour, C.~G. Brinton, V.~Aggarwal, H.~Dai, and M.~Chiang.
\newblock From federated to fog learning: Distributed machine learning over heterogeneous wireless networks.
\newblock \emph{IEEE Communications Magazine}, 58\penalty0 (12):\penalty0 41--47, 2020.
\newblock \doi{10.1109/MCOM.001.2000410}.

\bibitem[Huang et~al.(2022)Huang, Ye, and Du]{huang2022learn}
W.~Huang, M.~Ye, and B.~Du.
\newblock Learn from others and be yourself in heterogeneous federated learning.
\newblock In \emph{Proceedings of the IEEE/CVF Conference on Computer Vision and Pattern Recognition}, pages 10143--10153, 2022.

\bibitem[Huba et~al.(2022)Huba, Nguyen, Malik, Zhu, Rabbat, Yousefpour, Wu, Zhan, Ustinov, Srinivas, Wang, Shoumikhin, Min, and Malek]{Papaya}
D.~Huba, J.~Nguyen, K.~Malik, R.~Zhu, M.~Rabbat, A.~Yousefpour, C.-J. Wu, H.~Zhan, P.~Ustinov, H.~Srinivas, K.~Wang, A.~Shoumikhin, J.~Min, and M.~Malek.
\newblock {PAPAYA: Practical, Private, and Scalable Federated Learning}.
\newblock In \emph{MLSys}, 2022.

\bibitem[Huo et~al.(2020)Huo, Yang, Gu, Huang, et~al.]{huo2020faster}
Z.~Huo, Q.~Yang, B.~Gu, L.~C. Huang, et~al.
\newblock Faster on-device training using new federated momentum algorithm.
\newblock \emph{arXiv preprint arXiv:2002.02090}, 2020.

\bibitem[Izmailov et~al.(2018)Izmailov, Podoprikhin, Garipov, Vetrov, and Wilson]{izmailov2018averaging}
P.~Izmailov, D.~Podoprikhin, T.~Garipov, D.~Vetrov, and A.~G. Wilson.
\newblock Averaging weights leads to wider optima and better generalization.
\newblock \emph{arXiv preprint arXiv:1803.05407}, 2018.

\bibitem[Kairouz et~al.(2021)Kairouz, McMahan, Avent, Bellet, Bennis, Bhagoji, Bonawitz, Charles, Cormode, Cummings, et~al.]{kairouz2021advances}
P.~Kairouz, H.~B. McMahan, B.~Avent, A.~Bellet, M.~Bennis, A.~N. Bhagoji, K.~Bonawitz, Z.~Charles, G.~Cormode, R.~Cummings, et~al.
\newblock Advances and open problems in federated learning.
\newblock \emph{Foundations and Trends{\textregistered} in Machine Learning}, 14\penalty0 (1--2):\penalty0 1--210, 2021.

\bibitem[Karimireddy et~al.(2019)Karimireddy, Kale, Mohri, Reddi, Stich, and Suresh]{karimireddy2019scaffold}
S.~P. Karimireddy, S.~Kale, M.~Mohri, S.~J. Reddi, S.~U. Stich, and A.~T. Suresh.
\newblock Scaffold: Stochastic controlled averaging for on-device federated learning.
\newblock 2019.

\bibitem[Khaled et~al.(2020)Khaled, Mishchenko, and Richt{\'a}rik]{khaled2020tighter}
A.~Khaled, K.~Mishchenko, and P.~Richt{\'a}rik.
\newblock Tighter theory for local sgd on identical and heterogeneous data.
\newblock In \emph{International Conference on Artificial Intelligence and Statistics}, pages 4519--4529. PMLR, 2020.

\bibitem[Kim et~al.(2022)Kim, Kim, and Han]{kim2022communication}
G.~Kim, J.~Kim, and B.~Han.
\newblock Communication-efficient federated learning with acceleration of global momentum.
\newblock \emph{arXiv preprint arXiv:2201.03172}, 2022.

\bibitem[Kirkpatrick et~al.(2017)Kirkpatrick, Pascanu, Rabinowitz, Veness, Desjardins, Rusu, Milan, Quan, Ramalho, Grabska-Barwinska, et~al.]{kirkpatrick2017overcoming}
J.~Kirkpatrick, R.~Pascanu, N.~Rabinowitz, J.~Veness, G.~Desjardins, A.~A. Rusu, K.~Milan, J.~Quan, T.~Ramalho, A.~Grabska-Barwinska, et~al.
\newblock Overcoming catastrophic forgetting in neural networks.
\newblock \emph{Proceedings of the national academy of sciences}, 114\penalty0 (13):\penalty0 3521--3526, 2017.

\bibitem[Koloskova et~al.(2020)Koloskova, Loizou, Boreiri, Jaggi, and Stich]{koloskova2020unified}
A.~Koloskova, N.~Loizou, S.~Boreiri, M.~Jaggi, and S.~Stich.
\newblock A unified theory of decentralized sgd with changing topology and local updates.
\newblock In \emph{International Conference on Machine Learning}, pages 5381--5393. PMLR, 2020.

\bibitem[Kone{\v{c}}n{\`y}(2017)]{konevcny2017stochastic}
J.~Kone{\v{c}}n{\`y}.
\newblock Stochastic, distributed and federated optimization for machine learning.
\newblock \emph{arXiv preprint arXiv:1707.01155}, 2017.

\bibitem[Kone{\v{c}}n{\`y} et~al.(2015)Kone{\v{c}}n{\`y}, McMahan, and Ramage]{konevcny2015federated}
J.~Kone{\v{c}}n{\`y}, B.~McMahan, and D.~Ramage.
\newblock Federated optimization: Distributed optimization beyond the datacenter.
\newblock \emph{arXiv preprint arXiv:1511.03575}, 2015.

\bibitem[Kone{\v{c}}n{\`y} et~al.(2016)Kone{\v{c}}n{\`y}, McMahan, Yu, Richt{\'a}rik, Suresh, and Bacon]{konevcny2016federated}
J.~Kone{\v{c}}n{\`y}, H.~B. McMahan, F.~X. Yu, P.~Richt{\'a}rik, A.~T. Suresh, and D.~Bacon.
\newblock Federated learning: Strategies for improving communication efficiency.
\newblock \emph{arXiv preprint arXiv:1610.05492}, 2016.

\bibitem[Krizhevsky et~al.(2009)Krizhevsky, Hinton, et~al.]{krizhevsky2009learning}
A.~Krizhevsky, G.~Hinton, et~al.
\newblock Learning multiple layers of features from tiny images.
\newblock \emph{Canadian Institute for Advanced Research, 2009. URL http://www.cs.toronto.edu/~kriz/cifar.html}, 2009.

\bibitem[LeCun(1998)]{lecun1998mnist}
Y.~LeCun.
\newblock The mnist database of handwritten digits.
\newblock \emph{http://yann. lecun. com/exdb/mnist/}, 1998.

\bibitem[Li et~al.(2020)Li, Sahu, Zaheer, Sanjabi, Talwalkar, and Smith]{li2020federated}
T.~Li, A.~K. Sahu, M.~Zaheer, M.~Sanjabi, A.~Talwalkar, and V.~Smith.
\newblock Federated optimization in heterogeneous networks.
\newblock \emph{Proceedings of Machine Learning and Systems}, 2:\penalty0 429--450, 2020.

\bibitem[Li et~al.(2019)Li, Huang, Yang, Wang, and Zhang]{li2019convergence}
X.~Li, K.~Huang, W.~Yang, S.~Wang, and Z.~Zhang.
\newblock On the convergence of fedavg on non-iid data.
\newblock \emph{arXiv preprint arXiv:1907.02189}, 2019.

\bibitem[Lin et~al.(2020)Lin, Kong, Stich, and Jaggi]{lin2020ensemble}
T.~Lin, L.~Kong, S.~U. Stich, and M.~Jaggi.
\newblock Ensemble distillation for robust model fusion in federated learning.
\newblock \emph{Advances in Neural Information Processing Systems}, 33:\penalty0 2351--2363, 2020.

\bibitem[Liu et~al.(2023)Liu, Sun, Ding, Shen, Liu, and Tao]{liu2023enhance}
Y.~Liu, Y.~Sun, Z.~Ding, L.~Shen, B.~Liu, and D.~Tao.
\newblock Enhance local consistency in federated learning: A multi-step inertial momentum approach.
\newblock \emph{arXiv preprint arXiv:2302.05726}, 2023.

\bibitem[Malinovskiy et~al.(2020)Malinovskiy, Kovalev, Gasanov, Condat, and Richtarik]{malinovskiy2020local}
G.~Malinovskiy, D.~Kovalev, E.~Gasanov, L.~Condat, and P.~Richtarik.
\newblock From local sgd to local fixed-point methods for federated learning.
\newblock In \emph{International Conference on Machine Learning}, pages 6692--6701. PMLR, 2020.

\bibitem[McMahan et~al.(2017)McMahan, Moore, Ramage, Hampson, and y~Arcas]{mcmahan2017communication}
B.~McMahan, E.~Moore, D.~Ramage, S.~Hampson, and B.~A. y~Arcas.
\newblock Communication-efficient learning of deep networks from decentralized data.
\newblock In \emph{Artificial intelligence and statistics}, pages 1273--1282. PMLR, 2017.

\bibitem[Meng et~al.(2021)Meng, Rambhatla, and Liu]{meng2021cross}
C.~Meng, S.~Rambhatla, and Y.~Liu.
\newblock Cross-node federated graph neural network for spatio-temporal data modeling.
\newblock In \emph{Proceedings of the 27th ACM SIGKDD conference on knowledge discovery \& data mining}, pages 1202--1211, 2021.

\bibitem[Mirzasoleiman et~al.(2020)Mirzasoleiman, Bilmes, and Leskovec]{mirzasoleiman2020coresets}
B.~Mirzasoleiman, J.~Bilmes, and J.~Leskovec.
\newblock Coresets for data-efficient training of machine learning models.
\newblock In \emph{International Conference on Machine Learning}, pages 6950--6960. PMLR, 2020.

\bibitem[Mothukuri et~al.(2021)Mothukuri, Parizi, Pouriyeh, Huang, Dehghantanha, and Srivastava]{mothukuri2021survey}
V.~Mothukuri, R.~M. Parizi, S.~Pouriyeh, Y.~Huang, A.~Dehghantanha, and G.~Srivastava.
\newblock A survey on security and privacy of federated learning.
\newblock \emph{Future Generation Computer Systems}, 115:\penalty0 619--640, 2021.

\bibitem[Ozfatura et~al.(2021)Ozfatura, Ozfatura, and G{\"u}nd{\"u}z]{ozfatura2021fedadc}
E.~Ozfatura, K.~Ozfatura, and D.~G{\"u}nd{\"u}z.
\newblock Fedadc: Accelerated federated learning with drift control.
\newblock In \emph{2021 IEEE International Symposium on Information Theory (ISIT)}, pages 467--472. IEEE, 2021.

\bibitem[Pathak and Wainwright(2020)]{pathak2020fedsplit}
R.~Pathak and M.~J. Wainwright.
\newblock Fedsplit: An algorithmic framework for fast federated optimization.
\newblock \emph{Advances in Neural Information Processing Systems}, 33:\penalty0 7057--7066, 2020.

\bibitem[Peajcariaac and Tong(1992)]{peajcariaac1992convex}
J.~E. Peajcariaac and Y.~L. Tong.
\newblock \emph{Convex functions, partial orderings, and statistical applications}.
\newblock Academic Press, 1992.

\bibitem[Reddi et~al.(2020)Reddi, Charles, Zaheer, Garrett, Rush, Kone{\v{c}}n{\`y}, Kumar, and McMahan]{reddi2020adaptive}
S.~Reddi, Z.~Charles, M.~Zaheer, Z.~Garrett, K.~Rush, J.~Kone{\v{c}}n{\`y}, S.~Kumar, and H.~B. McMahan.
\newblock Adaptive federated optimization.
\newblock \emph{arXiv preprint arXiv:2003.00295}, 2020.

\bibitem[Sattler et~al.(2021)Sattler, Korjakow, Rischke, and Samek]{sattler2021fedaux}
F.~Sattler, T.~Korjakow, R.~Rischke, and W.~Samek.
\newblock Fedaux: Leveraging unlabeled auxiliary data in federated learning.
\newblock \emph{IEEE Transactions on Neural Networks and Learning Systems}, 34\penalty0 (9):\penalty0 5531--5543, 2021.

\bibitem[Shoham et~al.(2019)Shoham, Avidor, Keren, Israel, Benditkis, Mor-Yosef, and Zeitak]{shoham2019overcoming}
N.~Shoham, T.~Avidor, A.~Keren, N.~Israel, D.~Benditkis, L.~Mor-Yosef, and I.~Zeitak.
\newblock Overcoming forgetting in federated learning on non-iid data.
\newblock \emph{arXiv preprint arXiv:1910.07796}, 2019.

\bibitem[Shokri and Shmatikov(2015)]{shokri2015privacy}
R.~Shokri and V.~Shmatikov.
\newblock Privacy-preserving deep learning.
\newblock In \emph{Proceedings of the 22nd ACM SIGSAC conference on computer and communications security}, pages 1310--1321, 2015.

\bibitem[Stich and Karimireddy(2019)]{stich2019error}
S.~U. Stich and S.~P. Karimireddy.
\newblock The error-feedback framework: Better rates for sgd with delayed gradients and compressed communication.
\newblock \emph{arXiv preprint arXiv:1909.05350}, 2019.

\bibitem[Varno et~al.(2022)Varno, Saghayi, Rafiee~Sevyeri, Gupta, Matwin, and Havaei]{varno2022adabest}
F.~Varno, M.~Saghayi, L.~Rafiee~Sevyeri, S.~Gupta, S.~Matwin, and M.~Havaei.
\newblock Adabest: Minimizing client drift in federated learning via adaptive bias estimation.
\newblock In \emph{European Conference on Computer Vision}, pages 710--726. Springer, 2022.

\bibitem[Wang et~al.(2020)Wang, Liu, Liang, Joshi, and Poor]{fednova}
J.~Wang, Q.~Liu, H.~Liang, G.~Joshi, and H.~V. Poor.
\newblock Tackling the objective inconsistency problem in heterogeneous federated optimization.
\newblock \emph{Advances in neural information processing systems}, 33:\penalty0 7611--7623, 2020.

\bibitem[Wang et~al.(2019)Wang, Mathews, Kiddon, Eichner, Beaufays, and Ramage]{wang2019federated}
K.~Wang, R.~Mathews, C.~Kiddon, H.~Eichner, F.~Beaufays, and D.~Ramage.
\newblock Federated evaluation of on-device personalization.
\newblock \emph{arXiv preprint arXiv:1910.10252}, 2019.

\bibitem[Wang et~al.(2023)Wang, Hosseinalipour, Aggarwal, Brinton, Love, Su, and Chiang]{wang2023towards}
S.~Wang, S.~Hosseinalipour, V.~Aggarwal, C.~G. Brinton, D.~J. Love, W.~Su, and M.~Chiang.
\newblock Towards cooperative federated learning over heterogeneous edge/fog networks.
\newblock \emph{arXiv preprint arXiv:2303.08361}, 2023.

\bibitem[Woodworth et~al.(2020)Woodworth, Patel, Stich, Dai, Bullins, Mcmahan, Shamir, and Srebro]{woodworth2020local}
B.~Woodworth, K.~K. Patel, S.~Stich, Z.~Dai, B.~Bullins, B.~Mcmahan, O.~Shamir, and N.~Srebro.
\newblock Is local sgd better than minibatch sgd?
\newblock In \emph{International Conference on Machine Learning}, pages 10334--10343. PMLR, 2020.

\bibitem[Xu et~al.(2021)Xu, Wang, Wang, and Yao]{xu2021fedcm}
J.~Xu, S.~Wang, L.~Wang, and A.~C.-C. Yao.
\newblock Fedcm: Federated learning with client-level momentum.
\newblock \emph{arXiv preprint arXiv:2106.10874}, 2021.

\bibitem[Yu et~al.(2021)Yu, Zhang, Qin, Xu, Wang, Liu, Tian, and Chen]{yu2021fed2}
F.~Yu, W.~Zhang, Z.~Qin, Z.~Xu, D.~Wang, C.~Liu, Z.~Tian, and X.~Chen.
\newblock Fed2: Feature-aligned federated learning.
\newblock In \emph{Proceedings of the 27th ACM SIGKDD conference on knowledge discovery \& data mining}, pages 2066--2074, 2021.

\bibitem[Yurochkin et~al.(2019)Yurochkin, Agarwal, Ghosh, Greenewald, Hoang, and Khazaeni]{yurochkin2019bayesian}
M.~Yurochkin, M.~Agarwal, S.~Ghosh, K.~Greenewald, N.~Hoang, and Y.~Khazaeni.
\newblock Bayesian nonparametric federated learning of neural networks.
\newblock In \emph{International Conference on Machine Learning}, pages 7252--7261. PMLR, 2019.

\bibitem[Zhang et~al.(2021)Zhang, Guo, Ma, Wang, Xu, and Wu]{zhang2021parameterized}
J.~Zhang, S.~Guo, X.~Ma, H.~Wang, W.~Xu, and F.~Wu.
\newblock Parameterized knowledge transfer for personalized federated learning.
\newblock \emph{Advances in Neural Information Processing Systems}, 34:\penalty0 10092--10104, 2021.

\bibitem[Zhang et~al.(2020)Zhang, Hong, Dhople, Yin, and Liu]{zhang2020fedpd}
X.~Zhang, M.~Hong, S.~Dhople, W.~Yin, and Y.~Liu.
\newblock Fedpd: A federated learning framework with optimal rates and adaptivity to non-iid data.
\newblock \emph{arXiv preprint arXiv:2005.11418}, 2020.

\bibitem[Zhao et~al.(2018)Zhao, Li, Lai, Suda, Civin, and Chandra]{subopt}
Y.~Zhao, M.~Li, L.~Lai, N.~Suda, D.~Civin, and V.~Chandra.
\newblock Federated learning with non-iid data.
\newblock \emph{arXiv preprint arXiv:1806.00582}, 2018.

\end{thebibliography}


\newpage
\appendix
\onecolumn

\section{Tables of Notations and Abbreviations}
\label{notation} 
\begin{table}[H]
\begin{center}
\begin{tabular}{ |c|l| } 
 \hline
 & \\
 $\Omega$ & set of all clients, $|\Omega| =N$  \\ 
 $\mathcal{S}_t$ & set of active clients in round t \\ 
  $\mathcal{S}^f$ & set of filtered clients in round t, $\mathcal{S}^f \subseteq \mathcal{S}_t $  \\
   $\mathcal{A}_t$ & set of selected clients in round t, $|\mathcal{A}_t| = K$ \\ 
    $\mathcal{V}$ & filtering dataset, $|\mathcal{P}|= m$  \\ 
    $\mathcal{D}$ & union of private datasets \\ 
    $F_k$ & loss of client $k$  \\ 
    $x_{k,i}$ & data point from the client $k$ \\
    $F$ & average loss of all clients \\ 
      $F_\mathcal{V}$ & loss on filtering dataset \\ 
      $\ell$ & some loss function \\     
      $\mathcal{R}$ & objective function (reward) for RGF \\ 
      $m_k$ & number of data points for client $k$ \\ 
      $E$ & number of local steps \\ 
      $T$ & number of communication rounds\\ 
      $t$ & time step\\ 
      $\eta$ & learning rate\\ 
      $\mathbf{w}^k_t$ & parameters of client $k$ in round t\\
      $\mathbf{w}_t$ & global model parameters in round t\\
      $h$ & periodicity of RGF\\
      $p$ & acceptance probability of RGF\\
      $p_k$ & weight of the $k^{th}$ client\\
      $X$ & set of clients\\
      $Y$ & set of clients\\
      & \\
 \hline
\end{tabular}
\end{center}
\caption {Table of notations}
\label{tab:notation} 
\end{table}

\begin{table}[H]
\begin{center}
\begin{tabular}{ |c|l| } 
 \hline
 & \\
 $\chi$GF & greedy filtering \\
 RGF & randomized greedy filtering \\ 
 DGF & determinsitic greedy filtering \\ 
 OPT & optimal filtering (grid search) \\
 SGD & stochastic gradient descent \\
 RS & random sampling \cite{li2019convergence} \\ 
 PoC & power-of-choice \cite{cho2020client}\\ 
 FedAvg & federated averaging \cite{mcmahan2017communication}\\ 
 DivFL & diverse client selection \cite{balakrishnan2021diverse} \\ 
 FedProx & FL algorithm in \cite{li2020federated} \\
 w/o & without \\
 IID & independent and identically distributed \\
 & \\
 \hline
\end{tabular}
\end{center}
\caption {Table of abbreviations}
\label{tab:abbrev} 
\end{table}

\newpage
\section{Extended Related Work}
Generalization aims to find a global model that performs well for most FL clients \cite{kairouz2021advances, meng2021cross, yu2021fed2,  mothukuri2021survey}.
In FL, one of the most significant challenges is data heterogeneity, which causes training to slow down \cite{li2019convergence,karimireddy2019scaffold}. This heterogeneity leads to a learning trend that becomes noisy and unstable \cite{caldarola2022improving}, while the global model suffers from catastrophic forgetting of the knowledge acquired by previously involved clients \cite{shoham2019overcoming, kirkpatrick2017overcoming}.

Several methods have been proposed to enhance generalization in FL. Some of these methods focus on the client side, incorporating a regularization term \cite{karimireddy2019scaffold, li2020federated, acar2021federated, ozfatura2021fedadc, varno2022adabest} into the local objective to mitigate client drift. Others utilize momentum to integrate knowledge from previous updates, guiding local optimization along the trajectory defined by the global models across rounds \cite{karimireddy2019scaffold, kim2022communication, wang2019federated, xu2021fedcm, liu2023enhance}. Some other methods have been proposed on the server side by modifying the aggregation procedure \cite{fednova, duan2021feddna, caldarola2022improving, izmailov2018averaging} or by post-aggregation refinement \cite{lin2020ensemble, chen2020fedbe, sattler2021fedaux}. 

Another line of research focused on mitigating heterogeneity by adding conditions on the client selected in a specific round, including, sampling clients with more significant update norms with higher probability \cite{chen2020optimal}, using power-of-choice (PoC), a biased client selection method that selects clients with higher local losses \cite{cho2020client}, and diverse selection (DivFL), which selects a diverse subset of clients that carry representative gradient information \cite{balakrishnan2021diverse}. Nevertheless, the approaches above select participants from the entire available pool without considering whether they are all appropriate for collaboration at the current stage of the training process.

\section{Experimental Details}
\label{exp-appendix}

\subsection{CIFAR-based benchmarks}

\textit{Distribution.} We first split CIFAR-10 train datasets into private and server datasets, where the server partition fraction is $0.01$, and it is used by the filtering algorithm. The private dataset is distributed non-IID among all the clients and split into a train ($0.9$) and validation ($0.1$) datasets. Similar to existing works \cite{acar2021federated, he2020fedml, yurochkin2019bayesian}, to simulate the non-IID data distribution among clients, we use the Dirichlet distribution Dir($\alpha$)  where a smaller $\alpha$ indicates higher data heterogeneity. We report results with $\alpha = 0.5$. Finally, we use the existing CIFAR-10 test sets as global test sets. 

\textit{Model.} We employ ResNet18 \cite{he2016deep} as the basic backbone. 

\textit{Hyperparameters.} We set the number of local training epoch $E=5$, communication rounds $T=500$, and the number of clients $N=200$. To make the simulation more realistic, we also simulate behaviour heterogeneity by considering a time-varying set of available clients $\mathcal{S}_t$ of size $n = 100$, $n=30$, and $n=10$, in the Appendix Subsection "FilFL (FedAvg with $\chi$GF and PoC) vs FedAvg (PoC)", for Section "FilFL Convergence Analysis", and for Subsection "$\chi$GF Behavior", respectively, randomly selected without replacement from the entire pool of clients every $5$ round. We set the filtering periodicity as $h=5$ for both the Appendix Subsection "FilFL (FedAvg with $\chi$GF and PoC) vs FedAvg (PoC)" and Section "FilFL Convergence Analysis" and $h=1$ for Subsection "$\chi$GF Behavior". Then, we conduct client selection with the fraction $C=0.1$ (e.g., $K = \left|\mathcal{A}_t\right|=10$ for $n=100$). For local training, the batch size is $16$, and the weight decay is $1 e-3$. The learning rate is $0.1$, with a decaying factor of $0.998$ every $10$ rounds.

\subsection{FEMNIST-based benchmarks}

\textit{Distribution.} We use the FEMNIST dataset from the LEAF framework \cite{caldas2018leaf}. The dataset comprises train and test datasets containing a client-data mapping file that splits the data in a non-IID manner among the clients. It has natural heterogeneity stemming from the writing style of each person. Following \cite{caldas2018leaf}, we use only 5\% of the FEMNIST available dataset with 190 clients. We split the training data of each client into three parts; validation data ($0.2$), server data ($0.05$) and training data ($0.75$). We concatenate all the server datasets from all the clients to obtain a global server dataset representative of all clients. Finally, we use the test set as a global test set. 

\textit{Model.} Similar to \cite{caldas2018leaf}, we use a model with two convolutional layers followed by pooling and ReLU and a final dense layer with 2048 units. 

\textit{Hyperparameters.} We set the number of local training epoch $E=2$, communication rounds $T=500$, and the number of clients $N=190$. To make the simulation more realistic, we simulate behaviour heterogeneity by considering a time-varying set of available clients $\mathcal{S}_t$ of size $n = 50$, randomly selected without replacement from the full pool of clients every $5$ rounds, except for Fig.~\ref{fig:periodsensitivityfemnist2} where we it is done every 20 rounds. Moreover, we choose the filtering periodicity $h=5$. Furthermore, we study the impact of periodcity in Sec.~\ref{appendix:periodicity} by setting $h=1$, $h=3$, and $h=5$ in Fig.~\ref{fig:app:periodsensitivityfemnist}, and $h=10$ and $h=20$ in Fig.~\ref{fig:periodsensitivityfemnist2}. Then we conduct client selection with the fraction $C=0.1$ (i.e., $K = \left|\mathcal{A}_t\right|=5$ ). For local training, the batch size is $50$. The learning rate is $0.003$.

\subsection{Shakespeare-based benchmarks}
\label{appendix:shakespearesetup}

\textit{Distribution.} 
We use the Shakespeare dataset from the LEAF framework \cite{caldas2018leaf}. The dataset comprises train and test datasets containing a client-data mapping file that splits the data in a non-IID manner among the clients. It is built from The Complete Works of William Shakespeare, where
each speaking role in each play is considered a different device. Following \cite{caldas2018leaf}, we use only 5\% of the Shakespeare available dataset with 143 clients. We split the training data of each client into two parts; validation data ($0.2$) and training data ($0.8$). We use some text from our own work to build the filtering dataset. Finally, we use the test set as a global test set. 

\textit{Model.} We use a two-layer LSTM classifier containing 256 hidden units with an 8D embedding
layer. The task is a next-character prediction with 80 classes of characters in total. The model takes as input a sequence of 80 characters, embeds the characters
into a learned 8-dimensional space, and outputs one character per training sample after 2 LSTM layers
and a densely-connected layer.

\textit{Hyperparameters.} 
We set the number of local training epoch $E=1$, communication rounds $T=250$, and the number of clients $N=143$. To make the simulation more realistic, we simulate behavior heterogeneity by considering a time-varying set of available clients $\mathcal{S}_t$ of size $n = 100$, randomly selected without replacement from the full pool of clients every $5$ rounds. Moreover, we choose the filtering periodicity $h=5$. We conduct client selection with the fraction $C=0.1$ (i.e., $K = \left|\mathcal{A}_t\right|=10$ ). For local training, the batch size is $64$. The learning rate is $0.8$.

\textit{Server dataset.}
We use a small filtering dataset from a different distribution, specifically consisting of parts of our own text, as shown in Table~\ref{tab:filtering}. The first column represents the index of the data point, the middle column shows the phrase $x$, which consists of 80 characters (features), and the last column represents the next character to predict (label), denoted as $y$.

\begin{table}[H]
\begin{center}
\begin{tabular}{ |c|c|l| } 
 \hline
\# & X &  y\\
 \hline

& & \\
1 & Federated learning has emerged as a promising machine learning paradigm that al & l \\ 
2 & ows collaborative training across distributed clients while keeping their data & l\\
3 &  ocal. However, the success of federated learning heavily relies on overcoming t, & h\\
4 &  e challenges of training with a large number of clients and non-iid data, which, & \\
5 &  often leads to unstable and slow convergence and suboptimal model performance. , & T\\
6 &  o address these challenges, many client selection methods have been proposed to & \\
7 &  optimize partial client participation and mitigate the impact of heterogeneous , & c \\
8 &  lients. However, these methods only select participants from the pool of availa, & b\\
9 &  le clients without considering whether the cohort of clients selected at each r, & o\\
10 &  und contains the most suitable ones. In this context, we introduce a novel appr, & o\\
11 &  ach called FilFL, which proposes a \textit{client filtering} procedure to identify the c, & l\\
12 &  ients that should be considered at each stage of the training process. FilFL di, & s\\
13 &  cards clients that are likely to have only marginal improvements in the trained, &   \\
14 &  model compared to other more promising clients. The assessment of client improv, & e\\
15 & ment uses a filtering dataset held by the FL server to gauge the representativenes, & s\\
16 &  of different local client data towards global model performance. The main cont, & r\\
 17 & ibution of our work lies in proposing a yet unexplored approach to optimize cli, & e\\
 18 & nt participation in federated learning, based on joint representativeness of th, & e\\
19 & overall data. This approach identifies a subset of collaborative clients that , & a\\
20 &  re filtered based on their suitability as an addition to the other available cl, & i\\
21 & ents. The proposed filtering algorithm discards a client when it is not suitabl, & e\\
22 &  for the given stage of the training process but keeps it available for later r, & o\\
23 & unds. To filter clients, we define a non-monotone combinatorial maximization pr, & o\\
24 &  blem, and propose a randomized greedy filtering algorithm that adapts the best , & t\\
25 &  heoretical guarantees for offline and online submodular maximization. Our appro, & a\\
26 &  ch not only promises to improve the convergence and performance of federated le, & a\\
27 & rning, but it also ensures the privacy and security of the client data. Overall, & ,\\
28 & our work presents a novel and promising solution for optimizing client partici, & p\\
29 &  ation in federated learning and contributes to advancing the state-of-the-art i, & n\\
30 &  this important research direction. We introduce \textit{client filtering} in FL (or Fil, & F\\
31 & L), which incorporates \textit{client filtering} into the most widely studied FL scheme,, &  \\
32 &  federated averaging (FedAvg). We first present a combinatorial objective for cl, & i\\
33 & ent filtering. We then present the randomized greedy algorithm that periodicall, & y\\
34 &  optimizes the objective by selecting a filtered subset of clients to be used f  & o\\
 & & \\
 \hline
\end{tabular}
\end{center}
\caption{Server Dataset for Shakespeare Experiments}
\label{tab:filtering} 
\end{table}

\subsection{Compute and resources}
In our experiments we simulate different FL benchmarks. We use a
cluster of NVIDIA Tesla V100 GPUs, all having 32GB memory, to sequentially train $K$ clients. We implement using PyTorch v1.10.2. The code is provided in the supplementary material (will be made open source).

\newpage
\section{Main Lemmas with Proofs}
\label{proofs-appendix}

\begin{lemma}
\label{proof-lemma0-appendix}
Under assumptions \ref{ass1}, \ref{ass2}, \ref{ass3}, \ref{ass4}, and \ref{ass5}, for the gap $\delta_t$ defined in \ref{epsilont}, we have
\begin{equation}
\mathbb{E}\left[\delta_t\right] \geq \delta
\end{equation}
for some constant $\delta$.
\end{lemma}
\begin{proof}
By $\mu$-strong convexity, Assumption \ref{ass1}, and L-smoothness, Assumption \ref{ass2}, we have
\begin{equation}
\begin{aligned}
    F(\bar{\mathbf{z}}_t) - F(\bar{\mathbf{v}}_t) &\leq \frac{1}{2\mu}\|\nabla F(\bar{\mathbf{z}}_t) - \nabla F(\bar{\mathbf{v}}_t)\|^2 + \frac{1}{2} \langle \nabla F(\bar{\mathbf{v}}_t), \bar{\mathbf{z}}_t - \bar{\mathbf{v}}_t\rangle\\
\end{aligned}
\end{equation}
By the Cauchy–Schwarz inequality, we have
\begin{equation}
\label{eq16}
\begin{aligned}
\mathbb{E}\left[F(\bar{\mathbf{z}}_t) - F(\bar{\mathbf{v}}_t)\right] &\leq \frac{1}{2 \mu} \mathbb{E} \left[\|\nabla F(\bar{\mathbf{z}}_t) - \nabla F(\bar{\mathbf{v}}_t)\|^2 \right] + \frac{1}{2}  \mathbb{E} \left[\| \nabla F(\bar{\mathbf{v}}_t)\| \|\bar{\mathbf{z}}_t - \bar{\mathbf{v}}_t\|\right]\\
&\leq \frac{1}{2 \mu} \sum_k \mathbb{E}\left[\|\nabla F_k^{\mathcal{D}}(\bar{\mathbf{z}}_t) - \nabla F_k^{\mathcal{D}}(\bar{\mathbf{v}}_t)\|^2 \right] + \frac{1}{2}  \mathbb{E} \left[\|\nabla F(\bar{\mathbf{v}}_t)\| \|\bar{\mathbf{z}}_t - \bar{\mathbf{v}}_t\|\right]\\ 
&\leq \frac{1}{2 \mu} \sum_k \sigma_k^2 + \frac{1}{2}  \mathbb{E}\left[G \|\bar{\mathbf{z}}_t - \bar{\mathbf{v}}_t\|\right], 
\end{aligned}
\end{equation}
where the last inequality follows from Assumption \ref{ass3} and Assumption \ref{ass4}.

Moreover,
\begin{equation}
\label{eq17}
\begin{aligned}
\|\bar{\mathbf{z}}_t - \bar{\mathbf{v}}_t\| &= \|\sum_{k \in[N]} p_k \mathbf{v}_t^k - \frac{1}{|\mathcal{S}_t^{f}|}\sum_{k \in \mathcal{S}_t^{f}} \mathbf{v}_t^k\| \\
&\leq \|\sum_{k \in[N]} p_k \mathbf{v}_t^k\| + \|\frac{1}{|\mathcal{S}_t^{f}|}\sum_{k \in \mathcal{S}_t^{*}} \mathbf{v}_t^k\| \\ 
&\leq \sum_{k \in[N]} p_k \| \mathbf{v}_t^k\| + \frac{1}{|\mathcal{S}_t^{f}|} \sum_{k \in \mathcal{S}_t^{f}} \| \mathbf{v}_t^k\| \\ 
&\leq 2 \sum_{k \in[N]} \| \mathbf{v}_t^k\| \\ 
&\leq 2 \sum_{k \in[N]} \left[\| \mathbf{v}_t^k- \mathbf{v}_k^*\| +\|\mathbf{v}_k^*\|  \right]
\end{aligned}
\end{equation}

Furthermore, by $\mu$-strong convexity, Assumption \ref{ass2}, and Assumption \ref{ass4}, we have 
\begin{equation}
\label{eq18}
    \|\bar{\mathbf{v}}^k_{t}-\mathbf{v}_k^*\| \leq \frac{1}{\mu} \|\nabla F_k\left(\mathbf{v}^k_t\right)\| \leq \frac{G}{\mu}
\end{equation}

Thus, by Eq. (\ref{eq17}) and Eq. (\ref{eq18}), we have
\begin{equation}
\label{eq19}
\begin{aligned}
\|\bar{\mathbf{z}}_t - \bar{\mathbf{v}}_t\| \leq  \sum_{k \in[N]} 2 \left[ \frac{G}{\mu} +\|\mathbf{v}_k^*\|  \right] 
\end{aligned}
\end{equation}

Using Eq. (\ref{eq16}) and Eq. (\ref{eq19}), we have
\begin{equation}
\begin{aligned}
\mathbb{E}\left[F(\bar{\mathbf{z}}_t) - F(\bar{\mathbf{v}}_t)\right] \leq  \frac{1}{2 \mu} \sum_k \sigma_k^2 + G \sum_{k \in[N]}  \left[ \frac{G}{\mu} +\|\mathbf{v}_k^*\|  \right] \leq -\delta\\ 
\end{aligned}
\end{equation}
for $\delta = -\frac{1}{2 \mu} \sum_k \sigma_k^2 - G \sum_{k \in[N]}  \left[ \frac{G}{\mu} +\|\mathbf{v}_k^*\|  \right]$, which does not depend on $T$ and only on the problem parameters.

Therefore, we obtain
\begin{equation}
\begin{aligned}
\mathbb{E}\left[\delta_t\right]=\mathbb{E}\left[F(\bar{\mathbf{v}}_t) - F(\bar{\mathbf{z}}_t)\right] \geq  \delta\\ 
\end{aligned}
\end{equation}
\end{proof}

\begin{lemma}
\label{proof-lemma1-appendix}
Under assumptions \ref{ass1}, \ref{ass2}, and \ref{ass4} for the sequences, $\bar{\mathbf{z}}_t$ and $\bar{\mathbf{v}}_t$, we have
\begin{equation}
\begin{aligned}
\label{lemma1}
\mathbb{E} \left[\|\bar{\mathbf{z}}_t - \bar{\mathbf{v}}_t\|^2\right] \leq \frac{G^2}{\mu^2}  - \frac{2\delta}{\mu}
\end{aligned}
\end{equation}
\end{lemma}

\begin{proof}
By $\mu$-strong convexity, Assumption \ref{ass1}, and L-smoothness, Assumption \ref{ass2}, we have

\begin{equation}
\begin{aligned}
\|\bar{\mathbf{z}}_t - \bar{\mathbf{v}}_t\|^2 &\stackrel{}{\leq} \frac{2}{\mu} \left( F(\bar{\mathbf{z}}_t) - F(\bar{\mathbf{v}}_t) - \langle \nabla F(\bar{\mathbf{v}}_t), \bar{\mathbf{z}}_t - \bar{\mathbf{v}}_t\rangle \right) \\
&\stackrel{(\ref{epsilont})}{\leq} \frac{2}{\mu} \left( -\delta_t + \langle \nabla F(\bar{\mathbf{v}}_t), \bar{\mathbf{v}}_t - \bar{\mathbf{z}}_t\rangle \right)\\
&\stackrel{}{\leq} \frac{2}{\mu} \left( \|\nabla F(\bar{\mathbf{v}}_t)\| \|\bar{\mathbf{z}}_t - \bar{\mathbf{v}}_t\| - \delta_t\right)
\end{aligned}
\end{equation}
where the last inequality follows from the Cauchy–Schwarz inequality.

Therefore, 
\begin{equation}
\begin{aligned}
\|\bar{\mathbf{z}}_t - \bar{\mathbf{v}}_t\|^2 - 2 \frac{\|\nabla F(\bar{\mathbf{v}}_t)\|}{\mu} \|\bar{\mathbf{z}}_t  - \bar{\mathbf{v}}_t\| &\leq -\frac{2}{\mu}  \delta_t
\end{aligned}
\end{equation}
Thus, 
\begin{equation}
\begin{aligned}
\|\bar{\mathbf{z}}_t - \bar{\mathbf{v}}_t\|^2 - 2 \frac{\|\nabla F(\bar{\mathbf{v}}_t)\|}{\mu} \|\bar{\mathbf{z}}_t  - \bar{\mathbf{v}}_t\| + \frac{\|\nabla F(\bar{\mathbf{v}}_t)\|^2}{\mu^2} \leq -\frac{2}{\mu}  \delta_t + \frac{\|\nabla F(\bar{\mathbf{v}}_t)\|^2}{\mu^2}
\end{aligned}
\end{equation}
Hence, 
\begin{equation}
\begin{aligned}
\left(\|\bar{\mathbf{z}}_t - \bar{\mathbf{v}}_t\| + \frac{\|\nabla F(\bar{\mathbf{v}}_t)\|}{\mu} \right)^2 \leq -\frac{2}{\mu}  \delta_t + \frac{\|\nabla F(\bar{\mathbf{v}}_t)\|^2}{\mu^2}
\end{aligned}
\end{equation}
Hence, 
\begin{equation}
\begin{aligned}
\|\bar{\mathbf{z}}_t - \bar{\mathbf{v}}_t\|^2 &\leq   \frac{\|\nabla F(\bar{\mathbf{v}}_t)\|^2}{\mu^2} -\frac{2}{\mu}  \delta_t\\
& \stackrel{(\ref{FD})}{\leq} \frac{\| \sum_{k=1}^N p_k \nabla F_k(\bar{\mathbf{v}}_t)\|^2}{\mu^2} -\frac{2}{\mu}  \delta_t \\
&\leq  \frac{ (\sum_{k=1}^N p_k\|  \nabla F_k(\bar{\mathbf{v}}_t)\|)^2}{\mu^2} -\frac{2}{\mu}  \delta_t \\
&\stackrel{(\ref{ass4})}{\leq} \frac{ (\sum_{k=1}^N p_k G)^2}{\mu^2} -\frac{2}{\mu}  \delta_t \\
\end{aligned}
\end{equation}
Therefore, 
\begin{equation}
\begin{aligned}
\mathbb{E} \left[\|\bar{\mathbf{z}}_t - \bar{\mathbf{v}}_t\|^2\right] \leq \frac{G^2}{\mu^2}  - \frac{2\mathbb{E} \left[\delta_t\right]}{\mu}
\end{aligned}
\end{equation}
Therefore, by Lemma \ref{proof-lemma0-appendix}, we have
\begin{equation}
\begin{aligned}
\mathbb{E} \left[\|\bar{\mathbf{z}}_t - \bar{\mathbf{v}}_t\|^2\right] \leq \frac{G^2}{\mu^2}  - \frac{2 \delta}{\mu}
\end{aligned}
\end{equation}
\end{proof}

\begin{lemma} 
\label{proof-lemma2-appendix}
Under assumptions \ref{ass1}, \ref{ass2}, \ref{ass3},  \ref{ass4}, \ref{ass5}, and \ref{samplingassumptio}, we have
\begin{equation}
\label{lemma2}
\left\|\bar{\mathbf{v}}_{t}-\mathbf{w}^*\right\|  \leq \rho .
\end{equation}
for some constant $\rho$.
\end{lemma}

\begin{proof}
Note that under Assumption \ref{ass1} and Assumption \ref{ass5}, we have $\left\|\sum_{k \in[N]} p_k \mathbf{v}_k^*-\mathbf{w}^*\right\|$ is also bounded by a constant $M$.
\begin{equation}
\begin{aligned}
\left\|\bar{\mathbf{v}}_{t}-\mathbf{w}^*\right\| & \leq \left\|\bar{\mathbf{v}}_{t}-\sum_{k \in[N]} p_k \mathbf{v}_k^*\right\| +\left\|\sum_{k \in[N]} p_k \mathbf{v}_k^*-\mathbf{w}^*\right\| \\
& \leq \left\|\bar{\mathbf{v}}_{t}-\sum_{k \in[N]} p_k \mathbf{v}_k^*\right\|+M \\
& \leq \sum_{k \in[N]} \left\|p_k\left(\bar{\mathbf{v}}^k_{t}-\mathbf{v}_k^*\right)\right\|+M \\
&\leq \sum_{k \in[N]} p_k \left\|\bar{\mathbf{v}}^k_{t}-\mathbf{v}_k^*\right\|+M
\end{aligned}
\end{equation}

By $\mu$-strong convexity, Assumption \ref{ass2}, we have 
\begin{equation}
    \|\bar{\mathbf{v}}^k_{t}-\mathbf{v}_k^*\| \leq \frac{1}{\mu} \|\nabla F_k\left(\bar{\mathbf{v}}^k_t\right)\|
\end{equation}

Therefore,
\begin{equation}
\begin{aligned}
\left\|\bar{\mathbf{v}}_{t}-\mathbf{w}^*\right\| & \leq \sum_{k \in[N]} \frac{p_k}{\mu} \left\|\nabla F_k\left(\bar{\mathbf{v}}_t^k\right)\right\|+M \\
& \stackrel{(\ref{ass4})}{\leq} \frac{G}{\mu}+M \\
& \leq \rho .
\end{aligned}
\end{equation}
where $\rho = \frac{G}{\mu}+M$.
\end{proof}

\begin{lemma}
\label{proof-lemma3-appendix}
\label{lemma-statement}
Under assumptions \ref{ass1}, \ref{ass2}, \ref{ass3},  \ref{ass4}, \ref{ass5}, and \ref{samplingassumptio}, for any virtual iteration t, for the above defined sequences, $\bar{\mathbf{z}}_t$ and $\bar{\mathbf{v}}_t$, we have
\begin{equation}
\begin{aligned}
\label{lemma3}
\mathbb{E} \left[\| \bar{\mathbf{w}}_{t} - \mathbf{\bar{v}}_{t}\|^2\right]
&\stackrel{}{\leq} \xi
\end{aligned}
\end{equation}
for some constant $\xi$. 
\end{lemma}

\begin{proof}
If not aggregating,
$$
\bar{\mathbf{w}}_{t+1}=\bar{\mathbf{v}}_{t+1} .
$$
Hence, 
\begin{equation}
    \mathbb{E} \left[\| \bar{\mathbf{w}}_{t+1} - \mathbf{\bar{v}}_{t+1}\|^2\right] = 0
\end{equation}

If aggregating, using Lemma 4 in \cite{li2019convergence}, we know that if $t+1 \in \mathcal{I}_E$, for sampling scheme in Assumption \ref{samplingassumptio}, we have
\begin{equation}
\label{unbiasedness}
\mathbb{E}\left(\overline{\mathbf{w}}_{t+1}\right)=\overline{\mathbf{z}}_{t+1}
\end{equation}
$$
\begin{aligned}
\| \bar{\mathbf{w}}_{t+1} - \mathbf{\bar{v}}_{t+1}\|^2 
&=  \|\bar{\mathbf{w}}_{t+1} - \bar{\mathbf{z}}_{t+1} + \bar{\mathbf{z}}_{t+1} -\mathbf{\bar{v}}_{t+1} \|^2  \\
&=  \|\bar{\mathbf{w}}_{t+1} - \bar{\mathbf{z}}_{t+1}\|^2 +  \| \bar{\mathbf{z}}_{t+1} -\mathbf{\bar{v}}_{t+1} \|^2  + 2 <\bar{\mathbf{w}}_{t+1} - \bar{\mathbf{z}}_{t+1}, \bar{\mathbf{z}}_{t+1} - \mathbf{\bar{v}}_{t+1}>
\end{aligned}
$$
When expectation is taken over $\mathcal{S}_{t+1}$, the last term vanishes due to the unbiasedness of $\overline{\mathbf{w}}_{t+1}$.

Therefore, 
$$
\begin{aligned}
\mathbb{E} \left[\| \bar{\mathbf{w}}_{t+1} - \mathbf{\bar{v}}_{t+1}\|^2\right] 
&=  \mathbb{E} \left[\|\bar{\mathbf{w}}_{t+1} - \bar{\mathbf{z}}_{t+1}\|^2\right] +
\mathbb{E} \left[\| \bar{\mathbf{z}}_{t+1} -\mathbf{\bar{v}}_{t+1} \|^2\right] 
\end{aligned}
$$
Moreover, using Lemma 5 in \cite{li2019convergence}, we know that if $t+1 \in \mathcal{I}_E$, for sampling scheme in assumption \ref{samplingassumptio}, the expected difference between $\overline{\mathbf{z}}_{t+1}$ and $\overline{\mathbf{w}}_{t+1}$ is bounded by
\begin{equation}
\label{variancebound}
\mathbb{E} \left[\left\|\overline{\mathbf{w}}_{t+1}-\overline{\mathbf{z}}_{t+1}\right\|^2\right] \leq J .
\end{equation}
where $J$ is a constant.

Therefore, using Lemma \ref{proof-lemma1-appendix}, we have
\begin{equation}
\label{eq:26}
\begin{aligned}
\mathbb{E} \left[\| \bar{\mathbf{w}}_{t+1} - \mathbf{\bar{v}}_{t+1}\|^2\right]
&\leq J +
\mathbb{E} \left[\| \bar{\mathbf{z}}_{t+1} -\mathbf{\bar{v}}_{t+1} \|^2\right] \\
&\stackrel{(\ref{lemma1})}{\leq} J + \frac{G^2}{\mu^2}  - \frac{2\delta}{\mu} \\
&\leq \xi
\end{aligned}
\end{equation}
for $\xi = J + \frac{G^2}{\mu^2}  - \frac{2\delta}{\mu}$.

\end{proof}

\begin{corollary}
\label{proof-lemma4-appendix}
\label{lemma4-statement}
Under assumptions \ref{ass1}, \ref{ass2}, \ref{ass3},  \ref{ass4},  \ref{ass5}, and \ref{samplingassumptio}, for any virtual iteration t, for the above defined sequences, $\bar{\mathbf{z}}_t$ and $\bar{\mathbf{v}}_t$, we have
\begin{equation}
\begin{aligned}
\mathbb{E} \left[ \left\langle \overline{\mathbf{w}}_{t} -\overline{\mathbf{v}}_{t} , \overline{\mathbf{v}}_{t}-\mathbf{w}^*\right\rangle \right] \leq \rho \sqrt{\xi}
\end{aligned}
\end{equation}
\end{corollary}

\begin{proof}
By Cauchy-Schwarz inequality we have
\begin{equation}
\label{cauchy_}
      \left\langle \overline{\mathbf{w}}_{t} -\overline{\mathbf{v}}_{t} , \overline{\mathbf{v}}_{t}-\mathbf{w}^*\right\rangle  \leq  \| \overline{\mathbf{w}}_{t} -\overline{\mathbf{v}}_{t} \|\| \overline{\mathbf{v}}_{t}-\mathbf{w}^*\| 
\end{equation}

Moreover, by Lemma \ref{proof-lemma2-appendix}, we have
\begin{equation}
\label{rho_}
\left\|\bar{\mathbf{v}}^{t}-\mathbf{w}^*\right\| \leq \rho .
\end{equation}

Therefore, 
\begin{equation}
\label{cauchy_simplified}
      \left\langle \overline{\mathbf{w}}_{t} -\overline{\mathbf{v}}_{t} , \overline{\mathbf{v}}_{t}-\mathbf{w}^*\right\rangle  \leq \rho  \| \overline{\mathbf{w}}_{t} -\overline{\mathbf{v}}_{t} \|
\end{equation}

Using Jensen inequality \cite{peajcariaac1992convex} and Lemma \ref{proof-lemma3-appendix}, it follows that
\begin{equation}
\begin{aligned}
\label{sqrt_xi_}
\mathbb{E} \left[\| \bar{\mathbf{w}}_{t} - \mathbf{\bar{v}}_{t}\|\right] \leq \mathbb{E} \left[\| \bar{\mathbf{w}}_{t} - \mathbf{\bar{v}}_{t}\|^2\right]^{\frac{1}{2}} 
&\stackrel{(\ref{lemma3})}{\leq} \sqrt{\xi}
\end{aligned}
\end{equation}

Combine equations (\ref{cauchy_simplified}) and (\ref{sqrt_xi_}), we have
\begin{equation}
\begin{aligned}
    \mathbb{E} \left[ \left\langle \overline{\mathbf{w}}_{t} -\overline{\mathbf{v}}_{t} , \overline{\mathbf{v}}_{t}-\mathbf{w}^*\right\rangle \right] &\stackrel{(\ref{cauchy_simplified})}{\leq} \mathbb{E} \left[ \rho \| \overline{\mathbf{w}}_{t} -\overline{\mathbf{v}}_{t} \|  \right]\\
    &\leq  \rho \mathbb{E} \left[\| \overline{\mathbf{w}}_{t} -\overline{\mathbf{v}}_{t} \|  \right] \\
    &\stackrel{(\ref{sqrt_xi_})}{\leq} \rho \sqrt{\xi}
\end{aligned}
\end{equation}
\end{proof}

\newpage
\section{Additional Experiments}

\subsection{FilFL (FedAvg with $\chi$GF and PoC) vs FedAvg (PoC)}
\label{appendix:fedavg_poc}

We compare the performance of FilFL (FedAvg with $\chi$GF) against FedAvg, both using PoC for client selection on CIFAR-10, FEMNIST, and Shakespeare. Fig.~\ref{fig:poc_test_acc}, Fig.~\ref{fig:poc_train_loss}, Fig.~\ref{fig:poc_test_loss}, and Fig.~\ref{fig:selectedclients1} illustrate the test accuracy, training loss, test loss, and number of accepeted clients respectively. The results on the Shakespeare dataset, with a small filtering dataset from a different distribution; specifically consisting of parts of this paper's introduction (see the filtering dataset in Appendix \ref{appendix:shakespearesetup}). 

Our results demonstrate that FilFL using either DGF or RGF achieves significantly better performance than FedAvg. In particular, as depicted in Fig.~\ref{fig:poc_test_acc}, FilFL with both filtering methods accomplishes accelerated training and attains approximately 5, 7, 10 percentage points higher test accuracy than FedAvg, for CIFAR-10, FEMNIST, and Shakespeare, respectively. After 100 to 200 training rounds, Fig.~\ref{fig:poc_train_loss}  displays a lower training loss for FedAvg, while Fig.~\ref{fig:poc_test_loss} shows an increasing test loss for it but a significantly reduced test loss for FilFL. This discrepancy can be attributed to the overfitting of FedAvg and the superior generalization ability of our approach. Finally, Fig.~\ref{fig:selectedclients1}, confirms the same observation that DGF accepts less clients than RGF.

\begin{figure}[H]
\begin{minipage}{0.33\textwidth}
\begin{tikzpicture}
  \node (img)  {\includegraphics[scale=0.24]{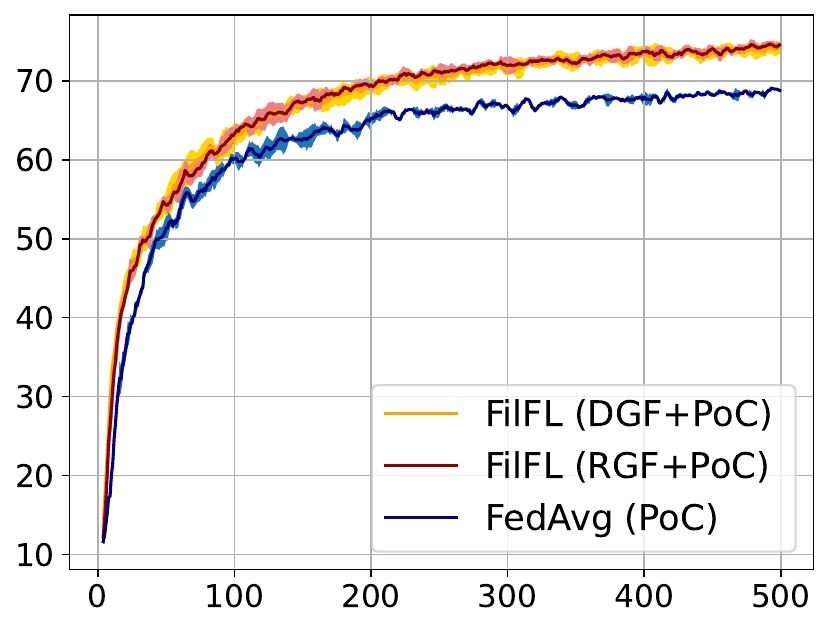}};
  \node[below=of img, node distance=0cm, rotate=0cm, anchor=center,yshift=3.8cm] {\tiny CIFAR-10};
  \node[below=of img, node distance=0cm, rotate=0cm, anchor=center,yshift=1.0cm] {\tiny Round};
  \node[left=of img, node distance=0cm, rotate=90, anchor=center,yshift=-1.0cm] {\tiny Test Accuracy};
 \end{tikzpicture}
\end{minipage}%
\begin{minipage}{0.33\textwidth}
\begin{tikzpicture}
  \node (img)  {\includegraphics[scale=0.24]{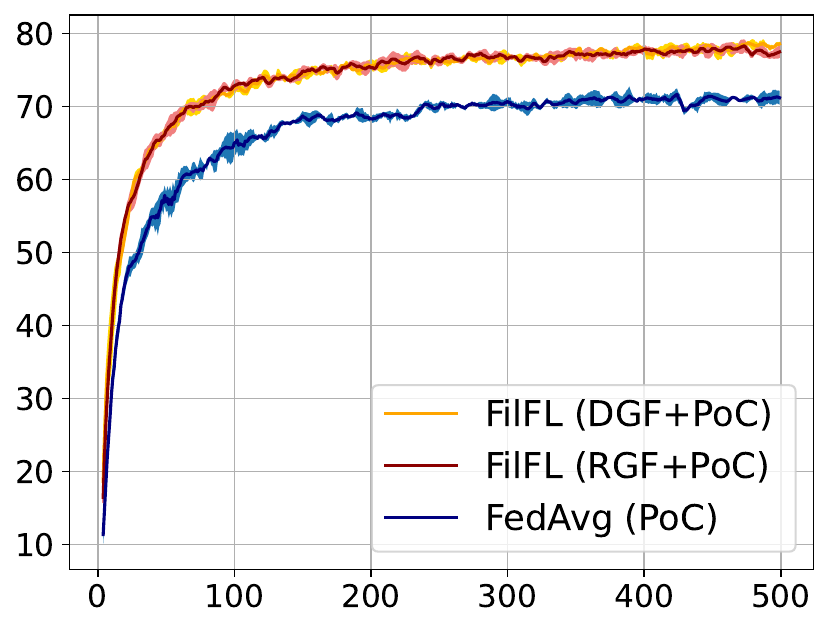}};
  \node[left=of img, node distance=0cm, rotate=90, anchor=center,yshift=-1.0cm] {\tiny Test Accuracy};
    \node[below=of img, node distance=0cm, rotate=0cm, anchor=center,yshift=3.8cm] {\tiny FEMNIST};
  \node[below=of img, node distance=0cm, rotate=0cm, anchor=center,yshift=1.0cm] {\tiny Round};
\end{tikzpicture}
\end{minipage}%
\begin{minipage}{0.33\textwidth}
\begin{tikzpicture}
  \node (img)  {\includegraphics[scale=0.24]{figures/Shakespeare_acc.pdf}};
  \node[left=of img, node distance=0cm, rotate=90, anchor=center,yshift=-1.0cm] {\tiny Test Accuracy};
    \node[below=of img, node distance=0cm, rotate=0cm, anchor=center,yshift=3.8cm] {\tiny Shakespeare};
  \node[below=of img, node distance=0cm, rotate=0cm, anchor=center,yshift=1.0cm] {\tiny Round};
\end{tikzpicture}
\end{minipage}%
\caption{FilFL vs FedAvg test accuracies both using PoC as a client selection method.}
\label{fig:poc_test_acc}
\end{figure}

\begin{figure}[H]
\begin{minipage}{0.33\textwidth}
\begin{tikzpicture}
  \node (img)  {\includegraphics[scale=0.24]{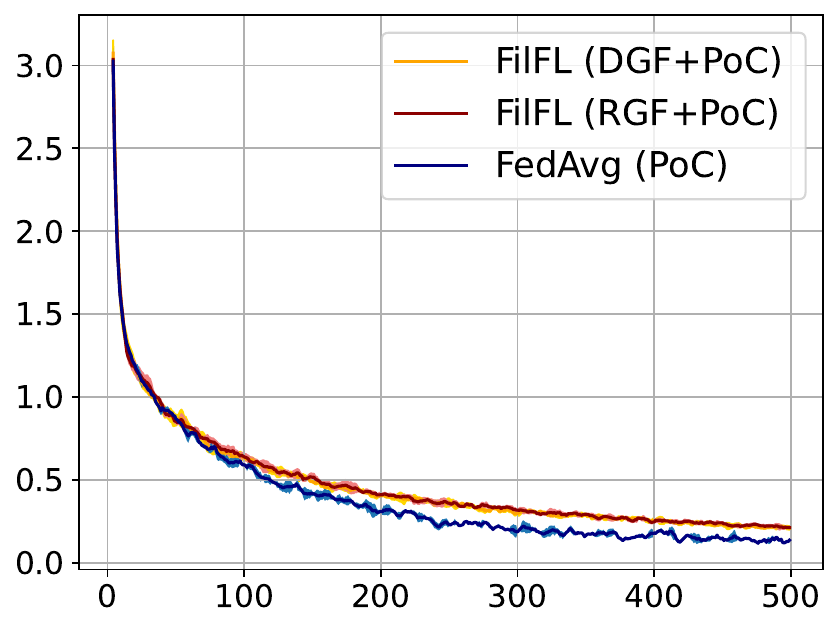}};
  \node[left=of img, node distance=0cm, rotate=90, anchor=center,yshift=-1.0cm] {\tiny Training Loss};
   \node[below=of img, node distance=0cm, rotate=0cm, anchor=center,yshift=3.8cm] {\tiny CIFAR-10};
  \node[below=of img, node distance=0cm, rotate=0cm, anchor=center,yshift=1.0cm] {\tiny Round};
 \end{tikzpicture}
\end{minipage}%
\begin{minipage}{0.33\textwidth}
\begin{tikzpicture}
  \node (img)  {\includegraphics[scale=0.24]{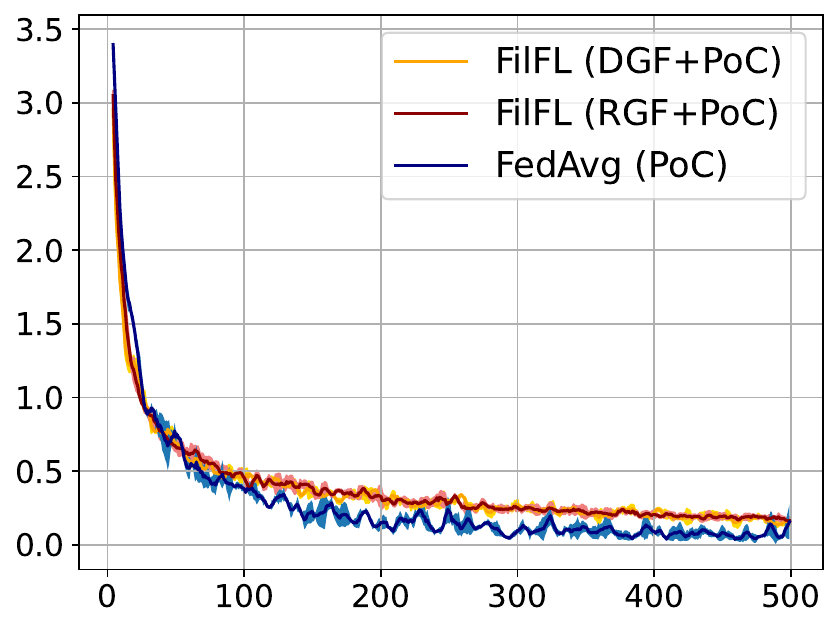}};
  \node[left=of img, node distance=0cm, rotate=90, anchor=center,yshift=-1.0cm] {\tiny Training Loss};
  \node[below=of img, node distance=0cm, rotate=0cm, anchor=center,yshift=3.8cm] {\tiny FEMNIST};
  \node[below=of img, node distance=0cm, rotate=0cm, anchor=center,yshift=1.0cm] {\tiny Round};
\end{tikzpicture}
\end{minipage}%
\begin{minipage}{0.33\textwidth}
\begin{tikzpicture}
  \node (img)  {\includegraphics[scale=0.24]{figures/loss_shakespeare.pdf}};
  \node[left=of img, node distance=0cm, rotate=90, anchor=center,yshift=-1.0cm] {\tiny Training Loss};
  \node[below=of img, node distance=0cm, rotate=0cm, anchor=center,yshift=3.8cm] {\tiny Shakespeare};
  \node[below=of img, node distance=0cm, rotate=0cm, anchor=center,yshift=1.0cm] {\tiny Round};
\end{tikzpicture}
\end{minipage}%
\caption{FilFL vs FedAvg training losses both using PoC as a client selection method.}
\label{fig:poc_train_loss}
\end{figure}

\begin{figure}[H]
\begin{minipage}{0.33\textwidth}
\begin{tikzpicture}
  \node (img)  {\includegraphics[scale=0.24]{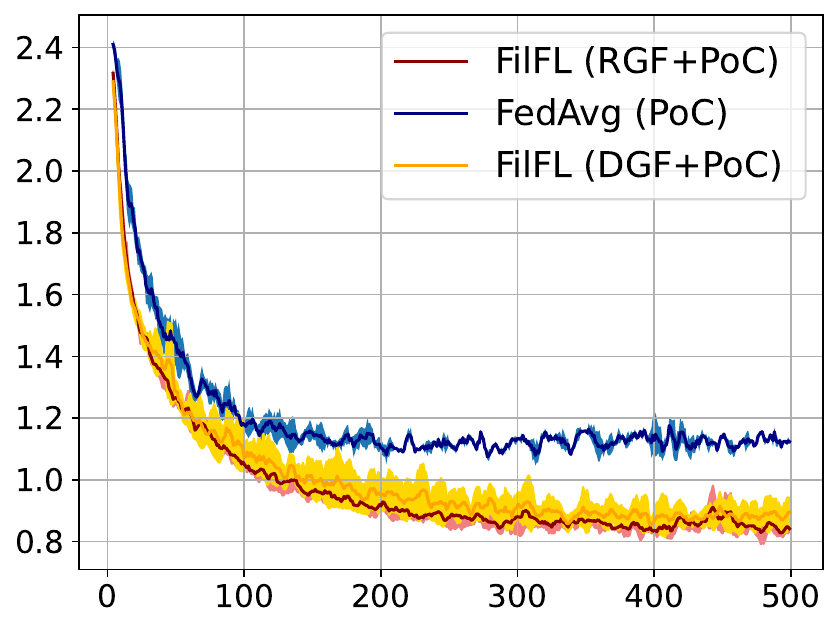}};
  \node[left=of img, node distance=0cm, rotate=90, anchor=center,yshift=-1.0cm] {\tiny Test Loss};
   \node[below=of img, node distance=0cm, rotate=0cm, anchor=center,yshift=3.8cm] {\tiny CIFAR-10};
  \node[below=of img, node distance=0cm, rotate=0cm, anchor=center,yshift=1.0cm] {\tiny Round};
 \end{tikzpicture}
\end{minipage}%
\begin{minipage}{0.33\textwidth}
\begin{tikzpicture}
  \node (img)  {\includegraphics[scale=0.24]{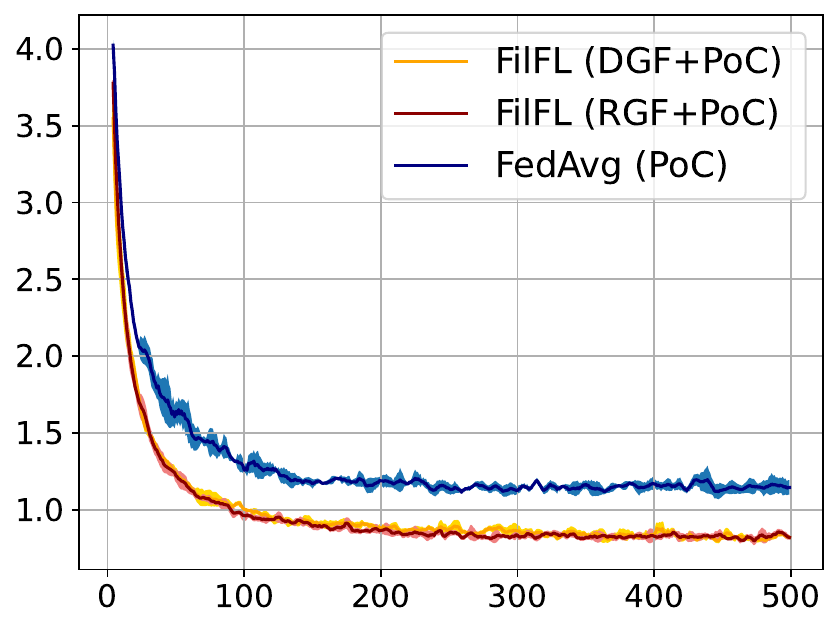}};
  \node[left=of img, node distance=0cm, rotate=90, anchor=center,yshift=-1.0cm] {\tiny Test Loss};
  \node[below=of img, node distance=0cm, rotate=0cm, anchor=center,yshift=3.8cm] {\tiny FEMNIST};
  \node[below=of img, node distance=0cm, rotate=0cm, anchor=center,yshift=1.0cm] {\tiny Round};
\end{tikzpicture}
\end{minipage}%
\begin{minipage}{0.33\textwidth}
\begin{tikzpicture}
  \node (img)  {\includegraphics[scale=0.24]{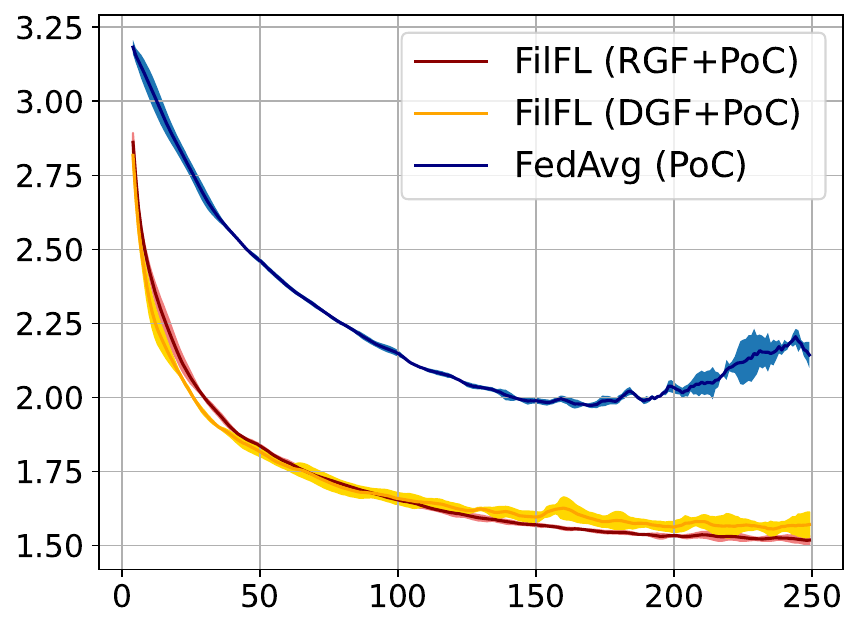}};
  \node[left=of img, node distance=0cm, rotate=90, anchor=center,yshift=-1.0cm] {\tiny Test Loss};
  \node[below=of img, node distance=0cm, rotate=0cm, anchor=center,yshift=3.8cm] {\tiny Shakespeare};
  \node[below=of img, node distance=0cm, rotate=0cm, anchor=center,yshift=1.0cm] {\tiny Round};
\end{tikzpicture}
\end{minipage}%
\caption{FilFL vs FedAvg test losses both using PoC as a client selection method.}
\label{fig:poc_test_loss}
\end{figure}
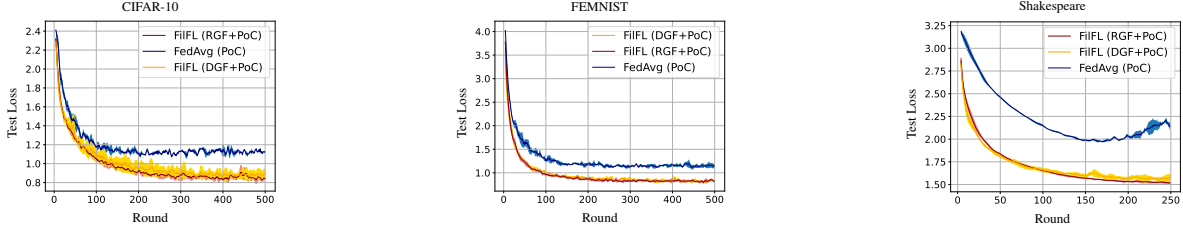

\begin{figure}[H]
\begin{minipage}{0.33\textwidth}
\begin{tikzpicture}
  \node (img)  {\includegraphics[scale=0.24]{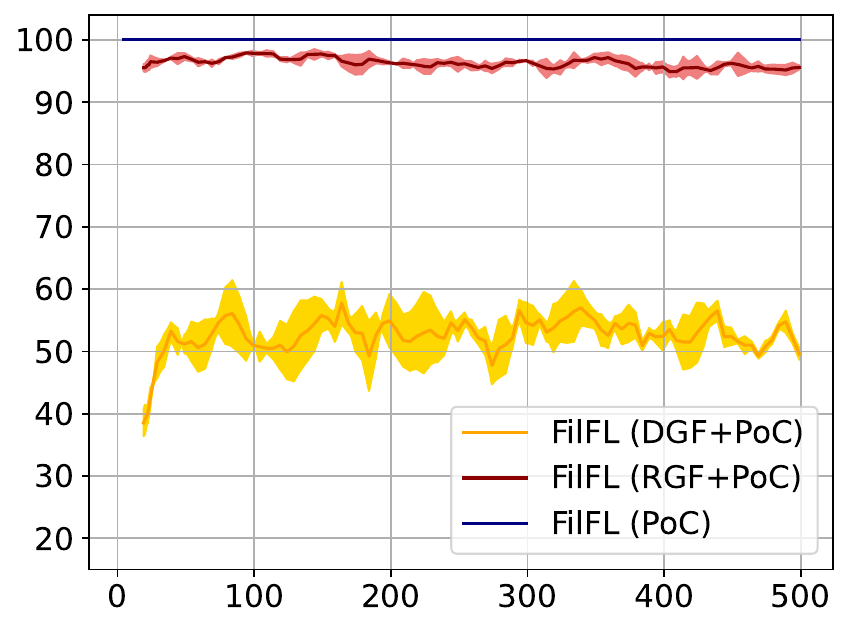}};
  \node[left=of img, node distance=0cm, rotate=90, anchor=center,yshift=-1.0cm] {\tiny $|\mathcal{S}^f|$};
   \node[below=of img, node distance=0cm, rotate=0cm, anchor=center,yshift=3.8cm] {\tiny CIFAR-10};
  \node[below=of img, node distance=0cm, rotate=0cm, anchor=center,yshift=1.0cm] {\tiny Round};
 \end{tikzpicture}
\end{minipage}%
\begin{minipage}{0.33\textwidth}
\begin{tikzpicture}
  \node (img)  {\includegraphics[scale=0.24]{figures/selected_clients_femnist.pdf}};
  \node[left=of img, node distance=0cm, rotate=90, anchor=center,yshift=-1.0cm] {\tiny $|\mathcal{S}^f|$};
  \node[below=of img, node distance=0cm, rotate=0cm, anchor=center,yshift=3.8cm] {\tiny FEMNIST};
  \node[below=of img, node distance=0cm, rotate=0cm, anchor=center,yshift=1.0cm] {\tiny Round};
\end{tikzpicture}
\end{minipage}%
\begin{minipage}{0.33\textwidth}
\begin{tikzpicture}
  \node (img)  {\includegraphics[scale=0.24]{figures/selected_clients_shakespeare.pdf}};
  \node[left=of img, node distance=0cm, rotate=90, anchor=center,yshift=-1.0cm] {\tiny $|\mathcal{S}^f|$};
  \node[below=of img, node distance=0cm, rotate=0cm, anchor=center,yshift=3.8cm] {\tiny Shakespeare};
  \node[below=of img, node distance=0cm, rotate=0cm, anchor=center,yshift=1.0cm] {\tiny Round};
\end{tikzpicture}
\end{minipage}%
\caption{FilFL vs FedAvg number of filtered-in clients with PoC as a client selection method.}
\label{fig:selectedclients1}
\end{figure}

\subsection{FilFL (FedAvg with $\chi$GF and DivFL) vs FedAvg (DivFL)}
\label{appendix:divfl}
As shown in \cite{balakrishnan2021diverse}, FedAvg with DivFL performs better than FedAvg with RS or PoC. However, it remains computationally more expensive than both selection methods. We compare FilFL using DivFL against FedAvg (DivFL). Fig. \ref{fig:divfl_cifar10} shows that on the CIFAR-10 dataset, DGF achieves 3 percentage points higher accuracy than FedAvg (DivFL) (left plot). While FedAvg (DivFL) exhibits slightly lower training loss than FilFL (middle plot), it suffers from a larger test loss (right plot), which can be due to the overfitting of FedAvg (DivFL) and the better generalization capabilities of FilFL. Therefore, FilFL with DivFL empirically outperforms FedAvg (DivFL). 

\begin{figure}[H]
\begin{minipage}{0.33\textwidth}
\begin{tikzpicture}
  \node (img)  {\includegraphics[scale=0.24]{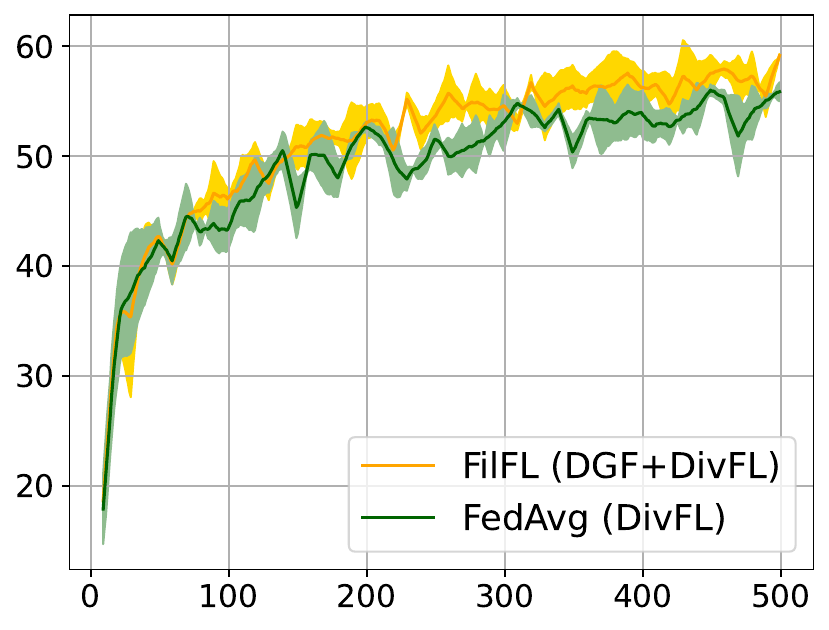}};
  \node[below=of img, node distance=0cm, rotate=0cm, anchor=center,yshift=3.8cm] {\tiny CIFAR-10};
  \node[below=of img, node distance=0cm, rotate=0cm, anchor=center,yshift=1.0cm] {\tiny Round};
  \node[left=of img, node distance=0cm, rotate=90, anchor=center,yshift=-1.0cm] {\tiny Test Accuracy};
 \end{tikzpicture}
\end{minipage}%
\begin{minipage}{0.33\textwidth}
\begin{tikzpicture}
  \node (img)  {\includegraphics[scale=0.24]{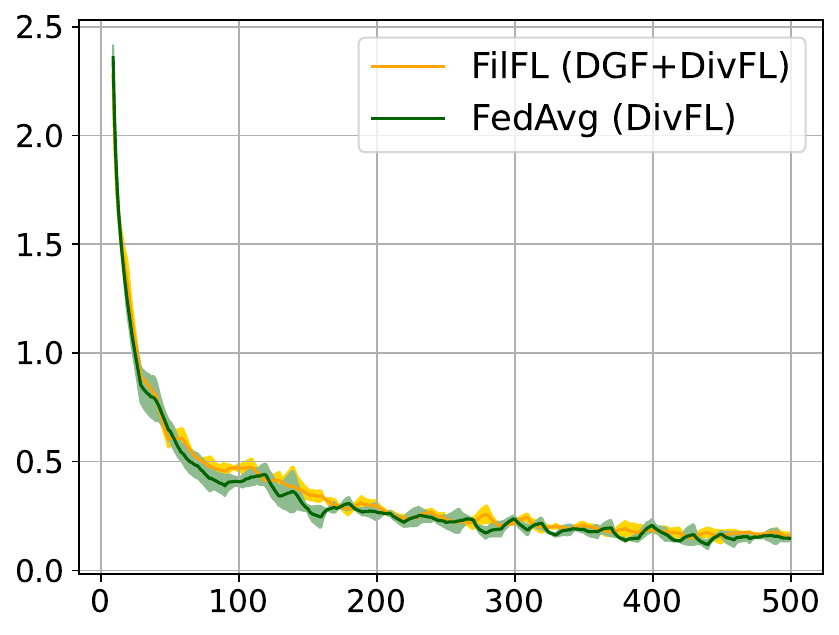}};
  \node[left=of img, node distance=0cm, rotate=90, anchor=center,yshift=-1.0cm] {\tiny Training Loss};
    \node[below=of img, node distance=0cm, rotate=0cm, anchor=center,yshift=3.8cm] {\tiny CIFAR-10};
  \node[below=of img, node distance=0cm, rotate=0cm, anchor=center,yshift=1.0cm] {\tiny Round};
\end{tikzpicture}
\end{minipage}%
\begin{minipage}{0.33\textwidth}
\begin{tikzpicture}
  \node (img)  {\includegraphics[scale=0.24]{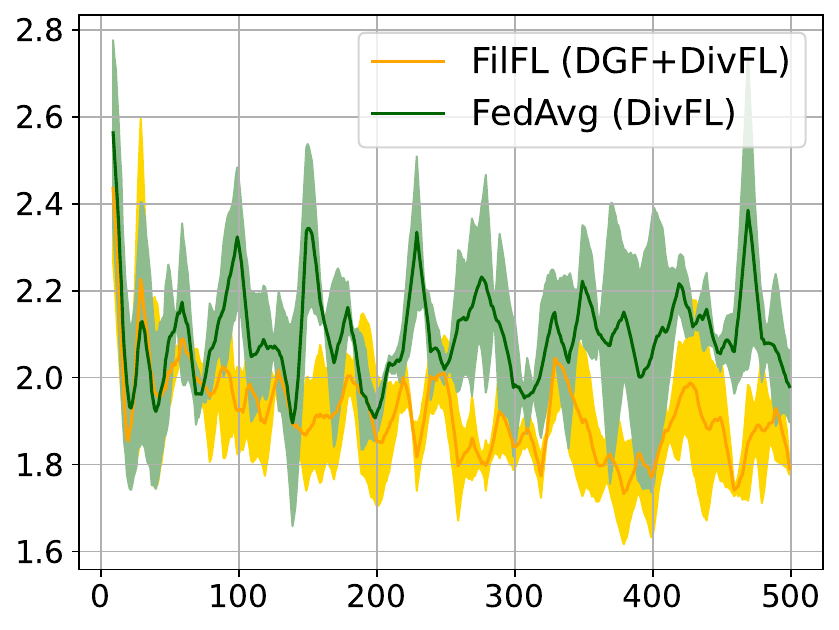}};
  \node[left=of img, node distance=0cm, rotate=90, anchor=center,yshift=-1.0cm] {\tiny Test Loss};
    \node[below=of img, node distance=0cm, rotate=0cm, anchor=center,yshift=3.8cm] {\tiny CIFAR-10};
  \node[below=of img, node distance=0cm, rotate=0cm, anchor=center,yshift=1.0cm] {\tiny Round};
\end{tikzpicture}
\end{minipage}%
\caption{FilFL (FedAvg + $\chi$GF + RS) vs FedAvg (DivFL) without filtering on CIFAR-10 dataset.}
\label{fig:divfl_cifar10}
\end{figure}
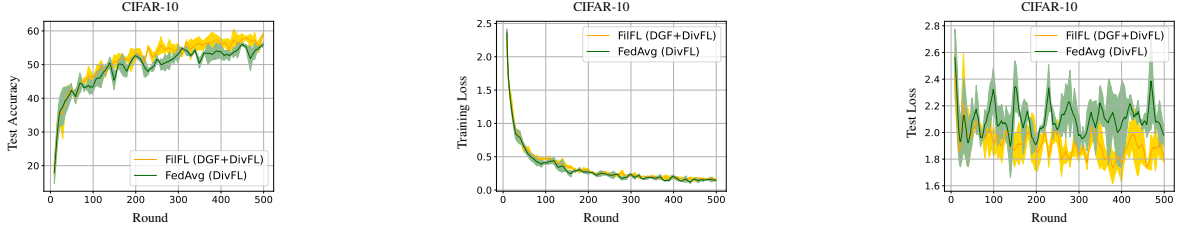

\subsection{FilFL (FedProx with $\chi$GF and RS) vs FedProx (RS)}
\label{appendix:fedprox_rs}
We compare the performance of FilFL (FedProx with $\chi$GF) against FedProx, both using RS for selection. Fig. \ref{fig:fedprox_shakespeare} demonstrates that FilFL using $\chi$GF achieves significantly superior performance compared to FedProx on the Shakespeare dataset. Specifically, the left plot illustrates that FilFL with DGF and RGF achieves approximately 3 and 6 percentage points higher test accuracy, respectively than FedProx. The middle plot reveal lower training loss for FilFL than FedProx. Finally, the right plot confirms the same observation, that DGF accepts less clients than RGF.

\begin{figure}[H]
\begin{minipage}{0.33\textwidth}
\begin{tikzpicture}
  \node (img)  {\includegraphics[scale=0.24]{figures/FedProx_Shakespeare_val_acc_20LocalSteps.pdf}};
  \node[below=of img, node distance=0cm, rotate=0cm, anchor=center,yshift=3.8cm] {\tiny Shakespeare};
  \node[below=of img, node distance=0cm, rotate=0cm, anchor=center,yshift=1.0cm] {\tiny Round};
  \node[left=of img, node distance=0cm, rotate=90, anchor=center,yshift=-1.0cm] {\tiny Test Accuracy};
 \end{tikzpicture}
\end{minipage}%
\begin{minipage}{0.33\textwidth}
\begin{tikzpicture}
  \node (img)  {\includegraphics[scale=0.24]{figures/FedProx_Shakespeare_train_loss_20LocalSteps.pdf}};
  \node[left=of img, node distance=0cm, rotate=90, anchor=center,yshift=-1.0cm] {\tiny Training Loss};
    \node[below=of img, node distance=0cm, rotate=0cm, anchor=center,yshift=3.8cm] {\tiny Shakespeare};
  \node[below=of img, node distance=0cm, rotate=0cm, anchor=center,yshift=1.0cm] {\tiny Round};
\end{tikzpicture}
\end{minipage}%
\begin{minipage}{0.33\textwidth}
\begin{tikzpicture}
  \node (img)  {\includegraphics[scale=0.24]{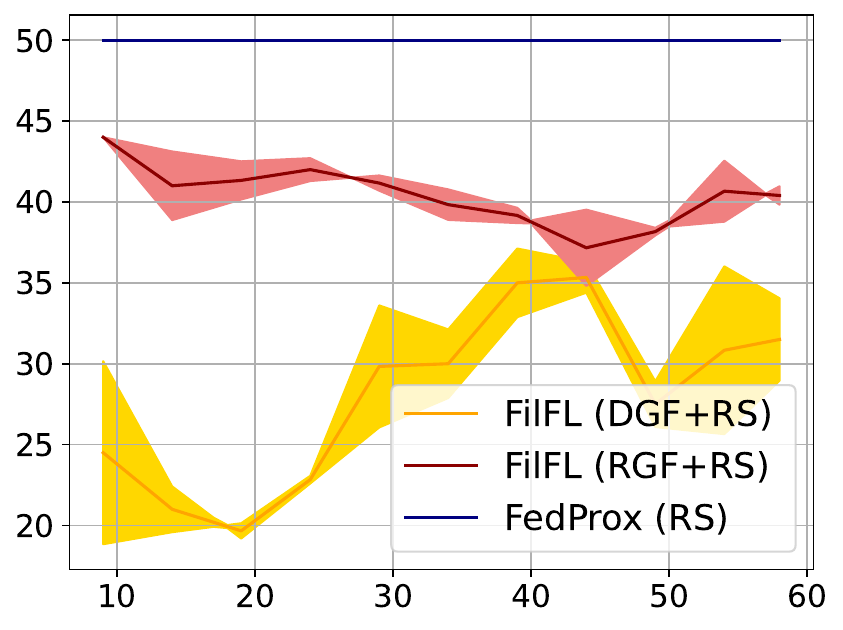}};
  \node[left=of img, node distance=0cm, rotate=90, anchor=center,yshift=-1.0cm] {\tiny $|\mathcal{S}^f|$};
    \node[below=of img, node distance=0cm, rotate=0cm, anchor=center,yshift=3.8cm] {\tiny Shakespeare};
  \node[below=of img, node distance=0cm, rotate=0cm, anchor=center,yshift=1.0cm] {\tiny Round};
\end{tikzpicture}
\end{minipage}%
\caption{FilFL (FedProx + $\chi$GF + RS) vs FedProx (RS) without filtering on Shakespeare dataset.}
\label{fig:fedprox_shakespeare}
\end{figure}

\subsection{FilFL Sensitivity to Filtering Periodicity $h$}
\label{appendix:periodicity}
We simulate two behavior heterogeneity settings on the FEMNIST dataset. In the first setting (Fig.~\ref{fig:app:periodsensitivityfemnist}), the environment changes every $5$ rounds. In the second setting (Fig.~\ref{fig:periodsensitivityfemnist2}), the environment changes every $20$ rounds. For the first setting, we experiment with different periodicities $h \in \{1,3,5\}$. Moreover, for the second setting, we experiment with different periodicities $h \in \{10,20\}$. We find that FilFL's performance in both settings is similar for the different values of $h$. However, from a computational perspective, our approach is more efficient for larger periodicities $h$.

\begin{figure}[H]
\begin{minipage}{0.25\textwidth}
\begin{tikzpicture}
  \node (img)  {\includegraphics[scale=0.2]{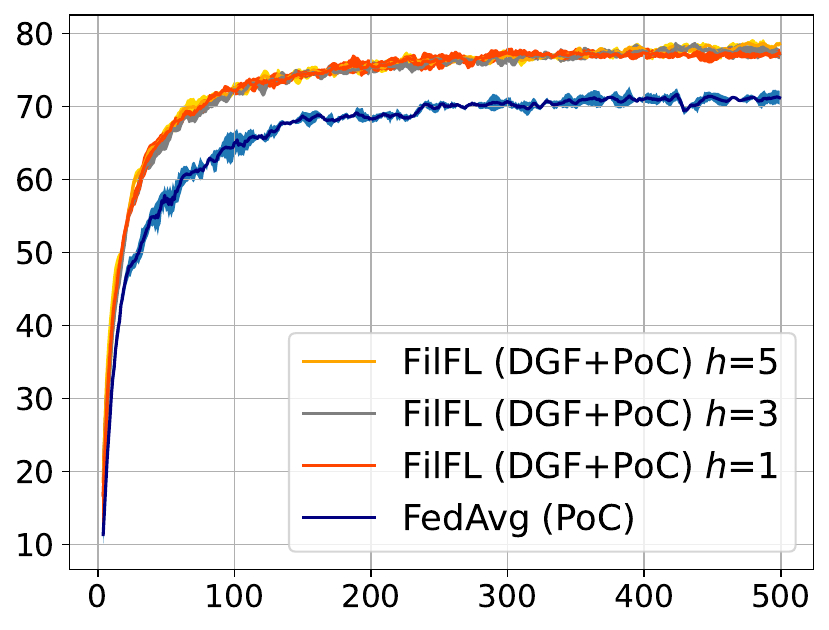}};
  \node[below=of img, node distance=0cm, rotate=0cm, anchor=center,yshift=3.4cm] {\tiny FEMNIST};
  \node[below=of img, node distance=0cm, rotate=0cm, anchor=center,yshift=1.0cm] {\tiny Round};
  \node[left=of img, node distance=0cm, rotate=90, anchor=center,yshift=-1.0cm] {\tiny Test Accuracy};
 \end{tikzpicture}
\end{minipage}%
\begin{minipage}{0.25\textwidth}
\begin{tikzpicture}
  \node (img)  {\includegraphics[scale=0.2]{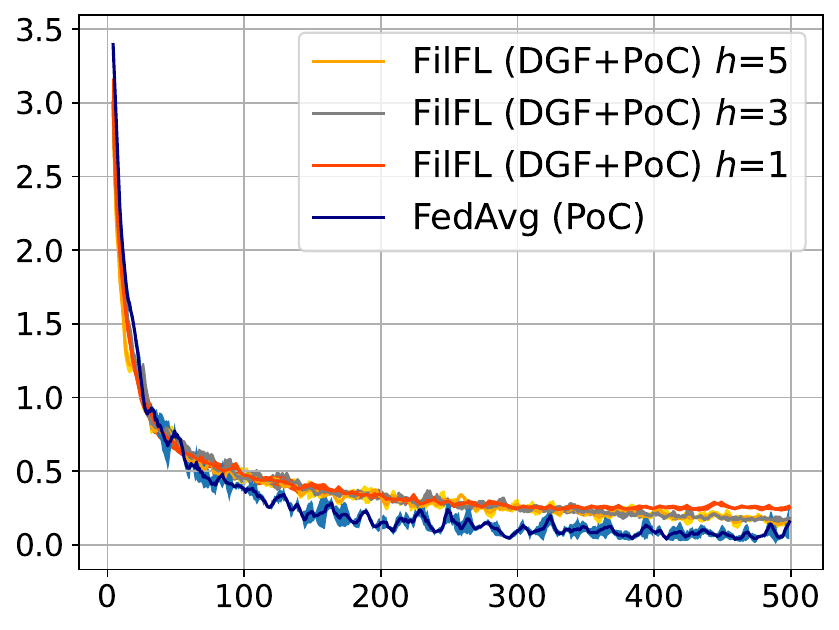}};
  \node[left=of img, node distance=0cm, rotate=90, anchor=center,yshift=-1.0cm] {\tiny Training Loss};
    \node[below=of img, node distance=0cm, rotate=0cm, anchor=center,yshift=3.4cm] {\tiny FEMNIST};
  \node[below=of img, node distance=0cm, rotate=0cm, anchor=center,yshift=1.0cm] {\tiny Round};
\end{tikzpicture}
\end{minipage}%
\begin{minipage}{0.25\textwidth}
\begin{tikzpicture}
  \node (img)  {\includegraphics[scale=0.2]{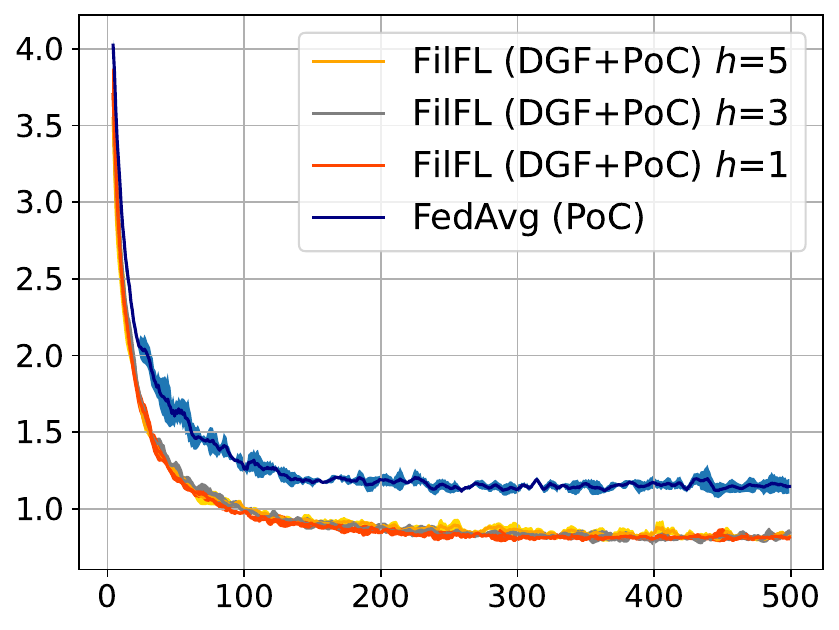}};
  \node[left=of img, node distance=0cm, rotate=90, anchor=center,yshift=-1.0cm] {\tiny Test Loss};
    \node[below=of img, node distance=0cm, rotate=0cm, anchor=center,yshift=3.4cm] {\tiny FEMNIST};
  \node[below=of img, node distance=0cm, rotate=0cm, anchor=center,yshift=1.0cm] {\tiny Round};
\end{tikzpicture}
\end{minipage}%
\begin{minipage}{0.25\textwidth}
\begin{tikzpicture}
  \node (img)  {\includegraphics[scale=0.2]{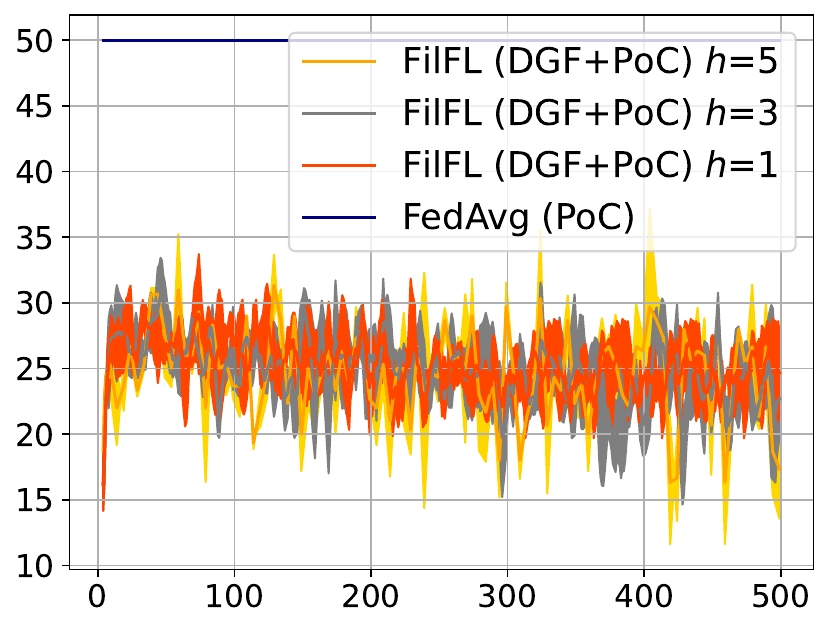}};
  \node[left=of img, node distance=0cm, rotate=90, anchor=center,yshift=-1.0cm] {\tiny $|\mathcal{S}^f|$};
    \node[below=of img, node distance=0cm, rotate=0cm, anchor=center,yshift=3.4cm] {\tiny FEMNIST};
  \node[below=of img, node distance=0cm, rotate=0cm, anchor=center,yshift=1.0cm] {\tiny Round};
\end{tikzpicture}
\end{minipage}%
\caption{FilFL (FedAvg + $\chi$GF + PoC) sensitivity to periodicity $h$ on FEMNIST dataset.}
\label{fig:app:periodsensitivityfemnist}
\end{figure}

\begin{figure}[H]
\begin{minipage}{0.25\textwidth}
\begin{tikzpicture}
  \node (img)  {\includegraphics[scale=0.2]{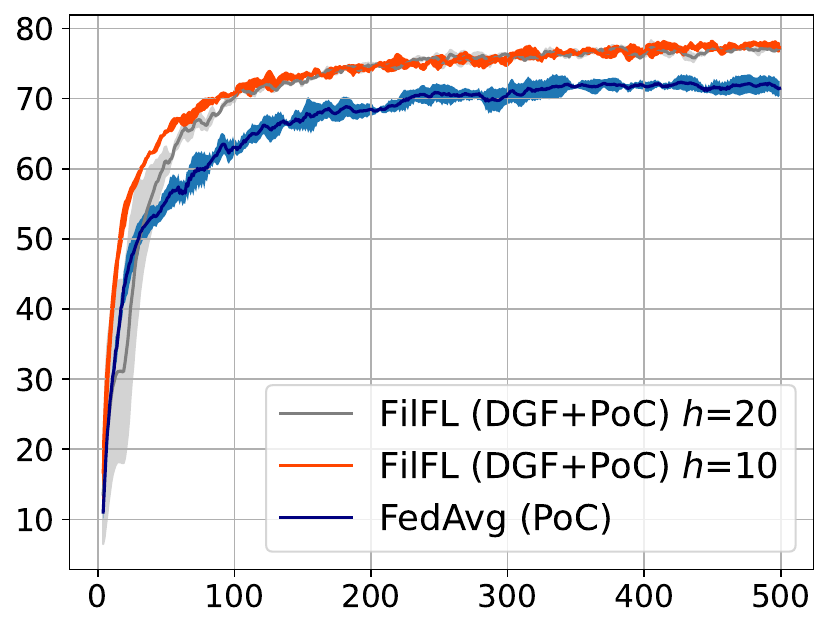}};
  \node[below=of img, node distance=0cm, rotate=0cm, anchor=center,yshift=3.4cm] {\tiny FEMNIST};
  \node[below=of img, node distance=0cm, rotate=0cm, anchor=center,yshift=1.0cm] {\tiny Round};
  \node[left=of img, node distance=0cm, rotate=90, anchor=center,yshift=-1.0cm] {\tiny Test Accuracy};
 \end{tikzpicture}
\end{minipage}%
\begin{minipage}{0.25\textwidth}
\begin{tikzpicture}
  \node (img)  {\includegraphics[scale=0.2]{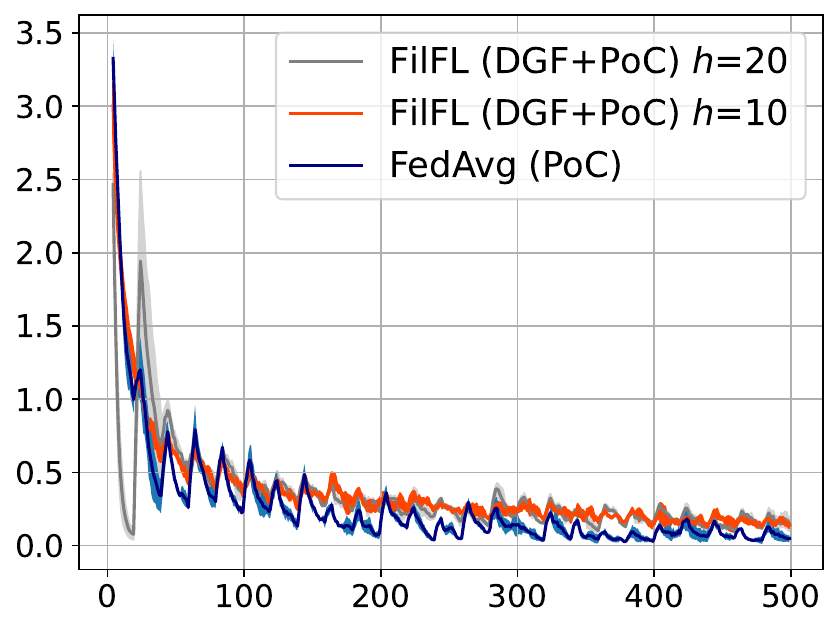}};
  \node[left=of img, node distance=0cm, rotate=90, anchor=center,yshift=-1.0cm] {\tiny Training Loss};
    \node[below=of img, node distance=0cm, rotate=0cm, anchor=center,yshift=3.4cm] {\tiny FEMNIST};
  \node[below=of img, node distance=0cm, rotate=0cm, anchor=center,yshift=1.0cm] {\tiny Round};
\end{tikzpicture}
\end{minipage}%
\begin{minipage}{0.25\textwidth}
\begin{tikzpicture}
  \node (img)  {\includegraphics[scale=0.2]{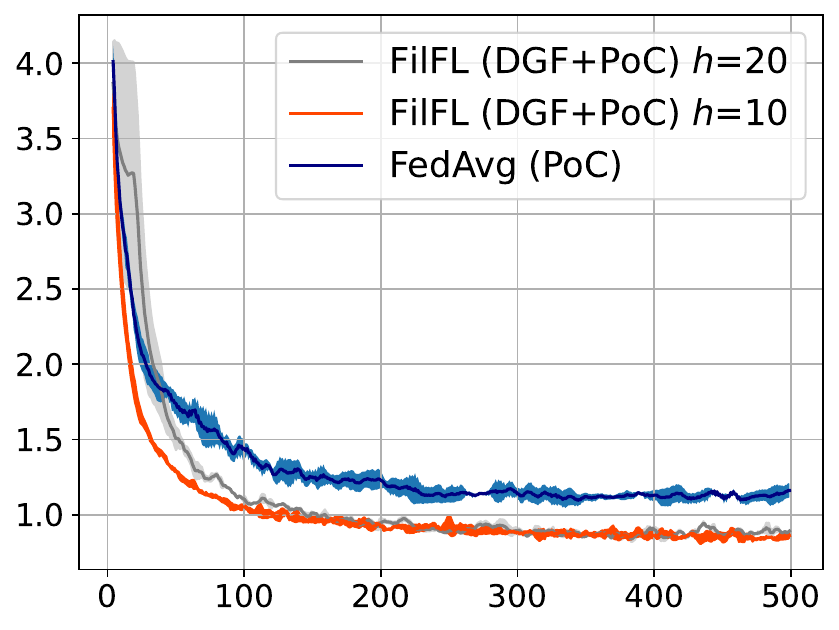}};
  \node[left=of img, node distance=0cm, rotate=90, anchor=center,yshift=-1.0cm] {\tiny Test Loss};
    \node[below=of img, node distance=0cm, rotate=0cm, anchor=center,yshift=3.4cm] {\tiny FEMNIST};
  \node[below=of img, node distance=0cm, rotate=0cm, anchor=center,yshift=1.0cm] {\tiny Round};
\end{tikzpicture}
\end{minipage}%
\begin{minipage}{0.25\textwidth}
\begin{tikzpicture}
  \node (img)  {\includegraphics[scale=0.2]{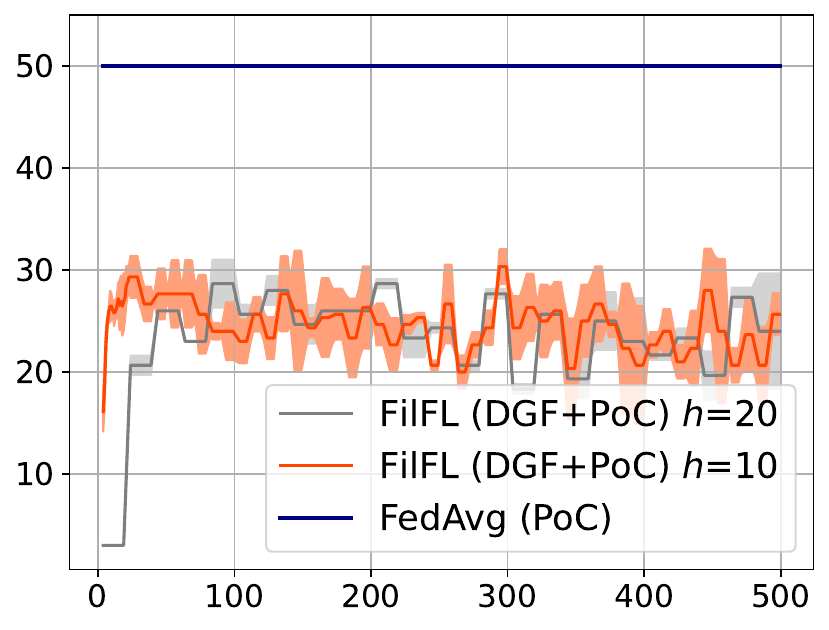}};
  \node[left=of img, node distance=0cm, rotate=90, anchor=center,yshift=-1.0cm] {\tiny $|\mathcal{S}^f|$};
    \node[below=of img, node distance=0cm, rotate=0cm, anchor=center,yshift=3.4cm] {\tiny FEMNIST};
  \node[below=of img, node distance=0cm, rotate=0cm, anchor=center,yshift=1.0cm] {\tiny Round};
\end{tikzpicture}
\end{minipage}%
\caption{FilFL (FedAvg + $\chi$GF + PoC) sensitivity to periodicity $h$ on FEMNIST dataset.}
\label{fig:periodsensitivityfemnist2}
\end{figure}

\subsection{FilFL Sensitivity to filtering Dataset Size $m$}
\label{appendix:pubdataset}
We evaluate the effectiveness of FilFL under different filtering datasets scenarios, showing its robustness across various sizes and distributions. In the Shakespeare experiment, we use small filtering datasets consisting of parts of this paper's introduction, containing only 34, 17, and 8 samples. Fig. \ref{fig:app:sizesensitivityshakespeare}, shows that FilFL remains effective even with tiny filtering datasets with different distributions than the clients' datasets. The left plot demonstrates higher test accuracy for FilFL than FedAvg, with a slight advantage for larger values of $m$.
The middle and right plots also reveal lower training loss for smaller $m$ and lower test loss for larger $m$, indicating that larger $m$ leads to better generalization. Hence, FilFL can perform well even with a few data points in the filtering dataset, even in distribution shifts, making our approach a versatile and robust method.

\begin{figure}[H]
\begin{minipage}{0.25\textwidth}
\begin{tikzpicture}
  \node (img)  {\includegraphics[scale=0.2]{figures/acc_shakespeare_dataset_size.pdf}};
  \node[below=of img, node distance=0cm, rotate=0cm, anchor=center,yshift=3.4cm] {\tiny Shakespeare};
  \node[below=of img, node distance=0cm, rotate=0cm, anchor=center,yshift=1.0cm] {\tiny Round};
  \node[left=of img, node distance=0cm, rotate=90, anchor=center,yshift=-1.0cm] {\tiny Test Accuracy};
 \end{tikzpicture}
\end{minipage}%
\begin{minipage}{0.25\textwidth}
\begin{tikzpicture}
  \node (img)  {\includegraphics[scale=0.2]{figures/train_loss_shakespeare_dataset_size.pdf}};
  \node[left=of img, node distance=0cm, rotate=90, anchor=center,yshift=-1.0cm] {\tiny Training Loss};
    \node[below=of img, node distance=0cm, rotate=0cm, anchor=center,yshift=3.4cm] {\tiny Shakespeare};
  \node[below=of img, node distance=0cm, rotate=0cm, anchor=center,yshift=1.0cm] {\tiny Round};
\end{tikzpicture}
\end{minipage}%
\begin{minipage}{0.25\textwidth}
\begin{tikzpicture}
  \node (img)  {\includegraphics[scale=0.2]{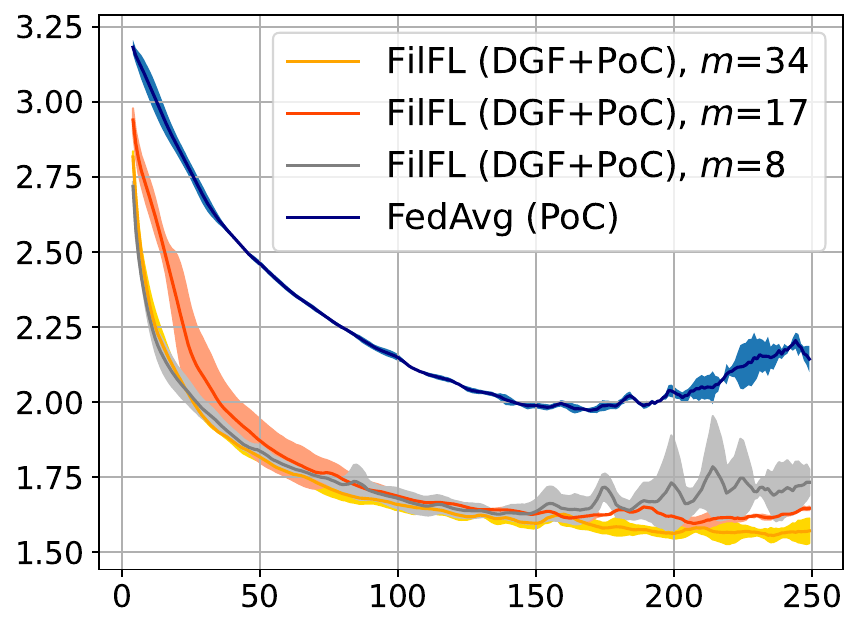}};
  \node[left=of img, node distance=0cm, rotate=90, anchor=center,yshift=-1.0cm] {\tiny Test Loss};
    \node[below=of img, node distance=0cm, rotate=0cm, anchor=center,yshift=3.4cm] {\tiny Shakespeare};
  \node[below=of img, node distance=0cm, rotate=0cm, anchor=center,yshift=1.0cm] {\tiny Round};
\end{tikzpicture}
\end{minipage}%
\begin{minipage}{0.25\textwidth}
\begin{tikzpicture}
  \node (img)  {\includegraphics[scale=0.2]{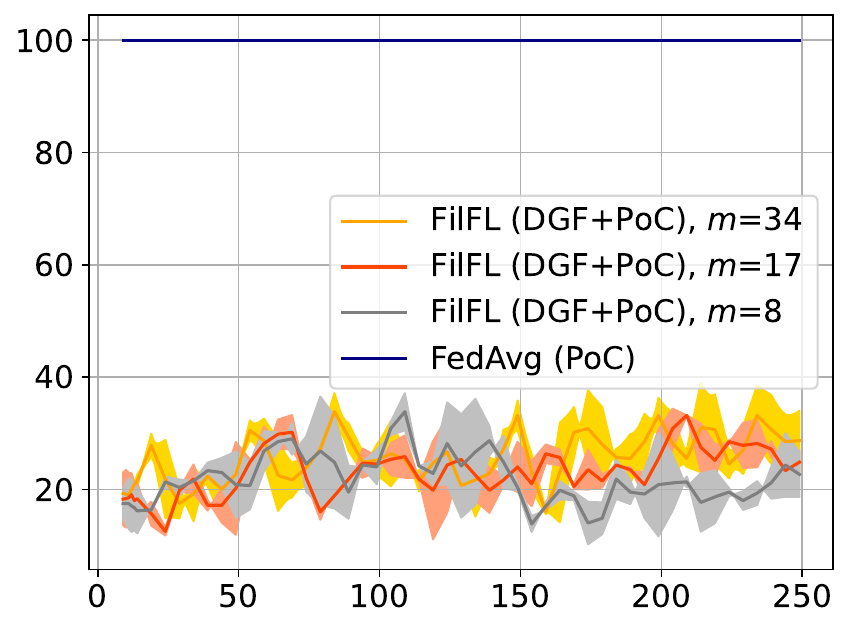}};
  \node[left=of img, node distance=0cm, rotate=90, anchor=center,yshift=-1.0cm] {\tiny $\mathcal{S}^f$};
    \node[below=of img, node distance=0cm, rotate=0cm, anchor=center,yshift=3.4cm] {\tiny Shakespeare};
  \node[below=of img, node distance=0cm, rotate=0cm, anchor=center,yshift=1.0cm] {\tiny Round};
\end{tikzpicture}
\end{minipage}%
\caption{FilFL (FedAvg with DGF) sensitivity to filtering dataset size $m$ on Shakespeare dataset.}
\label{fig:app:sizesensitivityshakespeare}
\end{figure}

In the FEMNIST experiment, we used filtering datasets with similar distributions to the clients, containing 500, 1000, and 2000 samples. Fig.~\ref{fig:sizesensitivityfemnist} shows that FilFL remains effective with the different sizes of the filtering dataset. All the plots demonstrate the effectiveness of FilFL compared to FedAvg across different values of $m$. Therefore, it is more efficient to use a small filtering dataset to reduce the computation cost of the oracle function while still preserving similar performance.

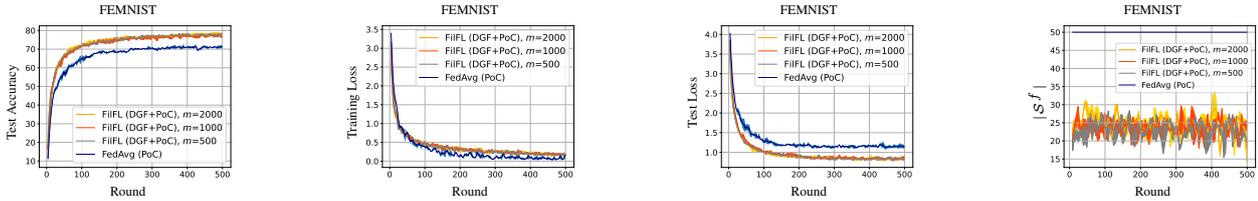
\begin{figure}[H]
\begin{minipage}{0.25\textwidth}
\begin{tikzpicture}
  \node (img)  {\includegraphics[scale=0.2]{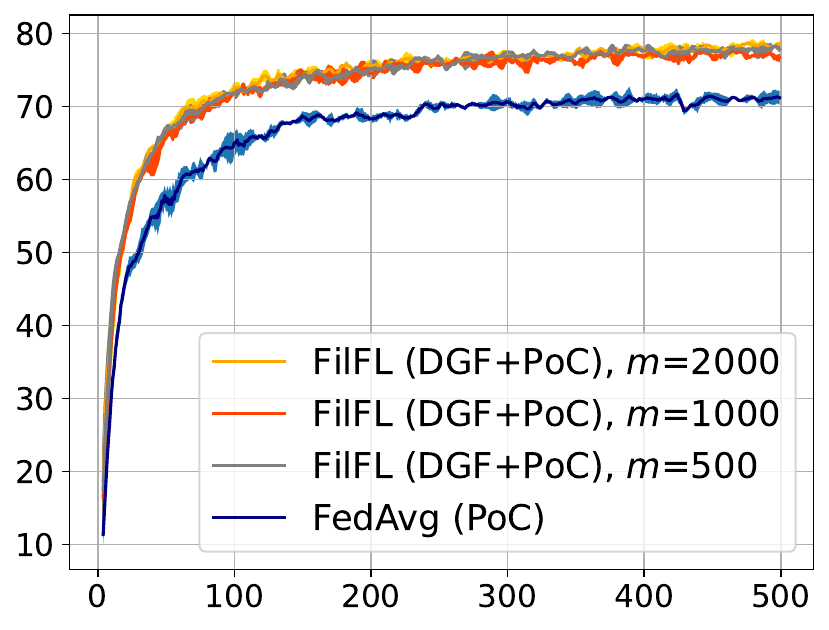}};
  \node[below=of img, node distance=0cm, rotate=0cm, anchor=center,yshift=3.4cm] {\tiny FEMNIST};
  \node[below=of img, node distance=0cm, rotate=0cm, anchor=center,yshift=1.0cm] {\tiny Round};
  \node[left=of img, node distance=0cm, rotate=90, anchor=center,yshift=-1.0cm] {\tiny Test Accuracy};
 \end{tikzpicture}
\end{minipage}%
\begin{minipage}{0.25\textwidth}
\begin{tikzpicture}
  \node (img)  {\includegraphics[scale=0.2]{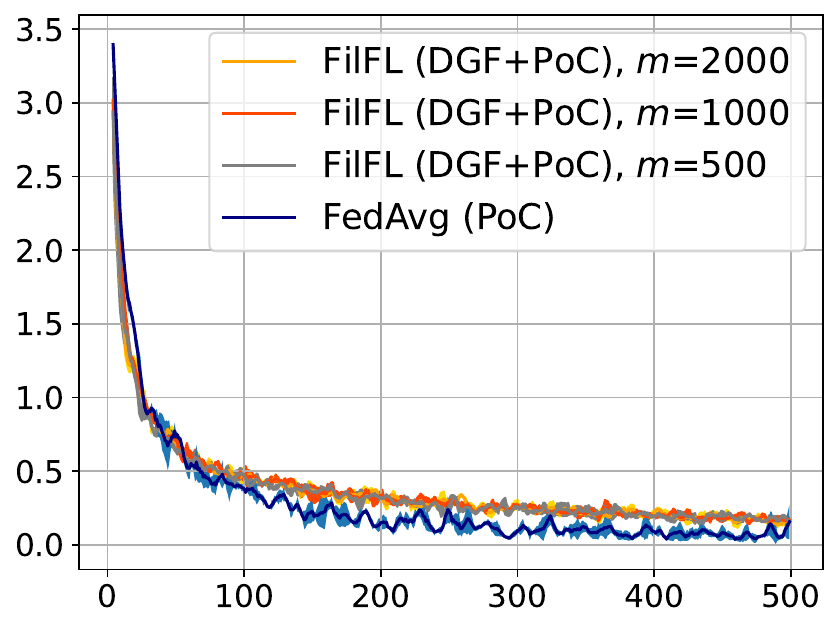}};
  \node[left=of img, node distance=0cm, rotate=90, anchor=center,yshift=-1.0cm] {\tiny Training Loss};
    \node[below=of img, node distance=0cm, rotate=0cm, anchor=center,yshift=3.4cm] {\tiny FEMNIST};
  \node[below=of img, node distance=0cm, rotate=0cm, anchor=center,yshift=1.0cm] {\tiny Round};
\end{tikzpicture}
\end{minipage}%
\begin{minipage}{0.25\textwidth}
\begin{tikzpicture}
  \node (img)  {\includegraphics[scale=0.2]{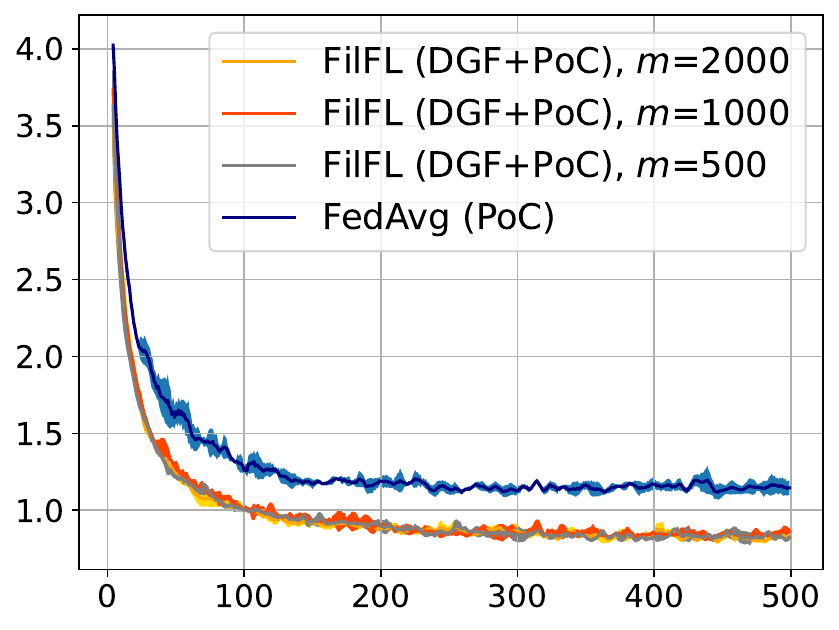}};
  \node[left=of img, node distance=0cm, rotate=90, anchor=center,yshift=-1.0cm] {\tiny Test Loss};
    \node[below=of img, node distance=0cm, rotate=0cm, anchor=center,yshift=3.4cm] {\tiny FEMNIST};
  \node[below=of img, node distance=0cm, rotate=0cm, anchor=center,yshift=1.0cm] {\tiny Round};
\end{tikzpicture}
\end{minipage}%
\begin{minipage}{0.25\textwidth}
\begin{tikzpicture}
  \node (img)  {\includegraphics[scale=0.2]{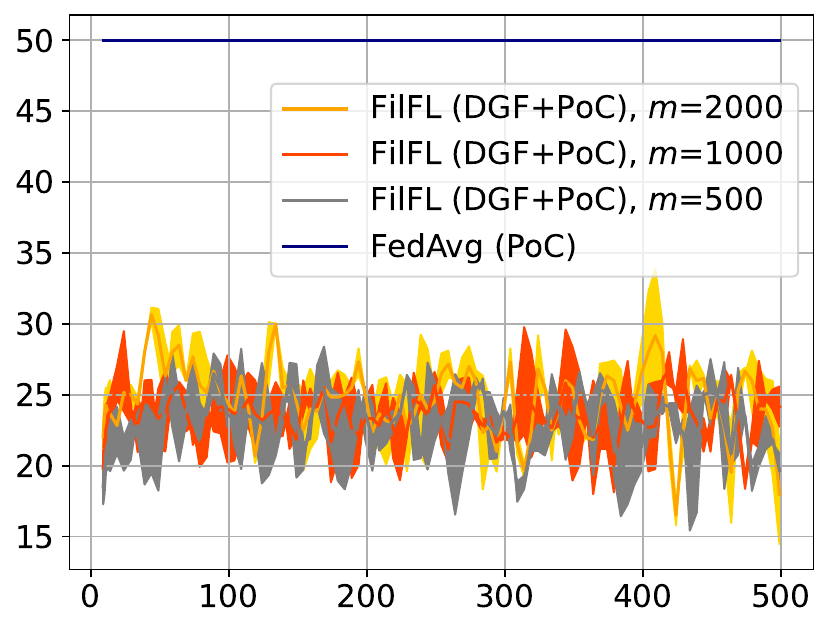}};
  \node[left=of img, node distance=0cm, rotate=90, anchor=center,yshift=-1.0cm] {\tiny $|\mathcal{S}^f|$};
    \node[below=of img, node distance=0cm, rotate=0cm, anchor=center,yshift=3.4cm] {\tiny FEMNIST};
  \node[below=of img, node distance=0cm, rotate=0cm, anchor=center,yshift=1.0cm] {\tiny Round};
\end{tikzpicture}
\end{minipage}%
\caption{FilFL (FedAvg with DGF) sensitivity to filtering dataset size $m$ on FEMNIST dataset.}
\label{fig:sizesensitivityfemnist}
\end{figure}

Hence, FilFL can perform well even with a few data points in the filtering dataset, even in distribution shifts, making our approach a versatile and robust method.

\subsection{FilFL (FedAvg with RGF and PoC) vs FedAvg (PoC) with a Variable Filtering Dataset}
\label{appendix:stoch}

We evaluate the effect of using a stochastic variable dataset for client filtering. Instead of solving the filtering objective on a central dataset, possibly on a subset of the server validation dataset or one single client throughout the training, we consider the case of randomly selecting a client from the available clients to perform the \textit{client filtering} task. The chosen client performs \textit{client filtering} on its own validation dataset. Therefore, the filtering dataset becomes variable depending on the chosen client in that round. Our results demonstrate that FilFL, using RGF, even in such a stochastic scenario, achieves significantly better performance than FedAvg. In particular, as depicted in Fig. \ref{fig:stoch_cifar10}, FilFL accomplishes accelerated training and attains approximately 10 percentage points higher test accuracy than FedAvg.

\begin{figure}[H]
\centering
\begin{minipage}{0.32\textwidth}
\begin{tikzpicture}
  \node (img)  {\includegraphics[scale=0.24]{figures/cifar10_PoC_acc_client.pdf}};
  \node[below=of img, node distance=0cm, rotate=0cm, anchor=center,yshift=3.8cm] {\tiny CIFAR-10};
  \node[below=of img, node distance=0cm, rotate=0cm, anchor=center,yshift=1.0cm] {\tiny Round};
  \node[left=of img, node distance=0cm, rotate=90, anchor=center,yshift=-1.0cm] {\tiny Test Accuracy};
 \end{tikzpicture}
\end{minipage}%
\begin{minipage}{0.32\textwidth}
\begin{tikzpicture}
  \node (img)  {\includegraphics[scale=0.24]{figures/cifar10_PoC_train_loss_client.pdf}};
  \node[left=of img, node distance=0cm, rotate=90, anchor=center,yshift=-1.0cm] {\tiny Training Loss};
    \node[below=of img, node distance=0cm, rotate=0cm, anchor=center,yshift=3.8cm] {\tiny CIFAR-10};
  \node[below=of img, node distance=0cm, rotate=0cm, anchor=center,yshift=1.0cm] {\tiny Round};
\end{tikzpicture}
\end{minipage}%
\begin{minipage}{0.32\textwidth}
\begin{tikzpicture}
  \node (img)  {\includegraphics[scale=0.24]{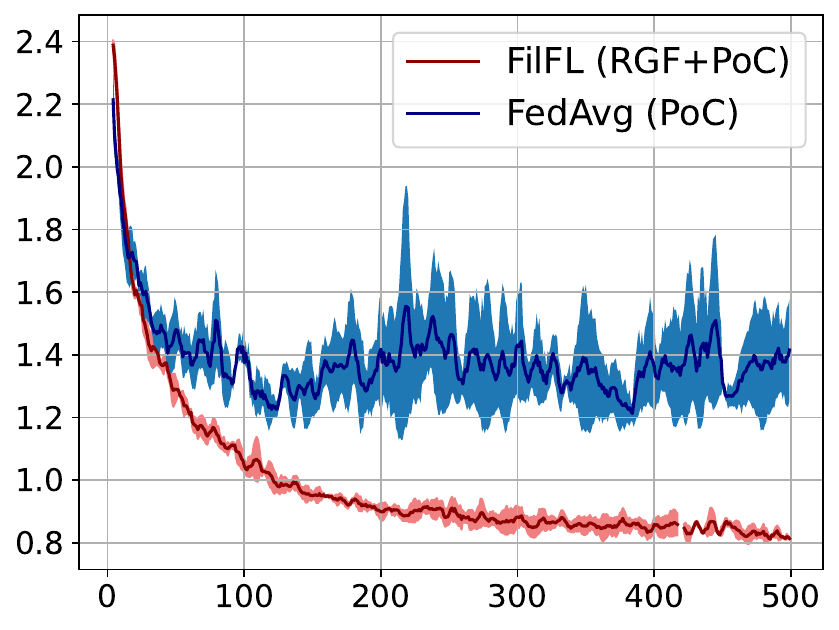}};
  \node[left=of img, node distance=0cm, rotate=90, anchor=center,yshift=-1.0cm] {\tiny Test Loss};
    \node[below=of img, node distance=0cm, rotate=0cm, anchor=center,yshift=3.8cm] {\tiny CIFAR-10};
  \node[below=of img, node distance=0cm, rotate=0cm, anchor=center,yshift=1.0cm] {\tiny Round};
\end{tikzpicture}
\end{minipage}%
\caption{FilFL (FedAvg + RGF + PoC) vs FedAvg (PoC) without filtering on CIFAR-10 dataset.}
\label{fig:stoch_cifar10}
\end{figure}





\end{document}